\documentclass[11pt]{article}

\usepackage[a4paper,margin=1in]{geometry}

\usepackage[round,authoryear]{natbib}

\usepackage{amsmath,amssymb,mathtools,amsthm}

\usepackage{xparse}   
\usepackage{xargs}    
\usepackage{ifthen}   
\usepackage{xcolor}
\usepackage{enumitem}
\usepackage{url}
\usepackage{dsfont}
\usepackage{algorithm}
\usepackage{algpseudocode}
\usepackage{cancel}
\usepackage{titletoc}

\usepackage[colorlinks=true,
            citecolor=blue,
            linkcolor=blue,
            urlcolor=blue]{hyperref}

\usepackage[capitalize,noabbrev]{cleveref}

\newcommand{\rmd}{\mathrm{d}}
\newcommand{\rme}{\mathrm{e}}
\newcommand{\eqsp}{\,}

\newcommand{\eqdef}{:=}
\newcommand{\pf}{p}
\newcommand{\pb}{\overleftarrow{p}}

\newcommand{\E}{\mathbb{E}}

\newcommand{\pidata}{\pi_{\mathrm{data}}}

\newcommand{\dobs}{d_\mathrm{y}}

\newcommand{\xobs}{y} 
 
\newcommand{\Xobs}{Y} 

\newcommand{\Xora}{\overrightarrow{\mathbf{X}}}
\newcommand{\Xola}{\overleftarrow{\mathbf{X}}}
\newcommand{\Xbar}{\bar{\mathbf{X}}}
\newcommand{\X}{\mathbf{X}}
\newcommand{\x}{\mathbf{x}}
\newcommand{\z}{\mathbf{z}}

\newcommand{\Id}{\mathbf{I}}

\newcommand{\Q}{\mathrm{Q}}

\newcommandx{\pdata}[1][1=]{\ifthenelse{\equal{#1}{}}{p_{\operatorname{data}}}{p_{\operatorname{data}}\left(#1\right)}}
\NewDocumentCommand{\gaussiand}{m m o}{\IfValueTF{#3}{\mathcal{N}\left(#3; #1, #2\right)}{\mathcal{N}\left(#1, #2\right)}}
\newcommandx{\gaussMarg}[2][2=]{\ifthenelse{\equal{#2}{}}{\varphi_{#1}}{\varphi_{#1}\left(#2\right)}}
\NewDocumentCommand{\unif}{m o}{\IfValueTF{#2}{\mathcal{U}_{#1}\left(#2\right)}{\mathcal{U}_{#1}}}
\NewDocumentCommand{\indi}{m o}{\mathrm{1}_{#1}\IfValueTF{#2}{\left(#2\right)}{}}
\newcommand{\PE}[2]{\mathbb{E}_{#1}\left[#2\right]}
\newcommand{\PV}[2]{\mathbb{V}_{#1}\left[#2\right]}

\newcommand{\set}[1]{\mathrm{#1}}
\newcommand{\kernel}[1]{\mathsf{#1}}

\newcommand{\borelians}[1]{\mathcal{B}(#1)}
\newcommand{\ball}[2]{\operatorname{B}_{#2}(#1)}

\NewDocumentCommand{\minconst}{m m o}{\varepsilon_{#2|#1}\IfValueTF{#3}{^{(#3)}}{}}

\NewDocumentCommand{\minmeas}{m m o}{\IfValueTF{#3}{\upsilon_{#2| #1}\left(#3\right)}{\upsilon_{#2|#1}}}

\def\rset{\mathbb{R}}
\def\rsetpos{\mathbb{R}_{+}}

\def\xdim{d}

\newcommand\signedmeasures[1]{\mathcal{M}_s\left(#1\right)}
\newcommand\vsignedmeasures[1]{\mathcal{M}_{V}\left(#1\right)}
\newcommand\measures[1]{\mathcal{M}_{+}\left(#1\right)}
\newcommand\probmeasures[1]{\mathcal{P} \left(#1\right)}
\newcommand\measfuns[1]{\mathcal{F}(#1)}

\NewDocumentCommand{\timechange}{o}{\IfValueTF{#1}{\tau\left(#1\right)}{\tau}}

\newcommandx{\fwdmarg}[2][2=]{\ifthenelse{\equal{#2}{}}{\pf_{#1}}{\pf_{#1}\left(#2\right)}}
\newcommandx{\fwdtrans}[4][3=,4=]{\ifthenelse{\equal{#3}{}}{\pf_{#2|#1}}{\pf_{#2|#1}\left(#4|#3\right)}}
\newcommand{\fwdmean}[2]{\operatorname{m}_{#2|#1}}

\newcommand{\fwdvar}[2]{\sigma^2_{#2|#1}}
\newcommandx{\bwdmarg}[2][2=]{\ifthenelse{\equal{#2}{}}{\pb_{#1}}{\pb_{#1}\left(#2\right)}}
\NewDocumentCommand{\bwdker}{m o o o}{\IfValueTF{#3}{\IfValueTF{#2}{\Q_{#2|#1}(#3; #4)}{\Q_{#1}(#3; #4)}}{\IfValueTF{#2}{\Q_{#2|#1}}{\Q_{#1}}}}
\NewDocumentCommand{\bwdkercomp}{O{} m m o o}{
\mathbb{Q}^{#1}_{#2:#3}\IfValueTF{#4}{(#4; #5)}{}}
\NewDocumentCommand{\pgen}{O{}}{\hat{\pi}^{\theta}}

\NewDocumentCommand{\bwdkerdiscr}{m o o o}{\IfValueTF{#3}{\IfValueTF{#2}{\bar{\Q}_{#2|#1}(#3; #4)}{\bar{\Q}_{#1}(#3; #4)}}{\IfValueTF{#2}{\bar{\Q}_{#2|#1}}{\bar{\Q}_{#1}}}}

\NewDocumentCommand{\approxbwdker}{O{\theta} m o o o o}{\IfValueTF{#4}{\IfValueTF{#3}{\Q^{#1}_{#3|#2}(#4; #5)}{\Q^{#1}_{#2}(#4; #5)}}{\IfValueTF{#3}{\Q^{#1}_{#3| #2}}{\Q^{#1}_{#2}}}}
\NewDocumentCommand{\approxbwdkercomp}{O{\theta} O{} m m o o}{
\mathbb{Q}^{\theta\IfValueTF{#2}{#2}{}}_{#3:#4}\IfValueTF{#5}{(#5; #6)}{}}
\newcommandx{\bridge}[4][3=,4=]{\ifthenelse{\equal{#3}{}}{\pf_{#2 | #1}}{\pf_{#2 | #1}\left(#4 |#3\right)}}

\NewDocumentCommand{\score}{o o}{\IfValueTF{#2}{\operatorname{S}_{#1}\left(#2\right)}{\operatorname{S}_{#1}}}
\NewDocumentCommand{\tildescore}{o o}{\IfValueTF{#2}{\operatorname{\tilde{S}}_{#1}\left(#2\right)}{\operatorname{\tilde{S}}_{#1}}}
\NewDocumentCommand{\jscore}{o o}{\IfValueTF{#2}{\nabla \operatorname{S}_{#1}\left(#2\right)}{\nabla \operatorname{S}_{#1}}}
\NewDocumentCommand{\scorenet}{o o O{\theta} }{\IfValueTF{#1}{\operatorname{s}_{#3}\left(#2, #1\right)}{\operatorname{s}_{#3}}}
\NewDocumentCommand{\denoiser}{o o}{\IfValueTF{#2}{\operatorname{D}_{#1}(#2)}{\operatorname{D}_{#1}}}
\NewDocumentCommand{\pset}{m O{}}{\IfValueTF{#2}{\mathcal{P}_{#2}(#1)}{\mathcal{P}\left(#1\right)}}

\newcommand{\dotprod}[2]{\left<#1, #2\right>}

\NewDocumentCommand{\csp}{m m o}{\IfValueTF{#3}{\mathcal{C}^{#1}\left(#2; #3\right)}{\mathcal{C}^{#1}\left(#2\right)}}
\NewDocumentCommand{\lpsp}{m m o}{\IfValueTF{#3}{\mathcal{L}_{#1}\left(#2; #3\right)}{\mathcal{L}_{#1}\left(#2\right)}}

\def\refmeas{\pi_{\infty}}
\def\isvp{\alpha}
\newcommand{\noisesch}[1]{\beta_{#1}}
\newcommand{\bwdnoisesch}[1]{\bar{\beta}_{#1}}

\newcommand{\brownvar}[1]{\operatorname{B}_{t}}

\NewDocumentCommand{\pnorm}{O{2} m}{\left\|#2\right\|_{#1}}

\NewDocumentCommand{\lyapunov}{m o}{%
  \operatorname{V}_{#1}%
  \IfValueT{#2}{\!\left(#2\right)}%
}

\NewDocumentCommand{\ctescorenorm}{o m}{%
  \IfValueTF{#1}
    {\tilde{\gamma}_{#2}}
    {\gamma_{#2}}
}

\NewDocumentCommand{\ctescoreoffset}{o m}{%
  \IfValueTF{#1}
    {\tilde{\kappa}_{#2}}
    {\kappa_{#2}}
}

\NewDocumentCommand{\grid}{o o}{
    {\IfValueTF{#1}{t_{\IfValueTF{#2}{#1:#2}{#1}}}{\mathcal{T}}}
}
\NewDocumentCommand{\multlyap}{m m o}{\lambda_{#2|#1}^{\IfValueTF{#3}{(#3)}{}}}
\NewDocumentCommand{\maxmultlyap}{m m o}{\overline{\lambda}_{#2 |#1}^{\IfValueTF{#3}{(#3)}{}}}
\NewDocumentCommand{\biaslyap}{m m o}{\mathrm{K}_{#2|#1}^{\IfValueTF{#3}{(#3)}{}}}
\NewDocumentCommand{\maxbiaslyap}{m m o}{\overline{\mathrm{K}}_{#2|#1}^{\IfValueTF{#3}{(#3)}{}}}

\NewDocumentCommand{\pmixtime}{o o}{%
  \IfValueTF{#1}
    {\bar\alpha_{#2|#1}}
    {\bar\alpha}
}
\newcommand{\nofrac}[2]{\left(#1 \middle/#2\right)}

\NewDocumentCommand{\bmetric}{s m m}{%
  \IfBooleanTF{#1}
    {\vnorm[#2 - #3][b^\star V]} 
    {\vnorm[#2 - #3][bV]}       
}

\NewDocumentCommand{\vmetric}{s m m}{%
  \IfBooleanTF{#1}
    {\vnorm[#2 - #3][V]} 
    {\vnorm[#2 - #3][V]}       
}
\newcommand{\normEc}[1]{\left\|#1\right\|}

\NewDocumentCommand{\wasserstein}{O{} O{} m m}{%
  \mathcal{W}_{#1}%
  \IfNoValueF{#2}{^{#2}}%
  \!\left(#3,#4\right)%
}
\newcommand{\norminfty}[1]{\left\|#1\right\|_{\infty}}

\NewDocumentCommand{\Cdiscr}{m o}{%
  C^{\rm discr}_{#1}%
  \IfValueT{#2}{\!\left[#2\right]}%
}
\NewDocumentCommand{\Cmix}{m o}{%
  C^{\rm mix}_{#1}%
  \IfValueT{#2}{\!\left[#2\right]}%
}
\NewDocumentCommand{\Cnet}{m o}{%
  C^{\rm net}_{#1}%
  \IfValueT{#2}{\!\left[#2\right]}%
}

\newcommand{\un}{\mathds{1}}

\NewDocumentCommand{\filter}{m o}{%
  \phi_{#1}
  \IfValueT{#2}{\!\left[#2\right]}%
}

\NewDocumentCommand{\mcfilter}{m o O{\theta} O{\npart}}{%
  \phi_{#1}^{#3,#4}%
  \IfValueT{#2}{\!\left[#2\right]}%
}

\NewDocumentCommand{\tildemcfilter}{m o}{%
  \widetilde{\phi}_{#1}^{\theta,M}%
  \IfValueT{#2}{\!\left[#2\right]}%
}

\NewDocumentCommand{\approxfilter}{m o}{%
  \phi_{#1}^\theta
  \IfValueT{#2}{\!\left[#2\right]}%
}

\NewDocumentCommand{\filterinit}{o}{%
  \phi_{0}
  \IfValueT{#1}{\!\left[#1\right]}%
}

\NewDocumentCommand{\approxfilterinit}{o}{%
  \phi^{\theta}_{0}
  \IfValueT{#1}{\!\left[#1\right]}%
}

\NewDocumentCommand{\muterr}{m o}{%
  M_{#1}^{M}%
  \IfValueT{#2}{\!\left[#2\right]}%
}

\NewDocumentCommand{\selerr}{m o}{%
  S_{#1}^{M}%
  \IfValueT{#2}{\!\left[#2\right]}%
}

\NewDocumentCommand{\filtertransform}{m s o}{
  \mathbb{F}_{#1}
  \IfValueTF{#3}{^{#3}}{}
}

\newcommand{\backwardfunInternal}{\beta}

\NewDocumentCommand{\backwardfun}{m m O{\theta}}{
  \backwardfunInternal_{#1|#2}\IfValueTF{#3}{^{#3}}{}
}
\NewDocumentCommand{\barbackwardfun}{m m O{\theta}}{
  \bar{\backwardfunInternal}_{#1|#2}\IfValueTF{#3}{^{#3}}{}
}
\NewDocumentCommand{\backwardnu}{m m O{\theta}}{
  \nu_{#1|#2}\IfValueTF{#3}{^{#3}}{}
}
\NewDocumentCommand{\normbackwardfun}{m m O{\theta} o}{
  \widetilde{\backwardfunInternal}_{#1|#2}^{#3}%
  \IfValueT{#4}{\!\left(#4\right)}%
}

\NewDocumentCommand{\muBias}{m m o}{
  \mu_{#1|#2}\IfValueTF{#3}{^{#3}}{}
}
\NewDocumentCommand{\vosc}{m O{b}}{\operatorname{osc}_{#2V}\left(#1\right)}
\NewDocumentCommand{\measvnorm}{m O{b}}{\|#1\|_{#2V}}
\NewDocumentCommand{\pot}{m o}{g_{#1}\IfValueTF{#2}{\left(#2\right)}{}}

\NewDocumentCommand{\approxpot}{m o}{\hat{g}_{#1}\IfValueTF{#2}{\left(#2\right)}{}}

\NewDocumentCommand{\potnorm}{m o}{\tilde{g}_{#1}\IfValueTF{#2}{\left(#2\right)}{}}

\NewDocumentCommand{\unpredictive}{m o}{\kernel{K}_{#1}\IfValueTF{#2}{^{#2}}{}}
\NewDocumentCommand{\predictive}{m o}{\kernel{F}_{#1}\IfValueTF{#2}{^{#2}}{}}
\NewDocumentCommand{\fsmk}{m m o o o}{\kernel{G}_{#1|#2}\IfValueTF{#3}{^{#3}}{}\IfValueTF{#4}{(#4, #5)}{}}

\NewDocumentCommand{\fsmkstar}{m m o o o}{\kernel{G}_{#1|#2} \IfValueTF{#3}{^{#3,*}}{^{\star}} \IfValueTF{#4}{(#4, #5)}{} }

\NewDocumentCommand{\prodfsmk}{m m o o o}{\kernel{G}_{#1:#2}\IfValueTF{#3}{^{#3}}{}\IfValueTF{#4}{(#4, #5)}{}}

\NewDocumentCommand{\mk}{m o o o}{\kernel{Q}_{#1}\IfValueTF{#2}{^{#2}}{}\IfValueTF{#3}{(#3, #4)}{}}
\NewDocumentCommand{\eulermk}{m o}{ \overline{\kernel Q}_{#1}\IfValueTF{#2}{^{#2}}{} }
\def\particle{\xi}
\newcommand{\chunk}[3]{#1_{#2:#3}}
\NewDocumentCommand{\dirac}{m o}{\delta_{#1}\IfValueTF{#2}{\left(#2\right)}{}}
\def\npart{M}

\NewDocumentCommand{\empmeas}{m o}{\mu_{\left(#1\right)}\IfValueTF{#2}{\left(#2\right)}{}}
\NewDocumentCommand{\vnorm}{O{} O{V}}{\left\|#1\right\|_{#2}}

\NewDocumentCommand{\CsupSMC}{O{\theta} m m}{%
  C_{\beta,#2|#3}^{#1}%
}
\NewDocumentCommand{\ClpSMC}{O{\theta} m m}{%
  C_{\Gamma,#2|#3}^{#1}%
}

\NewDocumentCommand{\twistedlyap}{m m O{\theta} o}{%
  V_{#1|#2}^{#3}%
  \IfValueT{#4}{\!\left(#4\right)}%
}

\NewDocumentCommand{\Gminor}{m m o}{%
  \nu_{#1,#2}^{\mathsf G}%
  \IfValueT{#3}{\!\left(#3\right)}%
}

\NewDocumentCommand{\Ftwist}{m m O{\theta} o}{%
  F_{#1|#2}^{#3}%
  \IfValueT{#4}{\!\left(#4\right)}%
}

\NewDocumentCommand{\kappaG}{m O{a}}{%
  \kappa_{#1}^{\mathsf G,#2}%
}

\theoremstyle{plain}
\newtheorem{theorem}{Theorem}[section]
\newtheorem{proposition}[theorem]{Proposition}
\newtheorem{lemma}[theorem]{Lemma}
\newtheorem{corollary}[theorem]{Corollary}
\theoremstyle{definition}

\newtheorem{assumption}[theorem]{Assumption}
\theoremstyle{remark}
\newtheorem{remark}[theorem]{Remark}

\newlist{assumplist}{enumerate}{1}
\setlist[assumplist]{label={(\Roman*)},ref=\theassumption~(\Roman*)}

\crefalias{assumplisti}{assumption}

\usepackage{booktabs}
\usepackage{tabularx}
\usepackage{array}

\usepackage[utf8]{inputenc} 
\usepackage[T1]{fontenc}    
\usepackage{hyperref}       
\usepackage{url}            
\usepackage{booktabs}       
\usepackage{amsfonts}       
\usepackage{nicefrac}       
\usepackage{microtype}      
\usepackage{xcolor}         

\title{Theoretical Guarantees for SMC-Guided Diffusion Sampling}

\author{%
  Stanislas Strasman\textsuperscript{1} \quad
  Gabriel Victorino Cardoso\textsuperscript{2} \quad
  Sylvain Le Corff\textsuperscript{1} \\
  Vincent Lemaire\textsuperscript{1} \quad
  Antonio Ocello\textsuperscript{3}
  \\[0.75em]
  \textsuperscript{1} Sorbonne Université and Université Paris Cité, CNRS, LPSM, F-75005 Paris, France\\
  \textsuperscript{2} Center for Statistics and Images, Mines Paris, PSL University, Fontainebleau, France\\
  \textsuperscript{3} CREST, ENSAE Paris, Institut Polytechnique de Paris, Palaiseau, France\\
}

\date{} 

\begin{document}

\maketitle

\begin{abstract}

Post-hoc conditioning of pretrained diffusion models can be addressed using Sequential Monte Carlo (SMC) methods. By evolving an interacting particle system, SMC-guided diffusion samplers combine unconditional reverse-diffusion dynamics with sequential reweighting to approximate conditional distributions. Nevertheless, even in the infinite-particle limit, the implemented sampler may differ from the ideal conditional target because of errors in the diffusion model, its numerical implementation, and the guidance mechanism. We characterize how these local errors propagate through forward-smoothing kernels, which jointly account for the reverse dynamics and the remaining conditioning information. This yields non-asymptotic error bounds that capture both finite-particle fluctuations and approximation errors arising from initialization, numerical integration, score approximation, and potential design. In doing so, we extend stability guarantees for diffusion models to the conditional setting. Finally, we apply our framework to several state-of-the-art SMC-guided diffusion algorithms, providing a unified theoretical perspective on their approximation mechanisms and sources of error.

\end{abstract}

\section{Introduction}

Score-based Generative Models (SGMs) also called diffusion models provide a flexible way of learning complex high-dimensional data distributions together with a tractable sequential generative structure \citep{dickstein2015,song2019generative,ho2020denoising,song2021score}. While these models are primarily trained for unconditional generation,
many applications require controlling the generated distribution through
additional observations, constraints, or rewards. Conditional diffusion
methods broadly follow two paradigms: conditioning may be incorporated
during training through the model architecture or objective
\citep{ho2021classifierfree,rombach2022highresolution,saharia2022palette},
or imposed \emph{post hoc} at inference time by steering a pretrained
unconditional model. We focus on the latter setting, which arises for
instance in Bayesian inverse problems, missing-data imputation, reward guidance, and
conditional generation from partial observations.
How to turn a pretrained diffusion model into an \emph{accurate guided sampler} without retraining it?

Sequential Monte Carlo (SMC) provides a natural mechanism for inference-time guidance. It approximates a sequence of distributions using interacting particles that alternate between \emph{mutation} through proposal kernels and \emph{selection} through weighting and resampling \citep{doucet2001sequential,delmoral2004feynman,del2006sequential,chopin2020introduction}. Diffusion SMC samplers build their proposals from the pretrained reverse-diffusion dynamics and use intermediate potential functions to incorporate the guidance information \citep{trippe2023diffusion,wu2023practical,cardoso2024monte,dou2024diffusion,singhal2025general,nazemi2024particle}. With appropriate proposals and potentials, they aim to approximate a tilted target of the form \(\pi^\star(\mathrm{d}x)\propto G(x)\pi_{\rm data}(\mathrm{d}x)\), where \(G(x)\) assigns greater weight to samples with the desired attributes.



A natural next question is whether the particles accurately approximate \(\pi^\star\). Classical SMC theory quantifies the error due to using finitely many particles for a specified Feynman--Kac flow \citep{delmoral2004feynman,whiteley2012sequential,douc2014longterm}. In practical diffusion samplers, that flow is constructed using simulable approximations: a learned score and discretized reverse dynamics, and sometimes approximate intermediate potentials. Particle convergence therefore establishes accuracy for the implemented flow, which need not coincide with \(\pi^\star\). This motivates an analysis that combines classical SMC error bounds with the errors introduced when score-based approximations are used to construct the twisted flow.

The resulting error therefore has two qualitatively different sources: finite-particle Monte Carlo error and deterministic approximation bias. The latter already appears in unconditional diffusion sampling, where convergence analyses typically separate initialization mismatch, score-approximation error, and numerical discretization error \citep{chen2023sampling,benton2024nearly,conforti_kl,strasman2025an,strasman2025wasserstein,silveri205beyond}. In the guided setting, however, local errors in the diffusion dynamics are propagated through the remaining selection--mutation steps, altering their effect on the terminal distribution. This interaction is the missing link between diffusion-model convergence and SMC stability.

\paragraph{Contributions.}
Our main contributions can be summarized as follows.
\begin{itemize}[leftmargin=*,itemsep=1pt]
\item
\textbf{Joint particle and model-approximation guarantees.}
To the best of our knowledge, we derive the first non-asymptotic guarantee for diffusion-based SMC that jointly captures finite-particle error and approximation bias. The bound separates the usual $M^{-1/2}$ Monte Carlo contribution from propagated errors due to initialization, numerical discretization, score approximation, and guidance-potential mismatch. It therefore distinguishes errors that vanish by increasing the particle count from those that require improving the underlying diffusion model or guidance design.

\item \textbf{Guided forgetting on noncompact state spaces.}
We establish a forgetting mechanism for the time-inhomogeneous forward-smoothing dynamics induced by sequential guidance. Under structural conditions on the learned score and guidance potentials, Lyapunov drift and local minorization of the implemented diffusion proposals yield contraction in guidance-adapted weighted total-variation norms. Building on the Harris argument of \citet{HairerMattingly2008} and extending the unconditional analysis of \citet{strasman2026forgetting}, the resulting stability coefficients quantify how local particle and approximation errors are attenuated by the subsequent reverse dynamics and remaining guidance.

\item \textbf{Connection with practical diffusion SMC.}
We develop a flexible framework covering several existing diffusion-based SMC samplers, including \citet{wu2023practical,cardoso2024monte}, and recast them as particle approximations of an approximate Feynman--Kac flow whose terminal distribution may remain biased with respect to the desired target. By mapping their concrete proposal and weighting mechanisms onto our assumptions, the theory provides tools to analyze their stability, trace how the different sources of error propagate through the sampling trajectory, and better understand their strengths and potential failure modes.
\end{itemize}

\section{Conditional diffusion sampling with SMC}
\label{sec:diffusion_smc}

Let $\pidata\in\probmeasures{\rset^\xdim}$ be the data distribution for which a pretrained diffusion model is available. Assigning preferences to samples from $\pidata$ can be formulated as sampling from the tilted distribution
$$
\pi^\star(\mathrm{d}x) \propto G(x)\,\pidata(\mathrm{d}x)\eqsp,
$$
where $G$ is a nonnegative function encoding the desired preference and the normalizing constant is assumed to be finite and positive. For example, Bayesian posterior sampling can be recasted in this setting defining $G(x)=p(y\mid x)$ \citep{wu2023practical,cardoso2024monte}, for a fixed observation $y$, as well as reward-guided generation, with $G(x)=\exp\{\lambda r(x)\}$ for a reward function $r$ and guidance strength $\lambda>0$ \citep{wu2023human,black2024training, fan2023dpok, marion2025implicit}. We next show that the sequential structure of diffusion models naturally fits SMC: SGM transitions define the mutation kernels, while intermediate potentials progressively incorporate the guidance $G$. 


\subsection{Score-based generative models as priors}

We first recall how score-based generative models transform noise into samples, establishing the reverse dynamics used in the guided sampler.

\paragraph{Forward-backward Diffusion processes.} Fix a time horizon $T>0$. Fix $\isvp\geq 0$ and $t \in [0,T]\ni t\mapsto \noisesch{t}>0$ a continuous noise schedule. The \emph{forward diffusion} is defined as
\begin{align}
\label{eq:forward_sde_main}
    \rmd \Xora_t = -\isvp \beta(t)  \Xora_t \rmd t + \sqrt{2\beta(t)} \rmd B_t \eqsp,
    \qquad \Xora_0\sim \pidata \eqsp,
\end{align}
where $(B_t)_{t \in [0,T]}$ is a $d$-dimensional Brownian motion. Appropriate choices of $(\isvp,\noisesch{\cdot})$ recover the VE and VP diffusions \citep{song2019generative,dickstein2015,ho2020denoising}, we refer to \citet{strasman2026forgetting} for more details. We write $p_t$ to refer to both the density and, by abuse of notation, the law of $\Xora_t$. %
Under minimal assumptions on $\pi_{\rm data}$, the forward dynamics \eqref{eq:forward_sde_main} admits a time-reversed process \citep{haussmann1986time,cattiaux}, also known as \emph{backward process}, given by
\begin{align}
\label{eq:reverse_sde_main}
    \rmd \Xola_t = \left( \isvp \bwdnoisesch{t}\Xola_t + 2\bwdnoisesch{t}\nabla \log p_{T-t}(\Xola_t) \right)\rmd t + \sqrt{2\bwdnoisesch{t}}\,\rmd B_t \eqsp, \qquad \Xola_0\sim p_T \eqsp,
\end{align}
with $\bwdnoisesch{t}\eqdef\beta(T-t)$. %
To connect the reverse process with the SMC construction below, we fix the times at which particles will be weighted and resampled. For simplicity, let $N\in\mathbb N^\ast$ and choose the uniform grid
$$
\mathcal T^{\rm sel}\eqdef\{t_k\}_{k=0}^N \eqsp, \qquad t_k\eqdef \frac{kT}{N} \eqsp,  \quad k=0,\ldots,N \eqsp.
$$
For $k=0,\ldots,N-1$, let $\mk{k+1}$ be the exact transition kernel
of \eqref{eq:reverse_sde_main} over $[t_k,t_{k+1}]$:
$$
\mk{k+1}(\x,\set{A}) \eqdef \mathbb P\!\left( \Xola_{t_{k+1}}\in\set{A} \,\middle|\, \Xola_{t_k}=\x \right) \eqsp, \qquad \set{A}\in\borelians{\rset^\xdim} \eqsp.
$$
By time reversal, the unguided composition satisfies
$p_T\mk{1}\cdots\mk{N}=\pidata$. These kernels thus define an ideal reference for the mutation steps; the practical, simulable kernels introduced next approximate them. 

\paragraph{Practical Implementation.}
On a practical level, three ingredients are required to implement the reverse dynamics \eqref{eq:reverse_sde_main}: (i) a tractable reference distribution $\refmeas$ approximating $p_T$, (ii) a learned score network $\scorenet[t][\cdot][\theta]$ approximating $\nabla\log p_t$, and (iii) a numerical scheme to discretize the SDE. We therefore obtain a viable approximation by initializing the reverse process from the tractable reference distribution $\refmeas$ in place of $p_T$, substituting the learned score $\scorenet[t][\cdot][\theta]$ for the unavailable score $\nabla\log p_t$, and applying a numerical discretization to the resulting reverse SDE. Together, these three approximations turn the ideal reverse dynamics into a simulable Markov chain.

Two distinct grids on $[0,T]$ enter the construction. The grid $\mathcal T^{\rm sel}=\{t_k\}_{k=0}^N$ specifies the times at which particles are weighted and, when required, resampled. These times are
determined by the SMC construction, for instance by the times at which guidance information is introduced. By contrast, numerical simulation relies on a possibly finer discretization grid
$$
\mathcal T^{\rm EM}\eqdef\{s_n\}_{n=0}^{NJ} \eqsp, \qquad \text{ with } \quad s_n\eqdef\frac{nT}{NJ} \quad \text{ and } \quad h\eqdef\frac{T}{NJ} \eqsp,
$$
where $J\geq1$ is the number of numerical propagation steps performed between two consecutive SMC times. The grids are aligned through $t_k=s_{kJ}$, but serve different purposes: a single SMC mutation from $t_k$ to $t_{k+1}$ may comprise several Euler--Maruyama updates.

Defining
$
\Delta_n\eqdef\int_{s_n}^{s_{n+1}}\bwdnoisesch{u}\,\rmd u ,
$
the corresponding time-changed Euler--Maruyama scheme is
\begin{equation}
\label{eq:approx_reverse_update}
\Xbar_{s_{n+1}}^\theta
=
\Xbar_{s_n}^\theta + \Delta_n \left( \isvp\Xbar_{s_n}^\theta + 2\scorenet[T-s_n][\Xbar_{s_n}^\theta][\theta] \right) + \sqrt{2\Delta_n}\,\xi_n \eqsp, \qquad \Xbar_0^\theta \sim \refmeas \eqsp,
\end{equation}
where $(\xi_n)_{n=0}^{NJ-1}$ are i.i.d.\ $\gaussiand{0}{\Id_\xdim}$ random variables, independent of $\Xbar_0^\theta$. When $J=1$, the numerical grid coincides with the
SMC selection grid.
For $k=0,\ldots,N-1$, we denote by $\mk{k+1}[\theta,h]$ the transition kernel obtained by composing the
$J$ Euler--Maruyama updates in \eqref{eq:approx_reverse_update} over
$[t_k,t_{k+1}]$. It is a simulable approximation of the exact reverse
kernel $\mk{k+1}$, incorporating both score-approximation and
time-discretization errors. Hence, in the absence of guidance, $\refmeas\,\mk{1}[\theta,h]\cdots\mk{N}[\theta,h] \approx \pidata$.







\subsection{Guided targets and Feynman--Kac representation}
\label{sec:conditional_target}

Diffusion models provide an unconditional generative model for $\pidata$, whereas our target is $\pi^\star(\mathrm dx) \propto G(x)\pi_{\mathrm{data}}(\mathrm dx)$. A naive approach would propagate particles according to the unconditional reverse dynamics and apply $G$ only at the terminal time. This amounts to a terminal importance-weighting correction: it does not steer trajectories toward regions where $G$ is large and can therefore lead to severe weight degeneracy and high Monte Carlo variance, especially when the desired attributes are rare.

\paragraph{Ideal Feynman--Kac model.} Diffusion-guided SMC instead introduces guidance progressively along the reverse-time trajectory through a sequence of positive potentials $(\pot{k})_{k=0}^{N-1}$.
To make the selection step explicit, we introduce the \emph{normalized
selection operator}. For a probability measure $\nu$ and a positive
measurable function $g$ satisfying $0<\nu[g]<\infty$, where
$\nu[f]\eqdef\int f\,\rmd\nu$, define
\[
\tau_g(\nu)(\rmd x)
\eqdef
\frac{g(x)\nu(\rmd x)}{\nu[g]}.
\]
The measure $\tau_g(\nu)$ is the selected distribution associated with
$\nu$: states are favored according to their relative potential $g$,
and the denominator ensures that the resulting measure has unit mass.
The ideal Feynman--Kac flow alternates this \textit{selection step }with \textit{mutation} through the exact reverse kernels. Given
$(\pot{k},\mk{k+1})_{k=0}^{N-1}$, the flow is defined recursively by
\begin{equation}
\label{eq:ideal_guided_flow}
\filter{0}=p_T,
\qquad
\filter{k+1}
=\tau_{\pot{k}}(\filter{k})\mk{k+1},
\qquad k=0,\ldots,N-1.
\end{equation}

We consider potential sequences defining an ideal Feynman--Kac
representation of the target, namely
\begin{equation}
\label{eq:terminal_target}
\filter{N}=\pi^\star \eqsp.
\end{equation}
The corresponding Feynman--Kac path measure is obtained by tilting the
law of the reverse chain
$(\Xola_{t_k})_{k=0}^{N}$
by the accumulated potential
$\prod_{k=0}^{N-1}\pot{k}(\Xola_{t_k})$.

\paragraph{Approximate Feynman--Kac model.}
The ideal Feynman--Kac model cannot generally be implemented because
its initialization, mutation kernels, and potentials may involve
intractable quantities. We therefore use the computable approximations
introduced above. First, the unavailable initial law $p_T$ is replaced
by the tractable reference distribution $\refmeas$. Second, each exact reverse kernel $\mk{k+1}$ is replaced by the simulable kernel
$\mk{k+1}[\theta,h]$, obtained by composing the Euler--Maruyama updates
in \eqref{eq:approx_reverse_update}; the parameter $\theta$ records the
replacement of the exact score by the learned score, while $h$ records
the numerical discretization. Finally, when the ideal potential
$\pot{k}$ cannot be evaluated, it is replaced by a positive computable
approximation $\approxpot{k}$. These substitutions make selection and
mutation implementable, but respectively introduce initialization,
score-approximation, time-discretization, and potential-approximation
errors.

Starting from $\approxfilter{0}=\refmeas$, the resulting approximate
guided flow associated with
$(\approxpot{k},\mk{k+1}[\theta,h])_{k=0}^{N-1}$ is defined by
\begin{equation}
    \label{eq:implemented_guided_flow}
    \approxfilter{k+1} = \tau_{\approxpot{k}}(\approxfilter{k}) \mk{k+1}[\theta,h] \eqsp,
    \qquad  \text{ for } k=0,\ldots,N-1 \eqsp.
\end{equation}
Equivalently, consider the simulable numerical reverse chain
$(\Xbar_{s_n}^{\theta})_{n=0}^{NJ}$ defined in
\eqref{eq:approx_reverse_update} and observe it at the SMC times,
$\Xbar_{t_k}^{\theta}=\Xbar_{s_{kJ}}^{\theta}$.
The resulting chain has transitions
$(\mk{k}[\theta,h])_{k=1}^N$, and its Feynman--Kac path measure is obtained
by tilting its law by
$\prod_{k=0}^{N-1}\approxpot{k}(\Xbar_{t_k}^{\theta})$. Thus, $\approxfilter{N}$ is the terminal distribution associated with the implemented Feynman--Kac model, while its discrepancy from the desired target $\filter{N}=\pi^\star$ captures the price paid for using these approximations.
\paragraph{Particle-based approximation.}
To implement the recursion defining the flow $(\approxfilter{k})_{0\le k\le N}$, we resort to the standard SMC population approximation. At each selection time $t_k$, the distribution $\approxfilter{k}$ is represented by a population of $M$ interacting particles $\{\X_k^{\theta,(i)}\}_{i=1}^M$ and its associated empirical measure
$$
\mcfilter{k} \eqdef \frac{1}{M} \sum_{i=1}^{M} \delta_{\X_k^{\theta,(i)}} \eqsp,
$$
where $\delta_{\x}$ denotes the Dirac measure at $\x $. Applying the \textit{selection--mutation} update to $\mcfilter{k}$ gives the following
finite mixture
\begin{equation}
\label{eq:particle_predictive_mixture}
\tau_{\approxpot{k}}(\mcfilter{k}) \mk{k+1}[\theta,h](\rmd \x) = \sum_{i=1}^M w_k^{(i)} \mk{k+1}[\theta,h] (\X_k^{\theta,(i)},\rmd \x) \eqsp, \qquad
w_k^{(i)} \eqdef \frac{\approxpot{k}(\X_k^{\theta,(i)})} {\sum_{j=1}^M\approxpot{k}(\X_k^{\theta,(j)})} \eqsp.
\end{equation}
The resulting empirical measure $\mcfilter{k+1}$ provides an approximation $\approxfilter{k+1}$ in \eqref{eq:implemented_guided_flow}. Algorithm~\ref{alg:smc} summarizes this particle implementation. The formulation is sufficiently flexible to recover concrete diffusion-SMC constructions. In particular, \Cref{app:noiseless_inpainting_correspondence} gives the explicit correspondence with MCGdiff \citep{cardoso2024monte} and SMCDiff \citep{trippe2023diffusion} for noiseless inpainting.

\begin{algorithm}[t]
\caption{Diffusion-based SMC with multinomial resampling}
\label{alg:smc}
\begin{algorithmic}[1]
\State \textbf{Input:} $M,N,J,\approxfilter{0}$,
$(\approxpot{k},\mk{k+1}[\theta,h])_{k=0}^{N-1}$.

\State Sample $\X_0^{\theta,(j)}
\stackrel{\mathrm{iid}}{\sim}\approxfilter{0}$,
$j=1,\ldots,M$.

\For{$k=0,\ldots,N-1$}
    \State \textbf{Weight:}
    $w_k^{(j)}\gets
    \frac{\approxpot{k}(\X_k^{\theta,(j)})}
    {\sum_{m=1}^M\approxpot{k}(\X_k^{\theta,(m)})}$,
    $j=1,\ldots,M$.

    \State \textbf{Select:} sample independently \quad
    $\widetilde\X_k^{\theta,(j)}\sim
    \sum_{m=1}^M w_k^{(m)}\delta_{\X_k^{\theta,(m)}}$,
    $j=1,\ldots,M$.

    \State \textbf{Mutate:}
    $\X_{k+1}^{\theta,(j)}\sim
    \mk{k+1}[\theta,h](\widetilde\X_k^{\theta,(j)},\rmd\x)$,
    $j=1,\ldots,M$.
\EndFor

\State \textbf{Output:}
$\mcfilter{N}=M^{-1}\sum_{j=1}^M\delta_{\X_N^{\theta,(j)}}$.
\end{algorithmic}
\end{algorithm}

\section{Forgetting and error propagation in diffusion-based SMC}
\label{sec:stability_error}

The three constructions introduced in
\Cref{sec:diffusion_smc} naturally decompose the error with respect to
the desired target. The ideal flow terminates at
$\filter{N}=\pi^\star$, the implemented infinite-particle flow
terminates at $\approxfilter{N}$, and the particle output
$\mcfilter{N}$ approximates $\approxfilter{N}$. Hence, for any
admissible test function $\psi$,
\begin{equation}
\label{eq:error_decomposition}
\mcfilter{N}[\psi]-\filter{N}[\psi]
=
\underbrace{
\mcfilter{N}[\psi]-\approxfilter{N}[\psi]
}_{\text{Particle error}}
+
\underbrace{
\approxfilter{N}[\psi]-\filter{N}[\psi]
}_{\text{Approximation bias}}
\eqsp.
\end{equation}
The first term is the finite-particle Monte Carlo error, while the second measures the discrepancy between the implemented and ideal guided flows. We control both through the forgetting properties governing their propagation.


\subsection{Guided forgetting}
\label{sec:guided_forgetting}

\paragraph{Stability of the learned proposals.}
The analysis of SMC methods relies on two closely related tools: \textit{forgetting estimates}, which quantify how the dependence on initial conditions is attenuated along the dynamics, and \textit{stability arguments}, which control how local perturbations propagate through successive selection--mutation steps \citep{HairerMattingly2008,delmoral2004feynman,whiteley2012sequential,douc2014longterm}. For an adaptation of these ideas to the unguided SGM setting, see \citet{strasman2026forgetting}, who establish forgetting properties for the exact reverse-time dynamics under structural assumptions on the data distribution.

In diffusion-based SMC, mutation relies on approximate SGM kernels $\mk{k}[\theta,h]$ built from a learned score and numerical discretization. The classical forgetting and stability tools must therefore be adapted to account for these approximations. We show that Harris stability of the SGM kernels can be established from structural properties of  $\scorenet$.

To formulate stability on the unbounded state space $\rset^\xdim$, we use the quadratic Lyapunov function $V(\x)\eqdef\|\x\|^2$, naturally matched to the dissipativity and growth conditions imposed below and controls the second moments of the chain. The associated $V$-weighted metric records both discrepancies between measures and their behavior in the tails, making it well suited to Markov kernels satisfying Lyapunov drift conditions. Accordingly, for measurable test functions $\psi:\rset^\xdim\to\rset$ and finite signed measures $\mu$ with $|\mu|[V]<\infty$, we define the weighted supremum norm and the $V$-weighted total-variation norm by
\begin{equation}
\label{eq:weighted_norms_main}
    \normEc{\psi}_{V,\infty}
    \eqdef
    \sup_{\x\in\rset^\xdim}
    \frac{|\psi(\x)|}{1+V(\x)},
    \qquad
    \vnorm[\mu][V]
    \eqdef
    \int_{\rset^\xdim}(1+V)\,\rmd|\mu| \eqsp.
\end{equation}

\begin{assumption}[Learned score stability]
\label{ass:dissipative_score_main}
For every $k=0,\ldots,N-1$, there exist $\gamma_k>\isvp/2$, $\kappa_k\ge0$, and $L_k\ge0$ such that, for every $u\in[t_k,t_{k+1})$ and $\x\in\rset^\xdim$,
\begin{equation}
\label{eq:learned_score_conditions_main}
\dotprod{\x}{\scorenet[T-u][\x][\theta]} \le -\gamma_k V(\x)+\kappa_k \eqsp, \qquad \normEc{\scorenet[T-u][\x][\theta]}^2 \le L_k(1+V(\x)) \eqsp.
\end{equation}
\end{assumption}

These are standard dissipativity and linear-growth conditions for the
reverse drift. The first condition, together with $\gamma_k>\isvp/2$, makes the approximate reverse drift dissipative in the tails. The second is implied, for instance, by a stronger global Lipschitz regularity of the network, as satisfied by standard feed-forward architectures with Lipschitz activations \citep{virmaux2018lipschitz}. They are natural learned analogues of the estimates satisfied by the exact score under comparable structural assumptions on $\pidata$; see \citet{strasman2026forgetting}. Thus, when the score approximation preserves these tail properties, \Cref{ass:dissipative_score_main} is expected to hold for the learned reverse dynamics. 

Under \Cref{ass:dissipative_score_main},
the approximate SGM kernels satisfy the two ingredients of Harris stability given in the following proposition, proved in Appendix~\ref{app:diffusion_specialization}.

\begin{proposition}[Stability of the approximate proposal kernels]
\label{prop:drift_minorization_proposal_theta}
Suppose that \Cref{ass:dissipative_score_main} hold. Assume moreover that the numerical increments within
$[t_k,t_{k+1}]$ satisfy
\begin{equation}
\label{eq:block_stepsize_condition}
\Delta_n \le \frac{2\gamma_k-\isvp}{\isvp^2-4\isvp\gamma_k+4L_k} \eqsp,
\qquad n=kJ,\ldots,(k+1)J-1 \eqsp.
\end{equation}

Then,
we have
\begin{equation}
\label{eq:block_proposal_drift}
    \mk{k+1}[\theta,h]V \le \lambda_{k+1}^Q V+K_{k+1}^Q\eqsp.
\end{equation}
with $\lambda_{k+1}^Q = \exp\left( -(2\gamma_k-\isvp) \int_{t_k}^{t_{k+1}}\bwdnoisesch{u}\,\rmd u \right)$, $K_{k+1}^Q = C_k \int_{t_k}^{t_{k+1}} \bwdnoisesch{u}\,\rmd u$, and $C_k$ defined in \eqref{eq:def:C_k}, independent of $J$ and $h$.

Moreover, for every $R>0$, there exist $\varepsilon_{k+1}^Q(R)>0$ and a probability measure $\nu_{k+1,R}^Q$ such that
\begin{equation}
\label{eq:block_proposal_minorization}
\mk{k+1}[\theta,h](\x,A) \ge \varepsilon_{k+1}^Q(R)\, \nu_{k+1,R}^Q(A) \eqsp, \qquad V(\x)\le R \eqsp, \qquad A \in \borelians{\rset^\xdim} \eqsp.
\end{equation}
\end{proposition}

\paragraph{From proposal stability to guided forgetting.}
Since the potentials enter the Feynman--Kac flow at each selection step, stability of the mutation kernels must be complemented by regularity conditions on the guidance potentials. Following standard assumptions from SMC stability theory \citep{delmoral2004feynman,whiteley2012sequential,douc2014longterm}, we impose the following condition.

\begin{assumption}[Guidance potentials]
\label{ass:bounded_normalized_potentials}
For every $k=0,\ldots,N-1$, the ideal and implemented potentials are
strictly positive and bounded:
$$
0<\pot{k}(\x)\le \|\pot{k}\|_\infty<\infty,
\qquad 0 <\approxpot{k}(\x)\le \|\approxpot{k}\|_\infty<\infty,
\qquad \x\in\rset^\xdim \eqsp.
$$
\end{assumption}
Unlike classical strong-mixing assumptions, \Cref{ass:bounded_normalized_potentials} does not require
the potentials to be uniformly bounded away from zero; in particular,
$\inf_{\x}\approxpot{k}(\x)$ may vanish. This point accommodates bounded likelihood and transition factors that decay to zero, as used for motif conditioning and inverse problems in SMCDiff, MCGDiff, Filtering Posterior Sampling, and particle-filtering latent diffusion \citep{trippe2023diffusion,cardoso2024monte,dou2024diffusion,nazemi2024particle}. It also covers the positive twisting and reward potentials used by TDS and Feynman--Kac steering whenever the corresponding scores are bounded above, or are clipped in practice \citep{wu2023practical,singhal2025general}. Our analysis can therefore compare these methods at the level of their Feynman--Kac flows while allowing them to use different proposals and potential constructions.

To track how a perturbation at stage \(k\) propagates to the terminal distribution, we aggregate the remaining guidance weights conditionally on the current state. This defines the normalization underlying the forward-smoothing kernels.
Let $\mathbb E_{k,\x}^{\theta}$ denote expectation for the unweighted reverse chain \eqref{eq:approx_reverse_update} started from $\x$ at time $t_k$. We define the
\emph{continuation function} as
\begin{equation}
\label{eq:approx_backward_functions_main_norm}
\normbackwardfun{k}{N}[\theta][\x]
\eqdef \mathbb E_{k,\x}^{\theta}\!\left[ \prod_{j=k}^{N-1} \frac{\approxpot{j}(\Xbar_{t_j}^{\theta})} {\norminfty{\approxpot{j}}}
\right] \eqsp, \qquad \normbackwardfun{N}{N}[\theta]=\un \eqsp.
\end{equation}
It measures the expected remaining guidance weight of a trajectory
starting from $\x$. The associated terminal \emph{forward-smoothing kernel} is defined, for
an appropriate test function $\psi$, by
\begin{equation}
\label{eq:forward_smoothing_interpretation_main}
\prodfsmk{k}{N}[\theta]\psi(\x)
\eqdef \frac{ \mathbb E_{k,\x}^{\theta}\!\left[ \psi(\Xbar_{t_N}^{\theta}) \prod_{j=k}^{N-1} \frac{\approxpot{j}(\Xbar_{t_j}^{\theta})}{\norminfty{\approxpot{j}}} \right]
}{ \normbackwardfun{k}{N}[\theta][\x] } \eqsp.
\end{equation}
It is the law of the terminal state obtained by propagating from
$\x$ while incorporating all future guidance. In particular, $\approxfilter{N}
= \tau_{\normbackwardfun{k}{N}[\theta]} (\approxfilter{k}) \prodfsmk{k}{N}[\theta]$ (see \Cref{lem:terminal_forward_smoothing_decomposition}).

To transfer proposal stability to the forward-smoothing kernels, we introduce the guidance-adapted Lyapunov function $\twistedlyap{\ell}{N}[\theta](\x) \eqdef
\frac{V(\x)}{\normbackwardfun{\ell}{N}[\theta][\x]}$ for $0\le\ell\le N$. 
\begin{proposition}[Forward-smoothing contraction] \label{lem:finite_horizon_forward_smoothing_stability_main}
Under Assumptions \ref{ass:dissipative_score_main} \ref{ass:bounded_normalized_potentials}
and step-size condition \eqref{eq:block_stepsize_condition}, there exist
$b_N\in(0,1]$ and $\alpha_{k,N}\in(0,1)$, $1\le k\le N$, such that,
for every $0\le\ell\le N$ and probability measures $\mu_1,\mu_2$
with finite $\twistedlyap{\ell}{N}[\theta]$-moments,
\begin{equation}
\label{eq:guided_stability_main}
\vnorm[ \mu_1\prodfsmk{\ell}{N}[\theta] - \mu_2\prodfsmk{\ell}{N}[\theta] ][V] \le \Gamma_{\ell,N} \vnorm[ \mu_1-\mu_2 ][\twistedlyap{\ell}{N}[\theta]] \eqsp,
\end{equation}
where $\Gamma_{\ell,N} \eqdef \tfrac{1}{b_N} \prod_{k=\ell+1}^{N}\alpha_{k,N}$ and $\Gamma_{N,N}\eqdef1$. 
\end{proposition}
The factors $\alpha_{k,N}$ quantify contraction at successive forward-smoothing steps, while $b_N^{-1}$ accounts for conversion to the terminal $V$-norm. The explicit constants and proof are given in \Cref{app:forward_smothing_forgetting}. The learned-score and step-size conditions may alternatively be replaced by \Cref{ass:proposal_harris}. This result provides the forgetting mechanism through which local particle and approximation errors are attenuated sequentially along the guided dynamics, yielding a unified stability view on both errors.

\subsection{Unified particle and approximation errors}

We now leverage the forgetting property of \Cref{lem:finite_horizon_forward_smoothing_stability_main}  to combine particle and approximation errors into a single terminal guarantee. The following mild regularity condition is used to control the diffusion-model approximation terms and is standard in KL convergence analyses of SGMs \citep{conforti_kl}.
\begin{assumption}[Training-data regularity]
\label{ass:sgm_fisher_regularity_main}
    The data distribution admits a strictly positive density
    $p_0\in C^2(\rset^\xdim)$ with finite Fisher information and finite
    fourth moment. 
\end{assumption}
Alternative sufficient conditions on the data distribution and their relation to the learned-score assumptions, are discussed in \Cref{app:data_regularity}.

\begin{theorem}
\label{thm:total_smc_kernel_bias_error}
Let $p\ge1$ and fix the selection grid $\mathcal T^{\rm sel}$. Suppose that Assumptions\ref{ass:dissipative_score_main}, \ref{ass:bounded_normalized_potentials}, \ref{ass:local_lp_bound_moment} (with $q=2p$), \ref{ass:mc_normalization_moments}, \ref{ass:sgm_fisher_regularity_main} together with the step-size condition
\eqref{eq:block_stepsize_condition} hold. Then, for every $\psi$ such that $\|\psi\|_{V,\infty}<\infty$, there exist $C_\ell^{\rm MC}, C_k^{\rm discr},C_k^{\rm net},C_k^{\rm pot}, \Lambda_{k,N}$ finite such that
\begin{align*}
\left\|\mcfilter{N}[\psi]-\pi^\star[\psi]\right\|_{L_p}
&\le \|\psi\|_{V,\infty}\Bigg[ \frac{1}{\sqrt M}
\sum_{\ell=0}^{N}\Gamma_{\ell,N} C_\ell^{\rm MC}
+\Gamma_{0,N}\Lambda_{0,N}\vnorm[p_T-\refmeas][V]\nonumber\\
&\qquad+\sum_{k=1}^{N}\Gamma_{k,N}\Lambda_{k,N}
\left( C_k^{\rm discr}\sqrt h
+C_k^{\rm net}\varepsilon_k^{\rm score}(\theta,h)
+ C_k^{\rm pot}\varepsilon_k^{\rm pot}\right)\Bigg] \eqsp.
\end{align*}
where $\varepsilon_k^{\rm score}(\theta,h)^2 \eqdef \sum_{n=(k-1)J}^{kJ-1} \Delta_n\, \mathbb E_k [ \| \nabla\log p_{T-s_n}(\Xola_{s_n})
- \scorenet[T-s_n][\Xola_{s_n}][\theta] \|^2 ]$ and $\mathbb E_k$ denotes expectation under the exact reverse dynamics \eqref{eq:reverse_sde_main} initialized from $\tau_{\pot{k-1}}(\filter{k-1})$ at time $t_{k-1}$. $\varepsilon_k^{\rm pot}$ denotes the normalized potential mismatch defined in \eqref{eq:def_e_pot}. For $k<N$, $\Gamma_{k,N}=b_N^{-1}\prod_{j=k+1}^{N}\alpha_{j,N}$ with
$\alpha_{j,N}\in(0,1)$, and $\Gamma_{N,N}=1$ is the forgetting factor.
\end{theorem}
All constants in \Cref{thm:total_smc_kernel_bias_error} are quantitatively controlled by the corresponding model, stability, and moment parameters; their definitions and derivations are given in the Appendices \ref{app:forward_smothing_forgetting}, \ref{app:monte_carlo} and \ref{app:kernel_bias}.

\begin{proof}[Sketch of proof]
The bound results from propagating two types of local perturbations through the same forward-smoothing dynamics. Note that,
$$
\left\|\mcfilter{N}[\psi]-\pi^\star[\psi]\right\|_{L_p}
\le \left\|\mcfilter{N}[\psi]-\approxfilter{N}[\psi]\right\|_{L_p}
+\|\psi\|_{V,\infty}\vnorm[\approxfilter{N}-\filter{N}][V] \eqsp.
$$
\paragraph{Particle approximation of the approximate model.} The particle error is a finite sum of local sampling fluctuations, each of order $M^{-1/2}$ under multinomial resampling by conditional Marcinkiewicz--Zygmund inequalities \citep{delmoral2004feynman}. On the unbounded state space, propagating these errors additionally requires Lyapunov moment control and integrability of the random normalization factors induced by future guidance. Such moment and normalization conditions are standard in particle analyses on non-compact spaces \citep{whiteley2012sequential}. Under the particle-integrability conditions of Assumptions \ref{ass:local_lp_bound_moment} (with $q=2p$) and \ref{ass:mc_normalization_moments}, together with the forward-smoothing stability estimate of \Cref{lem:finite_horizon_forward_smoothing_stability_main}, for any measurable $\psi$ such that $\|\psi\|_{V,\infty}<\infty$,
\begin{equation}
\label{eq:particle_error_main}
\left\| \mcfilter{N}[\psi]-\approxfilter{N}[\psi] \right\|_{L_p}
\le \frac{\|\psi\|_{V,\infty}}{\sqrt M} \sum_{\ell=0}^{N} \Gamma_{\ell,N} \mathcal{C}_\ell^{\rm MC}\, \mathcal A_{\ell,N}  \eqsp.
\end{equation}
See \Cref{thm:final_smc_error_bound} for the proof. $\mathcal{C}_\ell^{\rm MC}$ controls the local empirical fluctuation, while $\mathcal A_{\ell,N}$ captures the inverse normalization induced by the remaining guidance; see \Cref{app:detailed_monte_carlo_assumptions} for their definitions.
\paragraph{Approximation bias.}
For $1\le k\le N$, comparing the ideal and implemented updates
from the same input law $\filter{k-1}$ gives the local discrepancy. A telescoping argument and forward-smoothing stability of \Cref{lem:finite_horizon_forward_smoothing_stability_main} yields (see \Cref{prop:kernel_bias_from_forward_forgetting})
\begin{align*}
\vnorm[\approxfilter{N}-\filter{N}][V]
&\le \sum_{k=0}^{N}
\Gamma_{k,N}\Lambda_{k,N}\mathcal B_k^{\rm bias}.
\end{align*}
where $\mathcal B_k^{\rm bias} = \vnorm[ \filter{k} - \tau_{\approxpot{k-1}}(\filter{k-1})\mk{k}[\theta,h]
][V]$ and $\mathcal B_0^{\rm bias} = \vnorm[ \filter{0} - \approxfilter{0}]$. The local discrepancies have three sources: the initialization mismatch $\mathcal B_0^{\rm bias}$, the SGM approximation error $\mathcal B_k^{\rm SGM}$ arising from score estimation and numerical discretization, and the guidance-potential mismatch $\mathcal B_k^{\rm pot}$.
Then, under Assumptions \ref{ass:dissipative_score_main} \ref{ass:bounded_normalized_potentials}
\ref{ass:sgm_fisher_regularity_main} and the step-size condition \eqref{eq:block_stepsize_condition} the following hold:
\begin{enumerate}[leftmargin=*,label=(\roman*),itemsep=1pt]
\item \emph{Initialization}  \citep[Section H.1 in][]{strasman2026forgetting}. $\filter{0}=p_T$ and $\approxfilter{0}=\refmeas$, hence
$\mathcal B_0^{\rm bias} = \vnorm[p_T-\refmeas][V]$. In particular, when
$\isvp>0$, $\mathcal B_0^{\rm bias}
\lesssim \exp\!\left\{ -\isvp\int_0^T\beta(s)\,\rmd s \right\}$ ;

\item \emph{SGM kernel approximation} (\Cref{lem:block_sgm_bias}).  For every $1\le k\le N$, there exist $C_k^{\rm discr},C_k^{\rm net}<\infty$ independent of $h$ such that
$\mathcal B_k^{\rm SGM} \le C_k^{\rm discr}\sqrt h + C_k^{\rm net}\varepsilon_k^{\rm score}(\theta,h) \eqsp.$

\item \emph{Potential mismatch} (\Cref{prop:potential_mismatch_control}). For $1\le k\le N$,
\begin{equation} \label{eq:def_e_pot}
    \mathcal B_k^{\rm pot} \le \varepsilon_k^{\rm pot} \left\{
    \filter{k-1}
    \left[
    (1+K_k^Q+\lambda_k^Q V)^2
    \right]
    \right\}^{1/2}
    \eqdef \varepsilon_k^{\rm pot} C_k^{\rm pot} \eqsp,
\end{equation}
with $(\varepsilon_k^{\rm pot})^2 \eqdef \filter{k-1}\!\left[
\left( \frac{\approxpot{k-1}}{\filter{k-1}[\approxpot{k-1}]} - \frac{\pot{k-1}} {\filter{k-1}[\pot{k-1}]} \right)^2\right]$.
\end{enumerate}
Combining these local estimates and using that $\mathcal B_k^{\rm bias} \le \mathcal B_k^{\rm pot} + \mathcal B_k^{\rm SGM}$ finishes the proof.
\end{proof}

\subsection{Interpretation and controlled setting illustration}

\Cref{thm:total_smc_kernel_bias_error} separates finite-particle fluctuations from the approximation bias, while making explicit how both are propagated by the guided dynamics. In contrast with many classical SMC error bounds formulated for bounded test functions, \Cref{thm:total_smc_kernel_bias_error} applies to test functions with at most quadratic growth. As a consequence, it yields non-asymptotic guarantees for estimating quantities such as the mean and second moments of the target distribution.

The Monte Carlo contribution decreases at the standard rate $M^{-1/2}$, whereas initialization, score approximation, numerical discretization, and potential mismatch persist even in the infinite-particle limit. Increasing $M$ therefore improves the approximation of the implemented flow $\approxfilter{N}$, but does not by itself bring $\approxfilter{N}$ closer to the desired target $\pi^\star$. These terms must instead be reduced by improving the reference initialization, score model, numerical discretization, or guidance construction. In particular, the bound yields a discretization contribution of order $\sqrt h$ and, in the Ornstein--Uhlenbeck setting with $\isvp>0$, an initialization error decaying at rate of order $\rme^{- \isvp T}$.

Moreover, a perturbation introduced at stage $k$ contributes to the terminal error through the forgetting factor $\Gamma_{k,N}$. Hence, local errors of comparable magnitude may have different terminal effects depending on where they occur along the SMC trajectory. In this regard, the implemented potentials $\approxpot{k}$ play a dual role: they determine the local mismatch $\varepsilon_k^{\rm pot}$ but also affect the stability through $\Gamma_{k,N}$.

Finally, the approximation terms are conditioning-dependent. The score error $\varepsilon_k^{\rm score}(\theta,h)$ is evaluated along reverse trajectories initialized from the ideal selected law, while $\varepsilon_k^{\rm pot}$ compares normalized ideal and implemented potentials under the true filtering distribution. The bound therefore measures approximation quality in the regions emphasized by the guidance. In other words, even though the proposal is built from an unconditional diffusion model, its contribution to the guarantee is determined by its accuracy along the trajectories explored by the conditional SMC flow.

\paragraph{Numerical illustration.} We illustrate these mechanisms in Gaussian-mixture toy problem, chosen because the scores, reverse kernels, and targets are available in closed form. The purpose is not to reassess the practical performance of diffusion-based SMC, which has been studied extensively in the previously cited literature, but to isolate the error components appearing in the theorem. We use a four-component anisotropic mixture conditioned on a noisy observation of its first coordinate (see details in \Cref{app:diffusion_exp_SMC}).  The left panel of \Cref{fig:num_2d_theorem_illustration_main} shows the prior and posterior, with $\psi_{\mathcal R}=\mathbf 1_{\mathcal R}$ (the red rectangle) indicating a region around one posterior mode; its exact posterior probability is computed from the closed-form Gaussian mixture.

At a fixed reverse-time grid, we compare SMC using exact reverse kernels, Euler kernels driven by the exact score, and Euler kernels with a controlled score perturbation. The exact kernels remove score and discretization errors from the proposals; the two Euler variants introduce discretization error and, in the latter case, an additional score error. The middle panel reports the empirical $L_2$ error over independent runs as the particle count $M$ increases. It illustrates Monte Carlo decay compatible with the $M^{-1/2}$ scale and error plateaus when approximation bias dominates: more particles improve the approximation of the implemented flow without removing its bias.


The right panel examines local error propagation. We replace only the $\ell$-th exact reverse kernel by a perturbed-score Euler kernel, keeping all other kernels exact, and vary $\ell$. Plotting the resulting terminal change in the modal-region probability against $N-\ell$ shows that perturbations introduced earlier in the SMC flow, and therefore followed by more guided transitions, are more strongly attenuated. This illustrates the guided forgetting mechanism quantified by the propagation factors $\Gamma_{\ell,N}$. Full specifications, additional diagnostics, and a dimension-$50$ benchmark are given in \Cref{app:diffusion_exp_SMC}.

\begin{figure}[t]
    \centering
    \includegraphics[width=.39\linewidth]{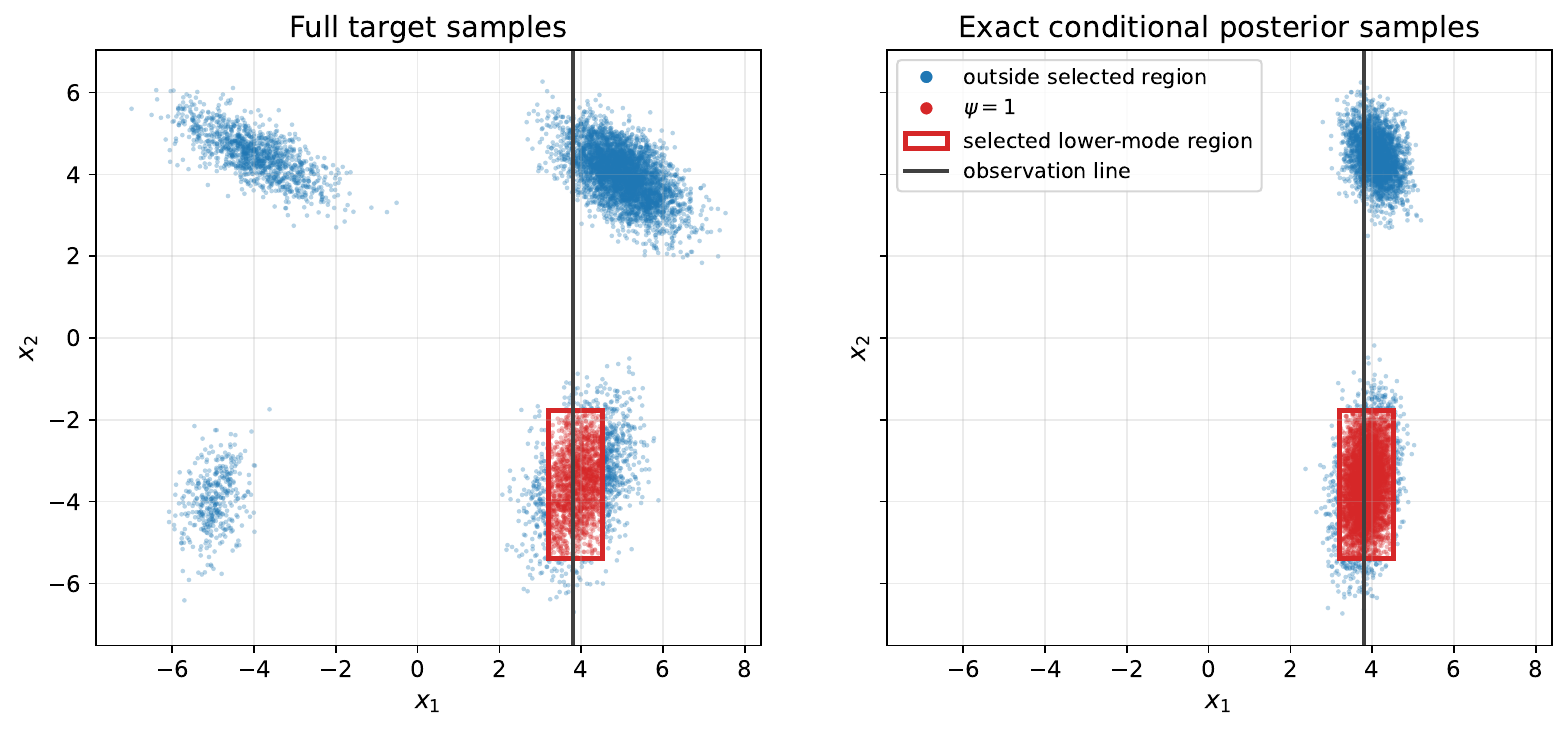}
    \hfill
    \includegraphics[width=.29\linewidth]{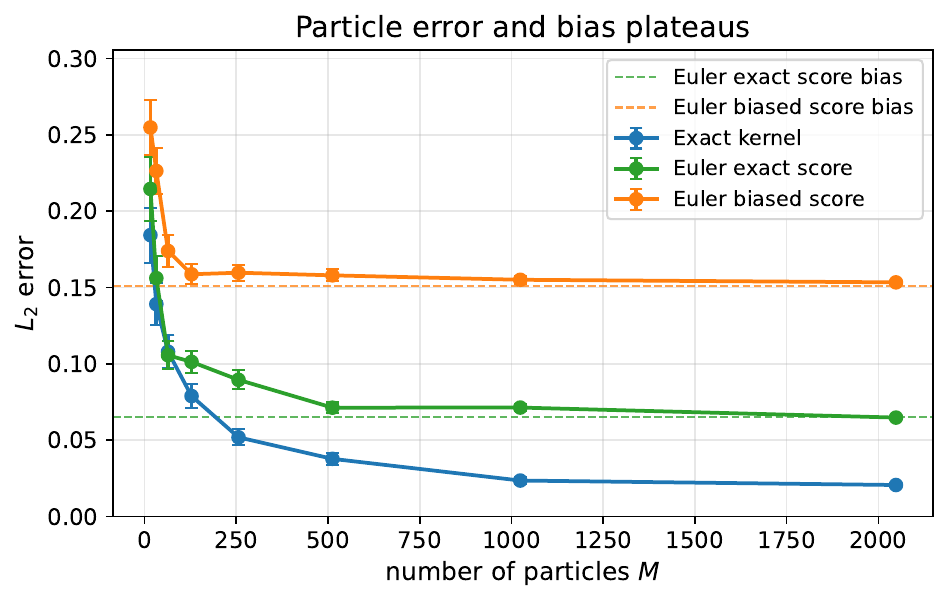}
    \hfill
    \includegraphics[width=.29\linewidth]{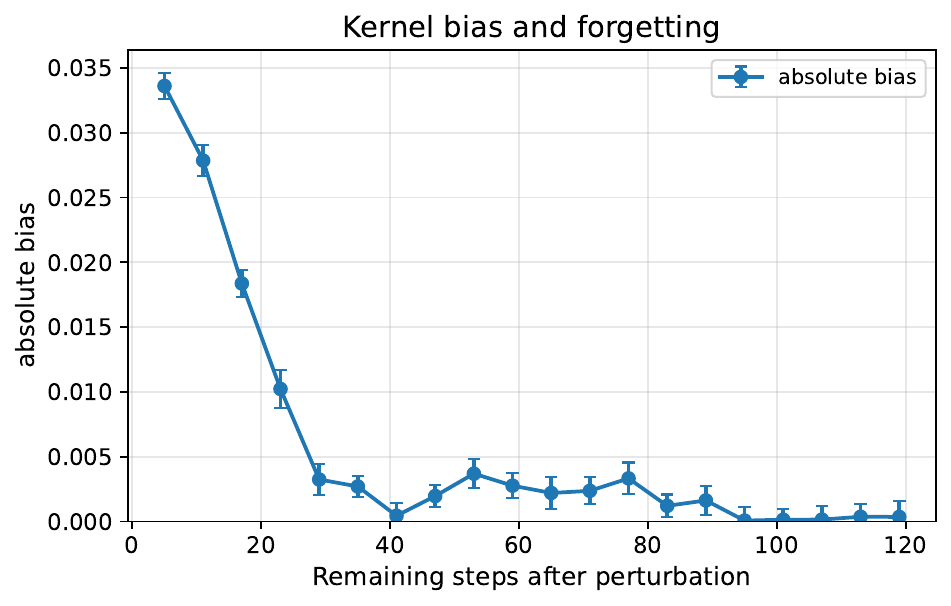}
    \caption{
    Controlled Gaussian-mixture illustration.
    Left: prior and conditional posterior, with the red
    rectangle defining the observable $\psi_{\mathcal R}$.
    Middle: empirical $L_2$ error versus particle count $M$,
    illustrating Monte Carlo decay and approximation-bias
    plateaus.
    Right: terminal effect of a single local kernel
    perturbation versus the number of remaining reverse
    steps; larger values correspond to earlier perturbations.
    Error bars show standard errors across independent runs.
    }
    \label{fig:num_2d_theorem_illustration_main}
\end{figure}

\section{Discussion}

Our analysis applies to diffusion-based SMC samplers including SMCDiff, TDS, MCGDiff, filtering posterior sampling, and particle-filtering latent diffusion \citep{trippe2023diffusion,wu2023practical,cardoso2024monte,dou2024diffusion,nazemi2024particle}. It separates finite-particle error from approximation bias due to learned dynamics, discretization, and guidance, while propagating both through the same guided forgetting mechanism.

In this work, the choice $V(\x)=\|\x\|^2$ is natural for our learned-score assumptions, but the abstract Feynman--Kac analysis allows geometry-adapted Lyapunov functions, potentially yielding sharper stability bounds or accommodating different observables. Finally, the implemented potentials affect both local approximation errors and the stability factors governing their propagation. Turning this interaction into practical criteria for guidance design is a natural direction for future work.

\newpage

\bibliographystyle{apalike}
\bibliography{ref}

\begin{thebibliography}{}

\bibitem[Bakry et~al., 2014]{bakry2014analysis}
Bakry, D., Gentil, I., and Ledoux, M. (2014).
\newblock {\em Analysis and geometry of Markov diffusion operators}.
\newblock Springer.

\bibitem[Benton et~al., 2024]{benton2024nearly}
Benton, J., Bortoli, V.~D., Doucet, A., and Deligiannidis, G. (2024).
\newblock Nearly {$d$}-linear convergence bounds for diffusion models via stochastic localization.
\newblock In {\em The Twelfth International Conference on Learning Representations}.

\bibitem[Bergmeister et~al., 2026]{bergmeister2026reinforce}
Bergmeister, A., Jegelka, S., N{\"u}sken, N., Domingo-Enrich, C., and Pidstrigach, J. (2026).
\newblock Reinforce adjoint matching: Scaling rl post-training of diffusion and flow-matching models.
\newblock {\em arXiv preprint arXiv:2605.10759}.

\bibitem[Black et~al., 2024]{black2024training}
Black, K., Janner, M., Du, Y., Kostrikov, I., and Levine, S. (2024).
\newblock Training diffusion models with reinforcement learning.
\newblock In {\em International Conference on Learning Representations}, volume 2024, pages 4965--4987.

\bibitem[Cardoso et~al., 2024]{cardoso2024monte}
Cardoso, G., Janati El~Idrissi, Y., Le~Corff, S., and Moulines, E. (2024).
\newblock Monte carlo guided denoising diffusion models for bayesian linear inverse problems.
\newblock In {\em International Conference on Learning Representations}.

\bibitem[Cattiaux et~al., 2023]{cattiaux}
Cattiaux, P., Conforti, G., Gentil, I., and L{\'e}onard, C. (2023).
\newblock {Time reversal of diffusion processes under a finite entropy condition}.
\newblock {\em {Annales de l'Institut Henri Poincar{\'e} (B) Probabilit{\'e}s et Statistiques}}, 59(4):1844--1881.

\bibitem[Chen et~al., 2023]{chen2023sampling}
Chen, S., Chewi, S., Li, J., Li, Y., Salim, A., and Zhang, A. (2023).
\newblock Sampling is as easy as learning the score: theory for diffusion models with minimal data assumptions.
\newblock In {\em The Eleventh International Conference on Learning Representations}.

\bibitem[Chopin and Papaspiliopoulos, 2020]{chopin2020introduction}
Chopin, N. and Papaspiliopoulos, O. (2020).
\newblock {\em An Introduction to Sequential Monte Carlo}.
\newblock Springer Series in Statistics. Springer.

\bibitem[Conforti et~al., 2025a]{conforti2025non}
Conforti, G., Durmus, A., Pham, L.-T.-N., and Raoul, G. (2025a).
\newblock Non-asymptotic convergence of discrete diffusion models: Masked and random walk dynamics.
\newblock {\em arXiv preprint arXiv:2512.00580}.

\bibitem[Conforti et~al., 2025b]{conforti_kl}
Conforti, G., Durmus, A., and Silveri~Gentiloni, M. (2025b).
\newblock Kl convergence guarantees for score diffusion models under minimal data assumptions.
\newblock {\em SIAM Journal on Mathematics of Data Science}, 7(1):86--109.

\bibitem[Del~Moral, 2004]{delmoral2004feynman}
Del~Moral, P. (2004).
\newblock {\em Feynman-Kac Formulae: Genealogical and Interacting Particle Systems with Applications}.
\newblock Probability and Its Applications. Springer.

\bibitem[Del~Moral et~al., 2006]{del2006sequential}
Del~Moral, P., Doucet, A., and Jasra, A. (2006).
\newblock Sequential monte carlo samplers.
\newblock {\em Journal of the Royal Statistical Society Series B: Statistical Methodology}, 68(3):411--436.

\bibitem[Domingo-Enrich et~al., 2025]{domingo2025adjoint}
Domingo-Enrich, C., Drozdzal, M., Karrer, B., and Chen, R. T.~Q. (2025).
\newblock Adjoint matching: Fine-tuning flow and diffusion generative models with memoryless stochastic optimal control.
\newblock In {\em International Conference on Learning Representations}, pages 53791--53846.

\bibitem[Domingo-Enrich et~al., 2026]{domingo2026unified}
Domingo-Enrich, C., Du, Y., and Albergo, M.~S. (2026).
\newblock A unified perspective on fine-tuning and sampling with diffusion and flow models.
\newblock {\em arXiv preprint arXiv:2605.00229}.

\bibitem[Domingo-Enrich and Han, 2026]{domingo2026smp}
Domingo-Enrich, C. and Han, J. (2026).
\newblock Adjoint matching through the lens of the stochastic maximum principle in optimal control.
\newblock {\em arXiv preprint arXiv:2604.08580}.

\bibitem[Dou and Song, 2024]{dou2024diffusion}
Dou, Z. and Song, Y. (2024).
\newblock Diffusion posterior sampling for linear inverse problem solving: A filtering perspective.
\newblock In {\em International Conference on Learning Representations}, volume 2024, pages 12069--12095.

\bibitem[Douc et~al., 2014]{douc2014longterm}
Douc, R., Moulines, {\'E}., and Olsson, J. (2014).
\newblock Long-term stability of sequential monte carlo methods under verifiable conditions.
\newblock {\em The Annals of Applied Probability}, 24(5):1767--1802.

\bibitem[Douc et~al., 2018]{douc2018markov}
Douc, R., Moulines, E., Priouret, P., and Soulier, P. (2018).
\newblock {\em Markov Chains}.
\newblock Springer {{Series}} in {{Operations Research}} and {{Financial Engineering}}. Springer, Cham.

\bibitem[Doucet et~al., 2001]{doucet2001sequential}
Doucet, A., De~Freitas, N., and Gordon, N.~J. (2001).
\newblock {\em Sequential Monte Carlo methods in practice}.
\newblock Springer.

\bibitem[Fan et~al., 2023]{fan2023dpok}
Fan, Y., Watkins, O., Du, Y., Liu, H., Ryu, M., Boutilier, C., Abbeel, P., Ghavamzadeh, M., Lee, K., and Lee, K. (2023).
\newblock Dpok: Reinforcement learning for fine-tuning text-to-image diffusion models.
\newblock {\em Advances in neural information processing systems}, 36:79858--79885.

\bibitem[Gentiloni~Silveri and Ocello, 2025]{silveri205beyond}
Gentiloni~Silveri, M. and Ocello, A. (2025).
\newblock Beyond log-concavity and score regularity: Improved convergence bounds for score-based generative models in w2-distance.
\newblock In {\em Proceedings of the 42nd International Conference on Machine Learning}, volume 267 of {\em Proceedings of Machine Learning Research}, pages 55631--55656. PMLR.

\bibitem[Gloaguen et~al., 2022]{gloaguen2022pseudo}
Gloaguen, P., Le~Corff, S., and Olsson, J. (2022).
\newblock A pseudo-marginal sequential monte carlo online smoothing algorithm.
\newblock {\em Bernoulli}, 28(4):2606--2633.

\bibitem[Hairer and Mattingly, 2008]{HairerMattingly2008}
Hairer, M. and Mattingly, J.~C. (2008).
\newblock Yet another look at harris' ergodic theorem for markov chains.
\newblock {\em arXiv preprint arXiv:0810.2777}.

\bibitem[Haussmann and Pardoux, 1986]{haussmann1986time}
Haussmann, U.~G. and Pardoux, E. (1986).
\newblock Time reversal of diffusions.
\newblock {\em The Annals of Probability}, 14(4):1188--1205.

\bibitem[Havens et~al., 2025]{havens2025adjoint}
Havens, A., Miller, B.~K., Yan, B., Domingo-Enrich, C., Sriram, A., Wood, B., Levine, D., Hu, B., Amos, B., Karrer, B., Fu, X., Liu, G.-H., and Chen, R. T.~Q. (2025).
\newblock Adjoint sampling: Highly scalable diffusion samplers via adjoint matching.
\newblock {\em arXiv preprint arXiv:2504.11713}.

\bibitem[Ho et~al., 2020]{ho2020denoising}
Ho, J., Jain, A., and Abbeel, P. (2020).
\newblock Denoising diffusion probabilistic models.
\newblock In {\em Advances in Neural Information Processing Systems (NeurIPS)}, pages 6840--6851.

\bibitem[Ho and Salimans, 2021]{ho2021classifierfree}
Ho, J. and Salimans, T. (2021).
\newblock Classifier-free diffusion guidance.
\newblock In {\em NeurIPS 2021 Workshop on Deep Generative Models and Downstream Applications}.

\bibitem[Liu et~al., 2025]{liu2025valuegradient}
Liu, Z., Xiao, T., Domingo-Enrich, C., Liu, W., and Zhang, D. (2025).
\newblock Value gradient guidance for flow matching alignment.
\newblock In {\em Advances in Neural Information Processing Systems}, volume~38, pages 11372--11404. Curran Associates, Inc.

\bibitem[Marion et~al., 2025]{marion2025implicit}
Marion, P., Korba, A., Bartlett, P., Blondel, M., De~Bortoli, V., Doucet, A., Llinares-L{\'o}pez, F., Paquette, C., and Berthet, Q. (2025).
\newblock Implicit diffusion: Efficient optimization through stochastic sampling.
\newblock In {\em Proceedings of The 28th International Conference on Artificial Intelligence and Statistics}, volume 258 of {\em Proceedings of Machine Learning Research}, pages 1999--2007. PMLR.

\bibitem[Nazemi et~al., 2024]{nazemi2024particle}
Nazemi, A., Sepanj, M.~H., Pellegrino, N., Czarnecki, C., and Fieguth, P. (2024).
\newblock Particle-filtering-based latent diffusion for inverse problems.
\newblock {\em arXiv preprint arXiv:2408.13868}.

\bibitem[Pham et~al., 2025]{pham2025discrete}
Pham, L.-T.-N., Shariatian, D., Ocello, A., Conforti, G., and Durmus, A.~O. (2025).
\newblock Discrete markov probabilistic models: An improved discrete score-based framework with sharp convergence bounds under minimal assumptions.
\newblock In {\em Forty-second International Conference on Machine Learning}.

\bibitem[Pidstrigach et~al., 2025]{pidstrigach2025malliavin}
Pidstrigach, J., Baker, E., Domingo-Enrich, C., Deligiannidis, G., and N{\"u}sken, N. (2025).
\newblock Conditioning diffusions using malliavin calculus.
\newblock {\em arXiv preprint arXiv:2504.03461}.

\bibitem[Rombach et~al., 2022]{rombach2022highresolution}
Rombach, R., Blattmann, A., Lorenz, D., Esser, P., and Ommer, B. (2022).
\newblock High-resolution image synthesis with latent diffusion models.
\newblock In {\em Proceedings of the IEEE/CVF Conference on Computer Vision and Pattern Recognition (CVPR)}, pages 10684--10695.

\bibitem[Rudin, 1987]{rudin1987real}
Rudin, W. (1987).
\newblock {\em Real and {{Complex Analysis}}}.
\newblock McGraw-Hill, Singapore.

\bibitem[Saharia et~al., 2022]{saharia2022palette}
Saharia, C., Chan, W., Chang, H., Lee, C.~A., Ho, J., Salimans, T., Fleet, D.~J., and Norouzi, M. (2022).
\newblock Palette: Image-to-image diffusion models.
\newblock In {\em ACM SIGGRAPH 2022 Conference Proceedings}, New York, NY, USA. Association for Computing Machinery.

\bibitem[Singhal et~al., 2025]{singhal2025general}
Singhal, R., Horvitz, Z., Teehan, R., Ren, M., Yu, Z., McKeown, K., and Ranganath, R. (2025).
\newblock A general framework for inference-time scaling and steering of diffusion models.
\newblock In {\em Proceedings of the 42nd International Conference on Machine Learning}, volume 267 of {\em Proceedings of Machine Learning Research}, pages 55810--55827. PMLR.

\bibitem[Sohl-Dickstein et~al., 2015]{dickstein2015}
Sohl-Dickstein, J., Weiss, E., Maheswaranathan, N., and Ganguli, S. (2015).
\newblock Deep unsupervised learning using nonequilibrium thermodynamics.
\newblock In Bach, F. and Blei, D., editors, {\em Proceedings of the 32nd International Conference on Machine Learning}, volume~37 of {\em Proceedings of Machine Learning Research}, pages 2256--2265, Lille, France. PMLR.

\bibitem[Song and Ermon, 2019]{song2019generative}
Song, Y. and Ermon, S. (2019).
\newblock Generative modeling by estimating gradients of the data distribution.
\newblock In {\em Advances in Neural Information Processing Systems (NeurIPS)}, pages 11918--11930.

\bibitem[Song et~al., 2021]{song2021score}
Song, Y., Sohl-Dickstein, J., Kingma, D.~P., Kumar, A., Ermon, S., and Poole, B. (2021).
\newblock Score-based generative modeling through stochastic differential equations.
\newblock In {\em International Conference on Learning Representations}.

\bibitem[Strasman et~al., 2026]{strasman2026forgetting}
Strasman, S., Cardoso, G., Le~Corff, S., Lemaire, V., and Ocello, A. (2026).
\newblock On forgetting and stability of score-based generative models.
\newblock {\em arXiv preprint arXiv:2601.21868}.

\bibitem[Strasman et~al., 2025a]{strasman2025an}
Strasman, S., Ocello, A., Boyer, C., Corff, S.~L., and Lemaire, V. (2025a).
\newblock An analysis of the noise schedule for score-based generative models.
\newblock {\em Transactions on Machine Learning Research}.

\bibitem[Strasman et~al., 2025b]{strasman2025wasserstein}
Strasman, S., Surendran, S., Boyer, C., Corff, S.~L., Lemaire, V., and Ocello, A. (2025b).
\newblock Wasserstein convergence of critically damped langevin diffusions.
\newblock In {\em The Thirty-ninth Annual Conference on Neural Information Processing Systems}.

\bibitem[Trippe et~al., 2023]{trippe2023diffusion}
Trippe, B.~L., Yim, J., Tischer, D., Baker, D., Broderick, T., Barzilay, R., and Jaakkola, T.~S. (2023).
\newblock Diffusion probabilistic modeling of protein backbones in 3d for the motif-scaffolding problem.
\newblock In {\em The Eleventh International Conference on Learning Representations}.

\bibitem[Virmaux and Scaman, 2018]{virmaux2018lipschitz}
Virmaux, A. and Scaman, K. (2018).
\newblock Lipschitz regularity of deep neural networks: analysis and efficient estimation.
\newblock {\em Advances in neural information processing systems}, 31.

\bibitem[Whiteley, 2012]{whiteley2012sequential}
Whiteley, N. (2012).
\newblock Sequential monte carlo samplers: error bounds and insensitivity to initial conditions.
\newblock {\em Stochastic Analysis and Applications}, 30(5):774--798.

\bibitem[Wu et~al., 2023a]{wu2023practical}
Wu, L., Trippe, B., Naesseth, C., Blei, D., and Cunningham, J.~P. (2023a).
\newblock Practical and asymptotically exact conditional sampling in diffusion models.
\newblock {\em Advances in Neural Information Processing Systems}, 36:31372--31403.

\bibitem[Wu et~al., 2023b]{wu2023human}
Wu, X., Sun, K., Zhu, F., Zhao, R., and Li, H. (2023b).
\newblock Human preference score: Better aligning text-to-image models with human preference.
\newblock In {\em 2023 IEEE/CVF International Conference on Computer Vision (ICCV)}, pages 2096--2105. IEEE.

\end{thebibliography}

\newpage

\appendix




\tableofcontents

\newpage

 \section{Additional discussion}

\subsection{Alternative data regularity conditions}
\label{app:data_regularity}

Assumption \ref{ass:sgm_fisher_regularity_main} is one possible sufficient condition for convergence bounds in SGMs \citep{conforti_kl}. An alternative path is taken in  \citet{strasman2026forgetting}, where $\pidata$ is assumed to be in $C^2(\rset^\xdim)$ such that, there exist $\gamma_{\rm data}>\isvp/2$, $\kappa_{\rm data}\ge0$, $C_{\rm data}<\infty$, and $r\ge1$,
\begin{align*}
    \dotprod{\x}{\nabla\log p_0(\x)}
    &\le -\gamma_{\rm data}\|\x\|^2+\kappa_{\rm data}
    \eqsp,
    \\
    \left\|\nabla^2\log p_0(\x)\right\|_{\rm F}
    &\le C_{\rm data}(1+\|\x\|^r) \eqsp.
\end{align*}
These conditions, as proven in \citet{strasman2026forgetting}, propagate to the forward marginals $p_t$ and their scores, and therefore guarantee that the ideal reverse dynamics satisfy \eqref{eq:learned_score_conditions_main}. They also impose stability directly on the true data score, while \Cref{ass:dissipative_score_main} asks the learned score to preserve analogous tail dissipativity and growth properties. In this regard, the alternative assumption allows us to interpret \Cref{ass:dissipative_score_main} as requiring the network approximation to preserve the dissipativity structure of the underlying diffusion.




\subsection{Related work}

\paragraph{Diffusion-based SMC.}
Several diffusion-based SMC methods provide principled constructions for
post-hoc conditional sampling. In particular, \citet{wu2023practical}
establish asymptotic exactness, as the number of particles increases, for
the conditional distribution induced by a prescribed diffusion model,
while \citet{cardoso2024monte} develop an SMC construction for Bayesian
linear inverse problems. These guarantees concern a fixed diffusion model
and therefore do not, by themselves, quantify the discrepancy between the
implemented sampler and a target distribution defined from the true data
distribution. Our analysis is complementary: it \textit{separates the
finite-particle error from the bias} introduced by initialization, score
approximation, numerical discretization, and approximate guidance
potentials.

\paragraph{Reward fine-tuning of diffusion and flow models.}
A related line of work modifies the parameters of a pretrained generative
model so that its terminal distribution favors high-reward samples.
Existing approaches include supervised fine-tuning on reward-filtered
datasets \citep{wu2023human}, reinforcement-learning objectives
\citep{fan2023dpok,black2024training}, and differentiation through the
iterative sampling procedure; see \citet{marion2025implicit} and the
references therein. The latter formulation views fine-tuning as an
optimization problem whose inner variable is the distribution generated
by an iterative sampler. Implicit Diffusion develops a single-loop method
for this problem, avoiding the memory cost of differentiating through a
complete sampling trajectory \citep{marion2025implicit}.

More recent approaches formulate reward fine-tuning and sampling from
exponentially tilted distributions as stochastic optimal-control
problems. Adjoint Matching turns this control problem into a memoryless
regression objective for diffusion and flow models
\citep{domingo2025adjoint}, while Adjoint Sampling extends this viewpoint
to scalable amortized sampling from unnormalized densities
\citep{havens2025adjoint}. Subsequent works study reinforcement-based
adjoint matching \citep{bergmeister2026reinforce}, provide a unified view
of score-matching, adjoint-based, and nonequilibrium objectives
\citep{domingo2026unified}, and derive adjoint matching from the
stochastic maximum principle \citep{domingo2026smp}. Value-gradient
guidance similarly learns the correction to a pretrained flow by matching
it to the gradient of a value function \citep{liu2025valuegradient}.

Our setting differs in that the pretrained score network is kept fixed.
Guidance is introduced at inference time through particle reweighting,
resampling, and approximate mutation kernels. Consequently, our goal is
not to analyze convergence of a parameter-optimization procedure, but to
quantify the sampling error of the resulting guided particle system. The
two perspectives are nevertheless related through their common tilted
target: fine-tuning seeks to amortize the tilted distribution into new
model parameters, whereas diffusion SMC approximates it directly at
inference time.

\paragraph{Doob transforms and guided dynamics.}
For a reference Markov process and a positive terminal reward, the ideal
guided process can be expressed through a Doob $h$-transform, where
$h$ is the conditional expectation of the remaining reward. In stochastic
control formulations, a useful point of view to analyse SGM convergence bounds \citep{conforti_kl, pham2025discrete,conforti2025non,silveri205beyond}, $\nabla\log h$ determines the optimal drift
correction; in our Feynman--Kac formulation, the analogous continuation
function determines the forward-smoothing kernels and the propagation of
local errors. Adjoint- and value-based methods learn or approximate this
control field
\citep{domingo2025adjoint,havens2025adjoint,liu2025valuegradient,
domingo2026unified,domingo2026smp}, while Malliavin-based conditioning
provides an alternative way to construct controls when terminal rewards
are nonsmooth or singular \citep{pidstrigach2025malliavin}. Diffusion SMC
instead approximates the same change of measure through a population of
weighted particles. Approximate potentials may therefore be interpreted
as tractable surrogates for the ideal continuation function, while the
coefficients $\Gamma_{k,N}$ in our analysis quantify the stability of the
resulting approximate transform.

\paragraph{Feynman--Kac stability and perturbation theory.}
At the level of general Feynman--Kac theory,
\citet{whiteley2012sequential} establishes finite-particle stability and
insensitivity to initialization for prescribed SMC models on noncompact
state spaces. Related perturbation questions arise in
\citet{gloaguen2022pseudo}, who analyze online smoothing when
unnormalized transition densities are replaced by possibly biased
estimators. In our setting, the perturbations have a specific generative
origin: they arise from a learned score, numerical approximation of the
reverse diffusion, and approximate guidance potentials. Our results
describe how these errors are transported jointly by the guided
Feynman--Kac flow.

\section{Notation and general Feynman--Kac framework}
\label{app:fk_framework}

\subsection{Notation summary}

\begin{table}[ht]
    \centering
    \normalsize
    \renewcommand{\arraystretch}{1.1}
    \setlength{\tabcolsep}{4pt}

    \begin{tabularx}{\textwidth}{@{}l>{\raggedright\arraybackslash}X@{}}
    \toprule
    Notation & Meaning \\
    \midrule
    $\mathcal T^{\rm EM}$ & Numerical Euler--Maruyama grid $\{s_n\}_{n=0}^{NJ}$, with $t_k=s_{kJ}$ \\ 
    $h$ & Numerical time step, $h=T/(NJ)$ \\ 
    $\Delta_n$ & Integrated noise increment, $\Delta_n=\int_{s_n}^{s_{n+1}}\bwdnoisesch{u}\,\rmd u$ \\    
    $\mk{k}$ & Ideal mutation kernel from time $t_{k-1}$ to time $t_k$ \\
    $\mk{k}[\theta,h]$ & Approximate mutation kernel from time $t_{k-1}$ to time $t_k$ \\
    $\pot{k}$ & Ideal potential at time $t_k$ \\
    $\approxpot{k}$ & Implemented potential at time $t_k$ \\
    $\potnorm{k}$ & Normalized potential, $\potnorm{k}=\approxpot{k}/\|\approxpot{k}\|_\infty$ \\
    $\tau_g$ & Tilting map by $g$, $\tau_g(\nu)[\psi]=\nu[g\psi]/\nu[g]$ \\
    $\unpredictive{k}$ & Ideal unnormalized selection--mutation kernel, $\unpredictive{k}=\pot{k-1}\mk{k}$ \\
    $\unpredictive{k}[\theta]$ & Approximate unnormalized selection--mutation kernel, $\unpredictive{k}[\theta]=\approxpot{k-1}\mk{k}[\theta,h]$ \\
    $\filtertransform{k}$ & Ideal normalized Feynman--Kac transform, $\filtertransform{k}(\nu)=\nu\unpredictive{k}/(\nu\unpredictive{k}\un)$ \\
    $\filtertransform{k}[\theta]$ & Approximate normalized Feynman--Kac transform, $\filtertransform{k}[\theta](\nu)=\nu\unpredictive{k}[\theta]/(\nu\unpredictive{k}[\theta]\un)$ \\
    $\filter{k}$ & Ideal Feynman--Kac flow at time $t_k$ \\
    $\approxfilter{k}$ & Approximate Feynman--Kac flow at time $t_k$ \\
    $\mcfilter{k}$ & Empirical particle approximation of $\approxfilter{k}$ \\
    $\tildemcfilter{k}$ & Selected empirical measure after reweighting/resampling at time $t_k$ \\
    $\normbackwardfun{k}{\ell}[\theta]$ & Normalized continuation function from time $k$ to terminal time $\ell$ \\
    $\fsmk{k}{\ell}[\theta]$ & Forward-smoothing kernel from time $k-1$ to time $k$, associated with terminal time $\ell$ \\
    $\prodfsmk{k}{\ell}[\theta]$ & Product of forward-smoothing kernels from time $k$ to time $\ell$ \\
    $\twistedlyap{k}{N}[\theta]$ & Twisted Lyapunov function, $V/\normbackwardfun{k}{N}[\theta]$ \\
    $\Gamma_{\ell,N}$ & Forward-smoothing stability coefficient from time $\ell$ to terminal time $N$ \\
    $\Lambda_{\ell,N}$ & Normalization factor appearing in the deterministic kernel-bias bound \\
    $\mathcal B_k^{\rm bias}$ & Local discrepancy between the ideal and implemented Feynman--Kac updates initialized from the same ideal law $\filter{k-1}$ \\ 
    $\mathcal B_k^{\rm SGM}$ & Local SGM approximation error due to score approximation and numerical discretization \\ 
    $\mathcal B_k^{\rm pot}$ & Local discrepancy due to mismatch between the ideal and implemented potentials \\

    $\eta_\ell^{M,\theta}$ & One-step predictive law before sampling at time $\ell$ in the particle system \\
    $\mathcal A_{\ell,N}^{(p)}$ & Monte Carlo normalization moment coefficient \\
    $\mathcal{C}_\ell^{\rm MC}$ & Local Monte Carlo fluctuation constant \\
    $w_k^{(i)}$ & Normalized particle weight at time $t_k$ \\
    $\X_k^{\theta,(i)}$ & Particle $i$ at time $t_k$ \\
    $\widetilde \X_k^{\theta,(i)}$ & Selected particle $i$ after resampling at time $t_k$ \\
    $\mathcal F_k^\theta$ & Particle filtration up to time $t_k$ \\
    $\widetilde{\mathcal F}_k^\theta$ & Filtration after selection/resampling at time $t_k$ \\
    \bottomrule
    \end{tabularx}

    \caption{Summary of the main notation.}
    \label{tab:notation_summary}
\end{table}

\subsection{Measures, kernels, and weighted norms}

Let $\measures{\rset^d}$ (resp. $\probmeasures{\rset^d}$) denote the set of finite (positive) measures (resp. probability measures) on the measurable space $(\rset^{d}, \borelians{\rset^d})$.
We define $\measfuns{\rset^d}$ as the set of Borel measurable functions from $\rset^d \rightarrow \rset$.
For every $(\mu, h) \in \measures{\rset^d}\times \measfuns{\rset^d}$ with $\mu [|h|] < \infty$ we define $\mu [h] \eqdef \int h(x) \mu(\rmd x) \in \rset$. We also define, for a set of \emph{particles} $\chunk{\particle}{0}{\npart-1} \in \rset^{\npart\xdim}$, the occupation (or empirical) measure supported on these particles as
\begin{align*}
    \empmeas{\chunk{\particle}{0}{\npart-1}}[\rmd \x] \eqdef \npart^{-1}\sum_{i=0}^{\npart-1}\dirac{\particle_i}[\rmd \x]\eqsp, 
\end{align*} 
where, for $\x\in \rset^\xdim$,
$$
    \dirac{\x}: \borelians{\rset^\xdim} \ni \set{A} \rightarrow
    \begin{cases}
        1 & \x \in \set{A} \eqsp, \\
        0 & \x \notin \set{A}\eqsp.
    \end{cases}
$$

We call a function $\kernel{K} : \rset^{d} \times \borelians{\rset^d} \to \rsetpos$ a finite kernel if, for each $\x \in \rset^d$, the map 
$\set{A} \mapsto \kernel{K}(\x,\set{A})$ is a measure on $\borelians{\rset^d}$, and for each $\set{A} \in \borelians{\rset^d}$, the map $\x \mapsto \kernel{K}(\x,\set{A})$ is a Borel measurable function from $\rset^d$ to $[0, \infty)$. Such kernels induce an operator on real-valued measurable functions defined as:
\begin{align*}
    \kernel{K} h : \rset^{d} \ni \x \mapsto \int h(y) \kernel{K} (\x,\rmd y) \eqsp,
\end{align*}
and an operator on measures $\mu \in \measures{\rset^d}$:
\begin{align*}
    \mu \kernel{K} : \borelians{\rset^d} \ni \set{A} \mapsto \int \mu(\rmd \x) \kernel{K} (\x,\set{A}) \eqsp. 
\end{align*}
A positive function $h \in \measfuns{\rset^d}$ may also act on kernels $\kernel{K}$ by defining a $h$-weighted kernels
\begin{align*}
    h \kernel{K}: \rset^\xdim \times \borelians{\rset^\xdim} \ni (\x, \set{A}) \rightarrow h(\x) \kernel{K}(\x, \set{A}) \in \rsetpos \eqsp.
\end{align*}
Denote also by $\un$ the constant function on $\rset^\xdim$, which is equal to one. For any positive measurable function $g$ and any probability measure $\nu$ such that $0<\nu[g]<\infty$, we define, the tilting map by
$$
\tau_g(\nu)[\psi]
\eqdef
\tfrac{\nu[g\psi]}{\nu[g]} \eqsp,
$$
for every measurable test function \(\psi\) such that \(\nu[|g\psi|]<\infty\). We also say that a function $\mu: \borelians{\rset^\xdim} \rightarrow \rset$ is a signed-measure if there exists $\mu_1$, $\mu_2$ in $\measures{\rset^\xdim}$ such that $\mu = \mu_1 - \mu_2$.
We define the set of signed measures as $\signedmeasures{\rset^\xdim}$. For every $\mu \in \signedmeasures{\rset^\xdim}$ we define the total variation measure $|\mu|$ as the smallest positive measure $\lambda$ satisfying $\lambda(\set{A}) \geq |\mu(\set{A})|$ for all $\set{A} \in \borelians{\rset^\xdim}$. We do not prove the existence of such a measure and refer the reader to \cite[Chapter 6]{rudin1987real} for the proof.
For any Borel measurable $V: \rset^{\xdim} \rightarrow [0, \infty)$, we define 
\begin{equation}
    \label{eq:def:rho_b}
    \vnorm: \signedmeasures{\rset^\xdim} \ni \mu \rightarrow |\mu|(1+V) = \int_{\rset^\xdim}(1+V(\x))\,|\mu|(\rmd \x)
 \in [0, \infty]\eqsp.
\end{equation}
If we define the set $\vsignedmeasures{\rset^\xdim} \eqdef \{\mu \in \signedmeasures{\rset^\xdim} | \vnorm[\mu][] < \infty\}$, it is possible to show \cite[Proposition D.3.3]{douc2018markov} that $ (\vsignedmeasures{\rset^\xdim}, \vnorm)$ is a Banach space.

\subsection{Feynman--Kac framework}
\label{sec:fk_smc_general}

We start by describing the
Feynman--Kac framework underlying our analysis, and later introduce its specialization to conditional sampling with diffusion models.

\paragraph{Idealized Feynman--Kac model.}
Let $\rset^\xdim$ be the state space. We consider a sequence of positive potentials
$(\pot{k})_{0\le k\le N-1}$ and Markov kernels $(\mk{k})_{1\le k\le N}$.
A pair $(\pot{k-1}, \mk{k})$ induces the following unnormalized predictive kernels
\begin{align} \label{eq:fk_kernel_approx_intro}
\unpredictive{k} \eqdef \pot{k-1}\mk{k} \eqsp, \qquad
 \unpredictive{k} \psi(\x) = \pot{k-1}(\x)\mk{k}\psi(\x) \eqsp.
\end{align} 
The corresponding normalized Feynman--Kac transformation is the nonlinear map
\begin{equation} \label{eq:filter_recursion}
\filtertransform{k}: \measures{\rset^\xdim} \to \probmeasures{\rset^\xdim} \eqsp, \qquad
\filtertransform{k} (\nu) \eqdef \frac{\nu\unpredictive{k}} {\nu\unpredictive{k}\un} \eqsp,
\end{equation}
whenever $0<\nu\unpredictive{k}\un<\infty$. Equivalently, for every bounded measurable test function $\psi$,
$$
\filtertransform{k} (\nu) \left[\psi \right] = \frac{ \nu[\pot{k-1}\mk{k}\psi] }{ \nu[\pot{k-1}] } = \left( \tau_{\pot{k-1}}(\nu)\mk{k} \right) \left[\psi \right] \eqsp.
$$
With this map, we define the selection-\textit{before}-propagation Feynman--Kac flow by the recursion
$$
\filter{k} =\filtertransform{k} (\filter{k-1}) \eqsp, \qquad k=1,\dots,N \eqsp.
$$

\paragraph{Approximate Feynman--Kac model.}

We consider a sequence of strictly positive measurable approximate potentials $(\approxpot{k})_{0\le k\le N-1}$ and approximate Markov kernels $(\mk{k}[\theta,h])_{1\le k\le N}$. For $1\le k\le N$, define the approximate unnormalized
selection--mutation kernel
\begin{align}
\label{eq:fk_kernel_approx}
\unpredictive{k}[\theta] \eqdef \approxpot{k-1}\mk{k}[\theta,h] \eqsp,
\qquad \unpredictive{k}[\theta]\psi(\x) = \approxpot{k-1}(\x) \mk{k}[\theta,h]\psi(\x) \eqsp.
\end{align}
The corresponding normalized Feynman--Kac transformation is
\begin{equation}
\label{eq:filter_recursion_approx}
\filtertransform{k}[\theta](\nu) \eqdef \frac{\nu\unpredictive{k}[\theta]}{\nu\unpredictive{k}[\theta]\un} \eqsp,
\end{equation}
whenever
$0<\nu\unpredictive{k}[\theta]\un<\infty$.
Equivalently, for every bounded measurable test function $\psi$,
$$
\filtertransform{k}[\theta](\nu)[\psi]
= \frac{ \nu[ \approxpot{k-1}\mk{k}[\theta,h]\psi ] }{ \nu[\approxpot{k-1}] }
= \left( \tau_{\approxpot{k-1}}(\nu) \mk{k}[\theta,h] \right)[\psi] \eqsp.
$$

Starting from an approximate initial distribution
$\approxfilter{0}$, we define the approximate Feynman--Kac flow
recursively by
$$
\approxfilter{k}
= \filtertransform{k}[\theta](\approxfilter{k-1}) \eqsp, \qquad k = 1,\ldots,N \eqsp.
$$

\paragraph{Particle-based approximation.}

Although the approximate flow
$(\approxfilter{k})_{0\le k\le N}$ is well defined recursively, it is
generally not available in closed form. SMC approximates it by empirical
measures carried by $M$ particles. Suppose that, at time $k$, we have
particles $\X_k^{\theta,(1)},\ldots,\X_k^{\theta,(M)}$ and the empirical measure
$$
\mcfilter{k} \eqdef \frac{1}{M} \sum_{i=1}^{M} \delta_{\X_k^{\theta,(i)}} \eqsp.
$$
Applying the approximate Feynman--Kac transformation
$\filtertransform{k+1}[\theta]$ to $\mcfilter{k}$ gives the finite mixture
$$
\filtertransform{k+1}[\theta](\mcfilter{k})(\rmd\x)
= \sum_{i=1}^{M} w_k^{(i)} \mk{k+1}[\theta,h] (\X_k^{\theta,(i)},\rmd\x) \eqsp,
$$
where
$$
w_k^{(i)} \eqdef \frac{ \approxpot{k}(\X_k^{\theta,(i)}) }{ \sum_{j=1}^{M} \approxpot{k}(\X_k^{\theta,(j)}) } \eqsp.
$$
Equivalently, particles are first selected according to the weighted
empirical measure
$$
\tau_{\approxpot{k}}(\mcfilter{k})
= \sum_{i=1}^{M} w_k^{(i)} \delta_{\X_k^{\theta,(i)}} \eqsp,
$$
and then propagated through the approximate mutation kernel $\mk{k+1}[\theta,h]$. Conditionally on the current particle system, sampling $M$ particles independently from the resulting mixture produces the empirical approximation $\mcfilter{k+1}$.

\subsection{Correspondence with existing methods: noiseless inpainting}
\label{app:noiseless_inpainting_correspondence}

We illustrate the correspondence of \Cref{alg:smc} with MCGdiff \citep{cardoso2024monte} and SMCDiff \citep{trippe2023diffusion}. For clarity, we consider their common specialization to noiseless inpainting. Write $\x=(\x_{\mathrm{o}},\x_{\mathrm{u}})$ for the observed and unobserved coordinates, respectively, and suppose that $\x_{\mathrm{o}}=\xobs$.
\begin{enumerate}[leftmargin=*,label=(\roman*),itemsep=1pt]
\item \textbf{MCGdiff} defines for noiseless inpainting, $h_k(\x) = \gaussiand{ \fwdmean{0}{T-t_k}\xobs }{ \fwdvar{0}{T-t_k}\Id_{\dobs} }[\x_{\rm o}]$, with
$$
\approxpot{k}(\x) = \frac{ \int h_{k+1}(z)\, \approxbwdker{t_k}[t_{k+1}](\x,\rmd z) }{ h_k(\x) } \eqsp,
\qquad \mk{k+1}[\theta](\x,\rmd z) = \frac{ h_{k+1}(z)\, \approxbwdker{t_k}[t_{k+1}](\x,\rmd z) }{ \int h_{k+1}(u)\, \approxbwdker{t_k}[t_{k+1}](\x,\rmd u) } \eqsp.
$$

\item \textbf{SMCDiff} generates a forward-diffused trajectory $(y_k)_{k=0}^N$ of the observed coordinates, with $y_N=\xobs$. Conditional on this trajectory, define $\mathbf u_k \eqdef (y_k,\x_{\rm u})$ and $\mathbf v_{k+1} \eqdef (y_{k+1},z_{\rm u})$. Then
$$
\approxpot{k}(\x_{\rm u})
= \int \approxbwdker{t_k}[t_{k+1}] \left( \mathbf u_k,\mathbf v_{k+1} \right) \,\rmd z_{\rm u} \eqsp, \qquad \mk{k+1}[\theta](\x_{\rm u},\rmd z_{\rm u})
= \int_{\rset^{\dobs}} \approxbwdker{t_k}[t_{k+1}] \left( \mathbf u_k, \rmd z_{\rm o}\,\rmd z_{\rm u} \right) \eqsp.
$$
\end{enumerate}

\section{Forward-smoothing kernels: properties and forgetting}
\label{app:forward_smothing_forgetting}

\subsection{Kernel notation and tail factorization}

In the main paper, we introduced the continuation function associated with the approximate Feynman--Kac model for terminal time $N$ through its probabilistic representation in \eqref{eq:approx_backward_functions_main_norm}. Here, we extend the notation to an arbitrary terminal index $0\le k\le\ell\le N$ and give its kernel representation.
For $0\le k<\ell\le N$,  the normalized continuation function is
\begin{equation}
\label{eq:approx_backward_functions_app}
\normbackwardfun{k}{\ell}[\theta] \eqdef \frac{ \unpredictive{k+1}[\theta] \cdots \unpredictive{\ell}[\theta]\un }{ \prod_{j=k}^{\ell-1} \norminfty{\approxpot{j}} } \eqsp, \qquad \normbackwardfun{\ell}{\ell}[\theta] \eqdef \un \eqsp.
\end{equation}
Since $\unpredictive{k+1}[\theta] =\approxpot{k}\mk{k+1}[\theta,h]$,
the continuation functions satisfy
\begin{equation}
\label{eq:continuation_recursion}
\normbackwardfun{k}{\ell}[\theta]
= \frac{\approxpot{k}}{\norminfty{\approxpot{k}}} \mk{k+1}[\theta,h] \normbackwardfun{k+1}{\ell}[\theta] \eqsp, \qquad 0\le k<\ell\le N \eqsp.
\end{equation}
We define \emph{forward smoothing kernels}, obtained by tilting the proposal kernel with the future continuation weight. For $1\le k \le \ell \le N$, define
\begin{equation} \label{eq:approx_smoothing_kernel_main}
\fsmk{k}{\ell}[\theta][\x_{k-1}][\rmd \x_k] \eqdef \frac{ \mk{k}[\theta,h][\x_{k-1}][\rmd \x_k]\, \normbackwardfun{k}{\ell}[\theta][\x_k] }{ \mk{k}[\theta,h]\normbackwardfun{k}{\ell}[\theta](\x_{k-1}) } \eqsp.
\end{equation}
The kernel $\fsmk{k}{\ell}[\theta]$ describes the transition from
time $k-1$ to time $k$ after accounting for the remaining
Feynman--Kac weights up to terminal time $\ell$. We write
$$
\prodfsmk{k}{\ell}[\theta] \eqdef \fsmk{k+1}{\ell}[\theta] \cdots \fsmk{\ell}{\ell}[\theta] \eqsp, \qquad \prodfsmk{\ell}{\ell}[\theta] \eqdef \operatorname{Id} \eqsp,
$$
so that $\prodfsmk{k}{\ell}[\theta]$ propagates a perturbation from
time $k$ to time $\ell$ under the guided dynamics.

The above kernel notation coincides with the probabilistic
representation used in the main text in \eqref{eq:forward_smoothing_interpretation_main}. Indeed, for $k<r\le \ell$, by
\eqref{eq:continuation_recursion},
$$
\mk{r}[\theta,h] \normbackwardfun{r}{\ell}[\theta](\x_{r-1})
= \frac{\normbackwardfun{r-1}{\ell}[\theta](\x_{r-1}) }{ \approxpot{r-1}(\x_{r-1})/ \norminfty{\approxpot{r-1}} } \eqsp.
$$
Hence,
$$
\fsmk{r}{\ell}[\theta][\x_{r-1}][\rmd\x_r]
= \mk{r}[\theta,h][\x_{r-1}][\rmd\x_r]\, \frac{\approxpot{r-1}(\x_{r-1})} {\norminfty{\approxpot{r-1}}} \frac{ \normbackwardfun{r}{\ell}[\theta][\x_r] }{ \normbackwardfun{r-1}{\ell}[\theta][\x_{r-1}] } \eqsp.
$$
Therefore, multiplying the forward-smoothing kernels
$\fsmk{k+1}{\ell}[\theta],\ldots,\fsmk{\ell}{\ell}[\theta]$
causes the continuation functions to telescope. Consequently, for
every bounded measurable $\psi$,
\begin{equation}
\label{eq:forward_smoothing_kernel_form_app}
\prodfsmk{k}{\ell}[\theta]\psi(\x)
= \frac{ \unpredictive{k+1}[\theta]\cdots \unpredictive{\ell}[\theta]\psi(\x) }{ \unpredictive{k+1}[\theta]\cdots \unpredictive{\ell}[\theta]\un(\x) } \eqsp, \qquad 0\le k<\ell\le N \eqsp.
\end{equation}
Thus, $\prodfsmk{k}{\ell}[\theta]$ is the normalized tail
Feynman--Kac kernel from time $k$ to time $\ell$. Taking $\ell=N$ and using the Markov-chain representation of the above recovers the probabilistic representation used in the main text in \eqref{eq:forward_smoothing_interpretation_main}.

\begin{proposition}[Tail factorization]
\label{prop:approx_tail_factorization}
Let $0 \le k \leq m \le N$. For every probability measure $\nu \in \mathcal{P} (\rset^\xdim)$ such that $0 < \nu(\normbackwardfun{k}{m}[\theta]) < \infty$, with $\normbackwardfun{k}{m}[\theta]$ defined as in \eqref{eq:approx_backward_functions_main_norm}, we have
\begin{align}
\label{eq:approx_normalized_tail_factorization}
    \left( \filtertransform{m}*[\theta] \circ \cdots \circ \filtertransform{k+1}*[\theta] \right) \left( \nu \right)
    &= \backwardnu{k}{m}[\theta]
    \fsmk{k+1}{m}[\theta]\cdots
    \fsmk{m}{m}[\theta]
    \eqsp,
    \qquad
    \text{ with }
    \backwardnu{k}{m}[\theta]
    \eqdef \frac{ \normbackwardfun{k}{m}[\theta] \nu}{\nu [ \normbackwardfun{k}{m}[\theta] ]} \eqsp.
\end{align}
\end{proposition}

\begin{proof}
For $0\le k<m\le N$, set
$$
C_{k:m} \eqdef \prod_{i=k}^{m-1}\norminfty{\approxpot{i}} \eqsp.
$$
Recall that
\begin{align*}
\left( \filtertransform{m}*[\theta] \circ \cdots \circ \filtertransform{k+1}*[\theta] \right) \left( \nu \right)
= \frac{ \nu \unpredictive{k+1}[\theta] \cdots \unpredictive{m}[\theta] }{ \nu \unpredictive{ k+1}[\theta] \cdots \unpredictive{m}[\theta] \un  } \eqsp,
\end{align*}
with the unnormalized kernels $\unpredictive{\ell}[\theta]$, for $k+1\leq\ell\leq m$, as in \eqref{eq:fk_kernel_approx_intro}. We first show by induction on $m - k$ that, for every bounded measurable test function $\psi$,
\begin{align} \label{eq:unnormalized_factorization_identity}
\unpredictive{k+1}[\theta] \cdots \unpredictive{m}[\theta] \psi
=
C_{k:m} \normbackwardfun{k}{m}[\theta] \fsmk{k+1}{m}[\theta]\cdots \fsmk{m}{m}[\theta] \psi \eqsp.
\end{align}
First note that when $ m = k +1$,
\begin{align*}
\normbackwardfun{k}{k+1}[\theta](\x_k) = \potnorm{k}(\x_k) \eqsp, \qquad \fsmk{k+1}{k+1}[\theta] = \mk{k+1}[\theta,h] \eqsp,
\end{align*}
and therefore, 
\begin{align*}
\unpredictive{k+1}[\theta] \psi(\x_k)
= \approxpot{k}(\x_k) \mk{k+1}[\theta,h] \psi(\x_k)
= \norminfty{\approxpot{k}} \normbackwardfun{k}{k+1}[\theta](\x_k) \fsmk{k+1}{k+1}[\theta]\psi (\x_k) \eqsp.
\end{align*}
Then, assume that \eqref{eq:unnormalized_factorization_identity} holds for the pair $(k+1,m)$, with $k+1<m$, namely
$$
\unpredictive{k+2}[\theta]\cdots \unpredictive{m}[\theta]\psi
=
C_{k+1:m}\normbackwardfun{k+1}{m}[\theta] \fsmk{k+2}{m}[\theta] \cdots \fsmk{m}{m}[\theta]\psi \eqsp.
$$
Then, for any bounded measurable $\psi$,
\begin{align*} 
\unpredictive{k+1}[\theta]\unpredictive{k+2}[\theta] \cdots \unpredictive{m}[\theta]\psi(\x_k)  
&\quad = C_{k+1:m}\, \approxpot{k}(\x_k) \mk{k+1}[\theta,h] \left( \normbackwardfun{k+1}{m}[\theta] \fsmk{k+2}{m}[\theta]\cdots \fsmk{m}{m}[\theta]\psi \right)(\x_k) \\ 
&\quad = C_{k:m}\, \potnorm{k}(\x_k) \mk{k+1}[\theta,h] \left( \normbackwardfun{k+1}{m}[\theta] \fsmk{k+2}{m}[\theta]\cdots \fsmk{m}{m}[\theta]\psi \right)(\x_k) \eqsp.
\end{align*}
By definition, the normalized continuation functions satisfy the recursion 
$$\normbackwardfun{k}{m}[\theta](\x_k) = \potnorm{k} (\x_k) \mk{k+1}[\theta,h] \normbackwardfun{k+1}{m}[\theta](\x_k) \eqsp, $$
so we obtain
\begin{align*}
\unpredictive{k+1}[\theta]\cdots \unpredictive{m}[\theta]\psi(\x_k)
&= C_{k:m} \normbackwardfun{k}{m}[\theta](\x_k)  \frac{ \mk{k+1}[\theta,h] \left( \normbackwardfun{k+1}{m}[\theta] \fsmk{k+2}{m}[\theta]\cdots \fsmk{m}{m}[\theta]\psi \right)(\x_k) }{ \mk{k+1}[\theta,h]\normbackwardfun{k+1}{m}[\theta](\x_k) } \eqsp.
\end{align*}
By definition of the forward smoothing kernel \eqref{eq:approx_smoothing_kernel_main}, 
\begin{align*}
&\frac{ \mk{k+1}[\theta,h] \left( \normbackwardfun{k+1}{m}[\theta] \fsmk{k+2}{m}[\theta] \cdots \fsmk{m}{m}[\theta]\psi  \right)(\x_k) }{ \mk{k+1}[\theta,h]\normbackwardfun{k+1}{m}[\theta](\x_k) } \\
&= \frac{ \int \normbackwardfun{k+1}{m}[\theta](\x_{k+1}) \left( \fsmk{k+2}{m}[\theta] \cdots \fsmk{m}{m}[\theta]\psi \right)(\x_{k+1}) \mk{k+1}[\theta,h](\x_k,\rmd \x_{k+1}) }{ \mk{k+1}[\theta,h]\normbackwardfun{k+1}{m}[\theta](\x_k) } \\
&= \int \left(  \fsmk{k+2}{m}[\theta] \cdots \fsmk{m}{m}[\theta]\psi \right)(\x_{k+1}) \fsmk{k+1}{m}[\theta](\x_k,\rmd \x_{k+1}) \\
&= \fsmk{k+1}{m}[\theta]\fsmk{k+2}{m}[\theta] \cdots \fsmk{m}{m}[\theta]\psi(\x_k) \eqsp.
\end{align*}
Therefore,
\begin{align*}
\unpredictive{k+1}[\theta] \cdots \unpredictive{m}[\theta] \psi(\x_k)
=
C_{k:m} \normbackwardfun{k}{m}[\theta] (\x_k) \fsmk{k+1}{m}[\theta]\cdots \fsmk{m}{m}[\theta]\psi(\x_k) \eqsp,
\end{align*}
which proves \eqref{eq:unnormalized_factorization_identity}. Finally, applying \eqref{eq:unnormalized_factorization_identity} with $\psi$ and with $\un$, and using that each $\fsmk{\ell}{m}[\theta]$ is a Markov kernel, for $k+1\le \ell\le m$, so that
\begin{align*}
\fsmk{k+1}{m}[\theta]\cdots \fsmk{m}{m}[\theta] \un =\un \eqsp,
\end{align*}
we get
\begin{align*}
\left( \filtertransform{m}*[\theta] \circ \cdots \circ \filtertransform{k+1}*[\theta] \right)(\nu)[\psi]
& = \frac{ \nu \unpredictive{k+1}[\theta]\cdots\unpredictive{m}[\theta]\psi  }{ \nu \unpredictive{k+1}[\theta]\cdots\unpredictive{m}[\theta]\un  } \\
&= \frac{ C_{k:m} \nu\left[ \normbackwardfun{k}{m}[\theta] \fsmk{k+1}{m}[\theta]\cdots\fsmk{m}{m}[\theta]\psi \right] }{ C_{k:m} \nu\left[ \normbackwardfun{k}{m}[\theta] \right] } \\
&= \frac{ \nu\left[ \normbackwardfun{k}{m}[\theta]\, \fsmk{k+1}{m}[\theta]\cdots\fsmk{m}{m}[\theta]\psi \right] }{ \nu\left[ \normbackwardfun{k}{m}[\theta] \right] }\\
& = \left( \backwardnu{k}{m}[\theta] \fsmk{k+1}{m}[\theta]\cdots\fsmk{m}{m}[\theta] \right) \left[\psi \right] \eqsp.
\end{align*}
which holds for every bounded measurable $\psi$, and therefore proves \eqref{eq:approx_normalized_tail_factorization}.
\end{proof}

\begin{remark}
This factorization is purely algebraic and is not specific to the approximate Feynman--Kac model. It holds for any sequence of Markov kernels and their associated continuation functions and forward-smoothing kernels.
\end{remark}

\begin{corollary}
\label{lem:terminal_forward_smoothing_decomposition}
For every $0\le k\le N$, we have
\begin{equation}
\label{eq:terminal_forward_smoothing_decomposition_app}
\approxfilter{N}
=
\tau_{\normbackwardfun{k}{N}[\theta]} (\approxfilter{k}) \prodfsmk{k}{N}[\theta] \eqsp.
\end{equation}

\end{corollary}

\begin{proof}
The case $k=N$ is immediate. For $k<N$, by the recursive definition
of the approximate Feynman--Kac flow,
$$
\approxfilter{N} = \left( \filtertransform{N}[\theta] \circ\cdots\circ \filtertransform{k+1}[\theta] \right) (\approxfilter{k}) \eqsp.
$$
Applying \Cref{prop:approx_tail_factorization} with
$m=N$ and $\nu=\approxfilter{k}$ yields
$$
\approxfilter{N} = \backwardnu{k}{N}[\theta] \prodfsmk{k}{N}[\theta] \eqsp.
$$
Since
$$
\backwardnu{k}{N}[\theta] = \tau_{\normbackwardfun{k}{N}[\theta]} (\approxfilter{k}) \eqsp,
$$
the result follows.
\end{proof}

\subsection{Proof of Proposition~\ref{prop:drift_minorization_proposal_theta}}
\label{app:diffusion_specialization}

We keep the SMC selection grid $\mathcal{T}^{\rm sel}=(t_k)_{k=0}^N$ fixed and use the possibly varying numerical grid $s_n=nh$, where $h=T/(NJ)$ and $t_k=s_{kJ}$.
Let $V(\x)\eqdef\normEc{\x}^2$. For $n=0,\ldots,NJ-1$, define
$$
\Delta_n \eqdef \int_{s_n}^{s_{n+1}}\bwdnoisesch{u}\,\rmd u \eqsp.
$$
Note that the total noise variance over a SMC step is 
$$
\int_{t_k}^{t_{k+1}}\bwdnoisesch{u}\,\rmd u = \sum_{n=kJ}^{(k+1)J-1}\Delta_n \eqsp.
$$
Denote the Euler kernel associated with the numerical scheme by
$$
\mathsf R_n^\theta(\x,\cdot) \eqdef \gaussiand{ \x+\Delta_n b_n^\theta(\x) }{ 2\Delta_n\Id_\xdim } \eqsp, \quad \text{with} \quad b_n^\theta(\x) \eqdef \isvp \x + 2 \scorenet[T-s_n][\x][\theta] \eqsp.
$$
The implemented SMC mutation is the composition kernel
$$
\mk{k+1}[\theta,h] = \mathsf R_{kJ}^\theta \cdots \mathsf R_{(k+1)J-1}^\theta \eqsp.
$$
\paragraph{Lyapunov drift condition.} For $n=kJ,\ldots,(k+1)J-1$, expanding the time-changed Euler step gives
\begin{align*}
\mathsf R_n^\theta V(\x) = V(\x) + 2 \Delta_n \dotprod{\x}{ \isvp\x+2\scorenet[T-s_n][\x][\theta]} + \Delta_n^2 \normEc{\isvp\x+2\scorenet[T-s_n][\x][\theta]}^2 + 2\xdim\Delta_n  \eqsp.
\end{align*}
By \Cref{ass:dissipative_score_main},
$$
\dotprod{\x}{ \isvp\x+2\scorenet[T-s_n][\x][\theta]} \le (\isvp-2\gamma_k)V(\x)+2\kappa_k
$$
and
$$
\normEc{\isvp\x+2\scorenet[T-s_n][\x][\theta]}^2 \le (\isvp^2-4\isvp\gamma_k+4L_k)V(\x) + 4 \isvp \kappa_k + 4 L_k \eqsp.
$$
Hence, using \eqref{eq:block_stepsize_condition},
$$
\mathsf R_n^\theta V \le \left[1-(2\gamma_k-\isvp)\Delta_n\right]V + C_k \Delta_n \eqsp,
$$
with 
\begin{equation} \label{eq:def:C_k}
C_k \eqdef 4 \kappa_k + 2 \xdim + (4\isvp\kappa_k+4L_k) \frac{2\gamma_k-\isvp}{\isvp^2-4\isvp\gamma_k+4L_k} \eqsp.
\end{equation}
Iterating this inequality over the $J$ numerical steps in
$[t_k,t_{k+1}]$ yields
\begin{align*}
\mk{k+1}[\theta,h]V \le \prod_{n=kJ}^{(k+1)J-1} \left[1-(2\gamma_k-\isvp)\Delta_n\right]V + C_k \sum_{n=kJ}^{(k+1)J-1} \Delta_n \prod_{m=n+1}^{(k+1)J-1} \left[1-(2\gamma_k-\isvp)\Delta_m\right] \eqsp.
\end{align*}
Hence,
$$
\mk{k+1}[\theta,h]V \le \prod_{n=kJ}^{(k+1)J-1} \left[1-(2\gamma_k-\isvp)\Delta_n\right]V + C_k \sum_{n=kJ}^{(k+1)J-1}\Delta_n \eqsp.
$$
Using $1-x\le e^{-x}$ and $\sum_{n=kJ}^{(k+1)J-1}\Delta_n = \int_{t_k}^{t_{k+1}}\bwdnoisesch{u}\,\rmd u,$ we obtain
\begin{equation*}
\mk{k+1}[\theta,h]V \le \exp\left( -(2\gamma_k-\isvp) \int_{t_k}^{t_{k+1}}\bwdnoisesch{u}\,\rmd u \right) V + C_k \int_{t_k}^{t_{k+1}} \bwdnoisesch{u}\,\rmd u \eqsp.
\end{equation*}
This proves \eqref{eq:block_proposal_drift}, with constants independent
of the numerical refinement.

\paragraph{Minorization.}
Recall that
$$
\mathsf{R}_n^\theta(\x,\cdot) = \gaussiand{ \x+\Delta_n b_n^\theta(\x) }{ 2\Delta_n\Id_\xdim } \eqsp, 
\qquad b_n^\theta(\x) = \isvp\x+2\scorenet[T-s_n][\x][\theta] \eqsp,
$$
and denotes its associated transition density by
$$
r_n^\theta(z_{n+1}\mid z_n) = \frac{1}{(4 \pi \Delta_n)^{\xdim/2}} \exp\!\left\{ -\frac{ \|z_{n+1}-z_n-\Delta_n b_n^\theta(z_n)\|^2 }{ 4 \Delta_n } \right\} \eqsp.
$$
Expanding the square gives
\begin{multline*}
\|z_{n+1}-z_n-\Delta_n b_n^\theta(z_n)\|^2 
= \|z_{n+1}-z_n\|^2  - 2 \Delta_n \dotprod{b_n^\theta(z_n)}{z_{n+1}-z_n} + \Delta_n^2\|b_n^\theta(z_n)\|^2 \eqsp,
\end{multline*}
and therefore
\begin{equation}
\label{eq:euler_density_ratio}
r_n^\theta(z_{n+1}\mid z_n)
= \gaussiand{z_{n+1} ; z_n}{2 \Delta_n \Id_\xdim} \exp \left\{ \frac12 \dotprod{b_n^\theta(z_n)}{z_{n+1}-z_n} - \frac{\Delta_n}{4} \|b_n^\theta(z_n)\|^2 \right\} \eqsp.
\end{equation}
Let $q_{k+1}^{\theta,h}(y\mid\x)$ denote the density of the composite kernel $\mk{k+1}[\theta,h]$. Setting the endpoints of the numerical path over $[t_k,t_{k+1}]$ to $z_{kJ}=\x$ and $z_{(k+1)J}=y$, composing the transition kernels from the time-changed Euler scheme 
\eqref{eq:approx_reverse_update} and \eqref{eq:euler_density_ratio} give
\begin{align*}
& q_{k+1}^{\theta,h}(y\mid\x) \\
&= \int \prod_{n=kJ}^{(k+1)J-1} r_n^\theta(z_{n+1}\mid z_n) \rmd z_{ kJ+1} \ldots \rmd z_{(k+1)J-1} \\
&= \int q_{1:J-1}(z_{kJ+1:(k+1)J - 1} | x, y) \gaussiand{x}{2\Delta_{kJ:(k+1)J}\Id_\xdim}[y]\, \rme^{\mathcal L_k} \rmd z_{ kJ+1} \ldots \rmd z_{(k+1)J-1}
\end{align*}
where
$$
\mathcal L_k \eqdef \sum_{n=kJ}^{(k+1)J-1} \left[ \frac12 \dotprod{b_n^\theta(z_n)}{z_{n+1}-z_n} - \frac{\Delta_n}{4} \|b_n^\theta(z_n)\|^2 \right] \eqsp,
$$
and
$$
q_{1:J-1}(z_{1:J - 1} | x, y) \eqdef \frac{\gaussiand{z_{1};x}{2\Delta_{kJ}\Id_\xdim}\prod_{n=1}^{J-2} \gaussiand{z_{n+1};z_n}{2\Delta_{kJ+n}\Id_\xdim} \gaussiand{z_{kJ-1}}{2\Delta_{kJ}\Id_\xdim}[y]}
{\gaussiand{z_{kJ-1}}{2\Delta_{kJ}\Id_\xdim}[y]}
$$
By Bayes law, we have that
\begin{equation}
    \label{eq:pont_rpz}
    q_{1:J-1}(z_{1:J - 1} | x, y) \eqdef q_{J-1|J,0}(z_{J-1}|y, x)\prod_{n=1}^{J-2}q_{n|n+1,0}(z_n |z_{n+1}, x)
\end{equation}
where
\begin{equation*}
    q_{n|n+1,0}(z_{n}|z_{n+1}, z_0) \eqdef \gaussiand{z_0 + \frac{u_n}{u_{n+1}}(z_{n+1} - z_0)}{2\frac{u_n}{u_{n+1}}(u_{n+1}-u_n)\Id}[z_n]\eqsp,
\end{equation*}
where $u_n \eqdef \sum_{m=0}^{n-1}\Delta_{kJ + m}$.
The law of appearing above, coincides with the discrete Markov Chain satisfying the following recursion 
\begin{equation*}
    Z_{kJ} \eqdef y\eqsp, \eqsp
    \text{For } n \in \{J-1, \cdots, 1\} \eqsp, \eqsp
    Z_{n} \sim q_{n|n+1, 0}(\cdot | Z_{n+1}, x) \eqsp,
\end{equation*}
where $q_{n|n+1, 0}(\cdot | z_{n}, x)$ is the probability law $\set{A} \rightarrow \int \indi{\set{A}}[z] q_{n|n+1, 0}(z | z_{n}, x) \rmd z$.

As,
$$
q_{k+1}^{\theta,h}(y\mid\x)
= \gaussiand{y;\x}{2u_J\Id_\xdim} \, \E \left[ \exp\left\{ \mathcal L_k \left( Z_{1}, \cdots Z_{J - 1} \right) \right\} \right] \eqsp,
$$
by Jensen's inequality,
$$
q_{k+1}^{\theta,h}(y\mid\x)
\geq \gaussiand{y;\x}{2u_J\Id_\xdim} \,   \exp\left[\E\left\{ \mathcal L_k \left( Z_{1}, \cdots Z_{J - 1} \right) \right\} \right]
$$
It therefore remains to lower bound 
\begin{align*}
\exp\left[\E\left\{ \mathcal L_k \left( Z_{1}, \cdots Z_{J - 1} \right) \right\} \right]
&= \frac{1}{2} \sum_{n=1}^{J-1} \E \left[ \dotprod{b_{n+kJ}^\theta(Z_n)}{Z_{n+1}-Z_n}\right] \\
& \qquad \qquad -\frac14 \sum_{n=1}^{J-1} \Delta_n \E \left[ \|b_{n+kJ}^\theta(Z_n)\|^2 \right] \eqsp.
\end{align*}
We start by controling the quadratic term.  Since $b_n^\theta(z)=\isvp z+2\scorenet[T-s_n][z][\theta]$, the linear-growth condition in \Cref{ass:dissipative_score_main} implies
$$
\| b_n^\theta(z) \|^2 \le 8L_k + ( 2 \isvp^2 + 8 L_k )(\|z\|^2) \le C_k (1 + \|z\|^2) \eqsp,
$$
with $2\isvp^2+8L_k$. Using \eqref{eq:pont_rpz}
\begin{equation}
    \label{eq:pont_marginal}
Z_n \sim \gaussiand{ \frac{u_{J} - u_n}{u_J}\x+\frac{u_n}{u_J}y }{ 2\frac{u_n}{u_J}(u_J - u_n)\Id_\xdim } \eqsp.    
\end{equation}
Hence, using that if $u \le a$ then $u (a-u) \le a^2/4$, for $V(\x)\le R$ and $\|y\|\le r$ we have that
$$
\E \left[ \|Z_n\|^2\right] = \left\| \frac{u_J-u_n}{u_J} \x + \frac{u_n}{u_J}y \right\|^2 + d \frac{2u_n(u_J-u_n)}{u_J}
\le (\sqrt R+r)^2+\frac{\xdim u_J}{2} \eqsp.
$$
The right-hand side is uniform in $J$. Therefore there exists $M$, independent of the numerical refinement, such that
$$
\E \left[ \|b_n^\theta(Z_n)\|^2  \right] \le M \eqsp.
$$
Hence,
$$
-\frac14 \sum_{n=1}^{J-1} \Delta_n \E \left[ \|b_{kJ+n}^\theta(Z_n)\|^2 \right] \ge - \frac{M}{4} u_J = - \frac{M}{4}  \int_{t_k}^{t_{k+1}} \bwdnoisesch{u}\,\rmd u \eqsp.
$$
For the other term, note that 
\begin{multline}
    \frac12 \sum_{n=1}^{J-1} \PE{}{\dotprod{b_{kJ+n}^\theta(Z_n)}{Z_{n+1}-Z_n}} \\
    \ge -\frac{1}{2}\sum_{n=1}^{J-1} \PE{}{\normEc{b_{kJ+n}^\theta(Z_n)}^2}^{1/2}\PE{}{\normEc{Z_{n+1}-Z_n}^2}^{1/2} \eqsp.
\end{multline}
For the term $\PE{}{\normEc{Z_{n+1}-Z_n}^2}^{1/2}$, we have that
\begin{equation*}
    Z_{n} - Z_{n+1} = x \left(1 - \frac{u_n}{u_{n+1}}\right) - Z_n(1 - \frac{u_n}{u_{n+1}}) + \left(2\frac{u_n}{u_{n+1}}(u_{n+1} - u_n)\right)^{1/2}{\xi}_{n-1} \eqsp,
\end{equation*}
where $\xi \sim \gaussiand{0}{\Id}$ and is independent of $Z_n$.
But by \eqref{eq:pont_marginal}
\begin{equation*}
    \PE{}{Z_{n+1}} = x + \frac{u_{n+1}}{u_J}\left(y - x\right) \eqsp, 
\end{equation*}
and then $\PE{}{Z_{n} - Z_{n+1}} = (x - y) \nofrac{u_{n+1}}{u_J}$.
For the variance term, we have that
\begin{equation*}
    \PV{}{Z_{n} - Z_{n+1}} = 2 (u_{n+1} - u_n) \left(1 - \frac{u_n}{u_J}\right) \Id \eqsp. 
\end{equation*}
Therefore, we have that
\begin{equation*}
    \PE{}{\|Z_{n} - Z_{n+1}\|^2} = \normEc{x - y}^2 \left(\frac{u_{n+1}-u_n}{u_J}\right)^2 + 2(u_{n+1}-u_n) \left(1 - \frac{(u_{n+1}-u_n)}{u_J}\right) \xdim \eqsp,
\end{equation*}
and finally 
\begin{multline*}
    \sum_{n=1}^{J-1}\PE{}{\|Z_{n} - Z_{n+1}\|^2} = \normEc{x - y}^2 \frac{\sum_{n=1}^{J-1} \Delta_n^2}{u_J^2} + 2 \frac{\xdim}{u_J} \left(u_J^2 - \sum_{n=1}^{J-1}\Delta_n^2\right)\\
    \leq \|x - y\|^2  + 2 \xdim u_J\eqsp.
\end{multline*}
Combining the two yields
$$
\exp\left[\E\left\{ \mathcal L_k \left( Z_{1}, \cdots Z_{J - 1} \right) \right\} \right]  \ge M' \eqsp.
$$
Therefore, 
$$
q_{k+1}^{\theta,h} (y\mid\x) \ge \gaussiand{y;\x}{2a_k\Id_\xdim} \rme^{-M'} \eqsp.
$$
For $V(\x)\le R$ and $\|y\|\le r$,
$$
\gaussiand{y ; \x}{ 2 a_k \Id_\xdim} \ge (4\pi a_k)^{-\xdim/2} \exp\left\{ -\frac{ (\sqrt R+r)^2 }{4a_k} \right\} 
\eqdef \underline{q}_{k,R,r}
$$
Hence, for $V(\x)\le R$, $\|y\|\le r$
$$
q_{k+1}^{\theta,h}(y\mid\x) \ge \underline{q}_{k,R,r} \rme^{-M'}
$$
Let $\nu_{k+1,R}^Q$ be the uniform probability measure on $\ball{0}{r}$, namely
$$
\nu_{k+1,R}^Q(A) \eqdef \frac{ \operatorname{Leb}(A\cap\ball{0}{r}) }{ \operatorname{Leb}(\ball{0}{r})
} \eqsp.
$$
Then, for every $A\in\borelians{\rset^\xdim}$ and $V(\x)\le R$,
\begin{align*}
\mk{k+1}[\theta,h](\x,A)
&= \int_A q_{k+1}^{\theta,h}(y\mid\x)\,\rmd y\\
&\ge \underline q_{k,R,r} \rme^{-M'} \operatorname{Leb} \left( A \cap \ball{0}{r} \right)\\
&= \varepsilon_{k+1}^Q(R)\, \nu_{k+1,R}^Q(A) \eqsp,
\end{align*}
where $\varepsilon_{k+1}^Q(R) \eqdef \underline q_{k,R,r} \operatorname{Leb}\!\left(\ball{0}{r}\right)>0$. This proves the local minorization.

\subsection{Forgetting of the forward smoothing kernel}

We now establish forgetting for the forward-smoothing kernels under proposal-level stability conditions. The following assumption is stated at an abstract kernel level for generality, but introduces no additional requirement for Algorithm~\ref{alg:smc}: by \Cref{prop:drift_minorization_proposal_theta}, it follows directly from \Cref{ass:dissipative_score_main} and the step-size condition \eqref{eq:block_stepsize_condition}. Thus, throughout the remainder of the Appendix, any result stated under \Cref{ass:proposal_harris} applies to Algorithm~\ref{alg:smc} under \Cref{ass:dissipative_score_main} and \eqref{eq:block_stepsize_condition}. This formulation also allows the subsequent analysis to apply to other proposal mechanisms satisfying the same drift and minorization properties.

\begin{assumption}[Proposal drift and local minorization]
\label{ass:proposal_harris}
The approximate proposal kernels $(\mk{k}[\theta])_{1\le k\le N}$ satisfy the
following conditions:

\begin{assumplist}

\item \textbf{Drift.}
There exists a measurable function $V: \rset^\xdim \to [0,\infty)$ such that, for every $1 \le k \le N$, there exist constants $\lambda_k^Q\in(0,1)$ and $K_k^Q<\infty$ satisfying, for any $\x \in \rset^\xdim$,
\begin{equation} \label{eq:drift_Q}
\mk{k}[\theta,h] V(\x) \le \lambda_k^Q V(\x) + K_k^Q \eqsp.
\end{equation}

\item \textbf{Minorization.}
For every $R>0$ and every $1\le k\le N$, there exist
$\varepsilon_k^Q(R)>0$ and a probability measure $\nu_{k,R}^Q$ such that, for every $\x \in C_R \eqdef \{ \x \in \rset^\xdim: V(\x) \le R\} $, and every $A \in \borelians{\rset^\xdim}$,
\begin{equation} \label{eq:minorization_Q}
\mk{k}[\theta,h](\x,A) \ge \varepsilon_k^Q(R)\nu_{k,R}^Q(A) \eqsp.
\end{equation}

\end{assumplist}
\end{assumption}


Throughout the following results, $N$ and the implemented model are
fixed. By \Cref{ass:bounded_normalized_potentials},
\[
0<
\normbackwardfun{k}{N}[\theta]
\le1.
\]
Consequently, for every proposal minorizing measure,
\begin{equation}
\label{eq:minoration_cQ}
c_{k,R}^Q
\eqdef
\nu_{k,R}^Q[
\normbackwardfun{k}{N}[\theta]]
>0.
\end{equation}

\begin{lemma}[Forgetting transfer conditions]
\label{lem:harris_forward_smoothing}
Suppose that
\Cref{ass:proposal_harris} and \Cref{ass:bounded_normalized_potentials}
hold and define
$$
V_{k|N}^{\theta}(\x) \eqdef \frac{V(\x)}{\widetilde\beta_{k|N}^{\theta}(\x)} \eqsp, \qquad 0\le k\le N \eqsp.
$$
Then, for every $R>0$, every $1\le k\le N$, every
$\x_{k-1}\in C_R$, and every $A\in\borelians{\rset^\xdim}$,
\begin{equation}
\label{eq:minorization_cond_G}
\fsmk{k}{N}[\theta][\x_{k-1}][A] \ge \varepsilon_k^Q(R) c_{k,R}^Q \nu_{k,R}^{\mathsf G}(A) \eqsp,
\end{equation}
where
$$
\nu_{k,R}^{\mathsf G}(A) \eqdef \frac{ \nu_{k,R}^Q \left[ \widetilde \beta_{k|N}^{\theta} \un_A \right]
}{
\nu_{k,R}^Q \left[ \widetilde \beta_{k|N}^{\theta} \right] } \eqsp.
$$
Here the dependence of $\nu_{k,R}^{\mathsf G}$ on $N$ and $\theta$ is dropped. Moreover, for every $1\le k\le N$, choose $R_k>0$ such that
$$
R_k > \frac{K_k^Q}{1 - \lambda_k^Q},
$$
and define
$$
\lambda_k^{\mathsf G} \eqdef \lambda_k^Q+\frac{K_k^Q}{R_k} \in(0,1) \eqsp,
\qquad K_k^{\mathsf G} \eqdef \frac{K_k^Q}{ \varepsilon_k^Q(R_k) c_{k,R_k}^Q} \eqsp.
$$
Then, for every $1\le k\le N$ and every $\x_{k-1} \in \rset^\xdim$,
\begin{equation}
\label{eq:drift_cond_G}
\fsmk{k}{N}[\theta] V_{k|N}^{\theta} (\x_{k-1}) \le \lambda_k^{\mathsf G} V_{k-1|N}^{\theta} (\x_{k-1}) +  K_k^{\mathsf G} \un_{C_{R_k}}(\x_{k-1}) \eqsp,
\end{equation}
where $C_{R_k} = \{ \x \in \rset^\xdim : V(\x) \leq R_k \}$.
\end{lemma}

\begin{proof}
Since $0 < \potnorm{\ell} \le 1$ for $0\le \ell \le N-1$, we have
$$
0 < \widetilde \beta_{k|N}^{\theta} \le 1 \eqsp, \qquad 0 \le k\le N \eqsp.
$$
Therefore,
$$
\mk{k}[\theta,h] \widetilde \beta_{k|N}^{\theta}(\x_{k-1}) \le 1 \eqsp.
$$

Let $R>0$, $\x_{k-1}\in C_R$, and $A\in\borelians{\rset^\xdim}$. By the definition of $\fsmk{k}{N}[\theta]$ and the minorization condition,
\begin{align*}
\fsmk{k}{N}[\theta][\x_{k-1}][A]
&= \frac{ \mk{k}[\theta,h] \left( \widetilde \beta_{k|N}^{\theta} \un_A \right)(\x_{k-1}) }{ \mk{k}[\theta,h] \widetilde \beta_{k|N}^{\theta} (\x_{k-1}) } \\
& \ge \mk{k}[\theta,h] \left( \widetilde \beta_{k|N}^{\theta} \un_A \right) (\x_{k-1}) \\
&\ge \varepsilon_k^Q(R) \nu_{k,R}^Q \left[ \widetilde \beta_{k|N}^{\theta} \un_A \right] \eqsp.
\end{align*}
Moreover, using \eqref{eq:minoration_cQ}
$$
\nu_{k,R}^Q \left[ \widetilde \beta_{k|N}^{\theta}\un_A \right]
= \nu_{k,R}^Q \left[ \widetilde \beta_{k|N}^{\theta} \right] \nu_{k,R}^{\mathsf G}(A)
= c_{k,R}^Q \nu_{k,R}^{\mathsf G}(A) \eqsp.
$$
Thus
$$
\fsmk{k}{N}[\theta][\x_{k-1}][A] \ge \varepsilon^Q_k(R)c_{k,R}^Q \nu_{k,R}^{\mathsf G}(A) \eqsp,
$$
which proves \eqref{eq:minorization_cond_G}.

For the drift condition, note that
\begin{equation} \label{eq:recursion_beta_tilde}
\widetilde \beta_{k-1|N}^{\theta} = \potnorm{k-1} \mk{k}[\theta,h] \widetilde \beta_{k|N}^{\theta} \le \mk{k}[\theta,h] \widetilde \beta_{k|N}^{\theta} \eqsp.
\end{equation}
Hence,
$$
\frac{ V(\x_{k-1}) }{ \mk{k}[\theta,h] \widetilde \beta_{k|N}^{\theta} (\x_{k-1}) } \le \frac{ V(\x_{k-1}) }{ \widetilde \beta_{k-1|N}^{\theta} (\x_{k-1}) } = V_{k-1|N}^{\theta}(\x_{k-1}) \eqsp.
$$
Using the drift condition in \Cref{ass:proposal_harris}, we obtain
\begin{align*}
\fsmk{k}{N}[\theta]V_{k|N}^{\theta}(\x_{k-1})
&= \int \frac{ V(\x_k) }{ \widetilde \beta_{k|N}^{\theta}(\x_k) }
\frac{ \widetilde \beta_{k|N}^{\theta}(\x_k) \mk{k}[\theta,h][\x_{k-1}][\rmd \x_k] }{ \mk{k}[\theta,h] \widetilde \beta_{k|N}^{\theta}(\x_{k-1}) } \\
&= \frac{ \mk{k}[\theta,h]V(\x_{k-1}) }{ \mk{k}[\theta,h] \widetilde \beta_{k|N}^{\theta}(\x_{k-1}) } \\
&\le \lambda_k^Q V_{k-1|N}^{\theta} (\x_{k-1}) + \frac{ K_k^Q }{ \mk{k}[\theta,h] \widetilde \beta_{k|N}^{\theta}(\x_{k-1}) } \eqsp.
\end{align*}

It remains to control the offset term. Fix $R_k>K_k^Q/(1-\lambda_k^Q)$. If
$\x_{k-1}\notin C_{R_k}$, then $V(\x_{k-1})>R_k$, and
$$
\frac{ K_k^Q }{ \mk{k}[\theta,h] \widetilde \beta_{k|N}^{\theta} (\x_{k-1}) }
= \frac{ K_k^Q }{ V(\x_{k-1}) } \frac{ V(\x_{k-1}) }{ \mk{k}[\theta,h]\widetilde \beta_{k|N}^{\theta} (\x_{k-1}) } \le \frac{K_k^Q}{R_k} V_{k-1|N}^{\theta}(\x_{k-1}) \eqsp,
$$
where we used \eqref{eq:recursion_beta_tilde} in the last inequality. If $\x_{k-1}\in C_{R_k}$, then the minorization condition \eqref{eq:minorization_Q} gives
$$
\mk{k}[\theta,h] \widetilde \beta_{k|N}^{\theta} (\x_{k-1}) \ge \varepsilon^Q_k(R_k) \nu_{k,R_k}^Q \left[ \widetilde \beta_{k|N}^{\theta} \right] = \varepsilon^Q_k(R_k) c_{k,R_k}^Q \eqsp.
$$
Therefore,
$$
\frac{ K_k^Q }{ \mk{k}[\theta,h] \widetilde \beta_{k|N}^{\theta}(\x_{k-1}) } \le \frac{ K_k^Q }{ \varepsilon^Q_k(R_k) c_{k,R_k}^Q } \un_{C_{R_k}}(\x_{k-1}) + \frac{K_k^Q}{R_k} V_{k-1|N}^{\theta}(\x_{k-1}) \eqsp.
$$
Combining this with the previous drift bound yields
$$
\fsmk{k}{N}[\theta] V_{k|N}^{\theta} (\x_{k-1}) \le \left( \lambda_k^Q+\frac{K_k^Q}{R_k} \right) V_{k-1|N}^{\theta} (\x_{k-1})
+ \frac{ K_k^Q }{ \varepsilon^Q_k(R_k)c_{k,R_k}^Q } \un_{C_{R_k}} (\x_{k-1}) \eqsp.
$$
By definition,
$$
\lambda_k^{\mathsf G} = \lambda_k^Q+\frac{K_k^Q}{R_k} \eqsp, \qquad K_k^{\mathsf G} = \frac{ K_k^Q }{ \varepsilon^Q_k(R_k)c_{k,R_k}^Q } \eqsp.
$$
Since $R_k>K_k^Q/(1-\lambda_k^Q)$, we have $\lambda_k^{\mathsf G}\in(0,1)$. This proves \eqref{eq:drift_cond_G}.
\end{proof}

\begin{corollary}[Contraction in Lyapunov weigthed norms]
\label{cor:contractive_forward_smoothing}

Fix $N\ge1$ and suppose that \Cref{ass:proposal_harris} and \Cref{ass:bounded_normalized_potentials}
hold. Let $\lambda_k^{\mathsf G}\in(0,1)$ and $K_k^{\mathsf G} \in[0,\infty) $ be the constants of \Cref{lem:harris_forward_smoothing}. For every $1\le k\le N$, choose
\begin{equation}
\label{eq:guided_contraction_radius}
R_k> \frac{2K_k^{\mathsf G}}{1-\lambda_k^{\mathsf G}} \eqsp,
\end{equation}
and define $\overline\varepsilon_k \eqdef \varepsilon_k^Q(R_k)c_{k,R_k}^Q \in(0,1]$.Choose $b_N\in(0,1]$ such that
$$
2b_NK_k^{\mathsf G} \le \overline\varepsilon_k \eqsp,
\qquad 1\le k\le N \eqsp.
$$
Such a choice exists since $N$ is finite, $K_k^{\mathsf G}<\infty$, and $\overline\varepsilon_k>0$. Define
\begin{equation}
\label{eq:guided_contraction_factor}
\alpha_{k,N}
\eqdef
\max\left\{ 1-\frac{\overline\varepsilon_k}{2} \eqsp , \eqsp \frac{ 1+\frac{b_N}{2}\lambda_k^{\mathsf G}R_k +b_NK_k^{\mathsf G}}{ 1 +\frac{b_N}{2}R_k } \right\}
\in(0,1) \eqsp.
\end{equation}
Then, for every $1\le k\le N$ and probability measures $\mu_1,\mu_2$ with finite $\twistedlyap{k-1}{N}[\theta]$-moments,
\begin{equation}
\label{eq:guided_one_step_contraction}
\vnorm[ \mu_1\fsmk{k}{N}[\theta] - \mu_2\fsmk{k}{N}[\theta] ][b_N\twistedlyap{k}{N}[\theta]]
\le \alpha_{k,N} \vnorm[ \mu_1-\mu_2 ][b_N\twistedlyap{k-1}{N}[\theta]] \eqsp.
\end{equation}
\end{corollary}

\begin{proof}
The following result adapts the weighted-total-variation
argument of \citet{HairerMattingly2008} to the time-inhomogeneous
forward-smoothing kernels. Following the two-region argument of
\citet{HairerMattingly2008}, fix $1\le k\le N$ and define the pairwise small set
$$
\overline C_k \eqdef \left\{ (\x,\x'): \twistedlyap{k-1}{N}[\theta](\x) + \twistedlyap{k-1}{N}[\theta](\x') \le R_k \right\} \eqsp.
$$
Let $\psi$ be measurable and such that
$$
|\psi| \le 1 + b_N \twistedlyap{k}{N}[\theta] \eqsp.
$$
By \Cref{lem:harris_forward_smoothing}, and using
$\un_{C_{R_k^{\rm tr}}}\le1$,
$$
\fsmk{k}{N}[\theta] \twistedlyap{k}{N}[\theta](\x) \le \lambda_k^{\mathsf G} \twistedlyap{k-1}{N}[\theta](\x) + K_k^{\mathsf G} \eqsp.
$$
Hence, for every $\x,\x'\in\rset^\xdim$,
\begin{align}
\label{eq:pointwise_drift_contraction}
\left| \fsmk{k}{N}[\theta]\psi(\x) - \fsmk{k}{N}[\theta]\psi(\x') \right|
\le 2 + b_N \lambda_k^{\mathsf G} \left( \twistedlyap{k-1}{N}[\theta](\x) + \twistedlyap{k-1}{N}[\theta](\x') \right) + 2 b_N K_k^{\mathsf G} \eqsp.
\end{align}

\paragraph{Outside $\overline C_k$.}
If $(\x,\x')\notin\overline C_k$, then
$$
\twistedlyap{k-1}{N}[\theta](\x) + \twistedlyap{k-1}{N}[\theta](\x') > R_k \eqsp.
$$
It follows that 
\begin{align}
\label{eq:outside_small_set_contraction}
\frac{ \left| \fsmk{k}{N}[\theta]\psi(\x) - \fsmk{k}{N}[\theta]\psi(\x') \right| }{ 2+ b_N \left( \twistedlyap{k-1}{N}[\theta](\x) + \twistedlyap{k-1}{N}[\theta](\x') \right)} 
&  \le \lambda_k^{\mathsf G} + \frac{ 2(1-\lambda_k^{\mathsf G})+2b_NK_k^{\mathsf G} }{ 2+b_NR_k } \nonumber \\ 
& \le \frac{ 1+\frac{b_N}{2}\lambda_k^{\mathsf G}R_k + b_NK_k^{\mathsf G} }{ 1+\frac{b_N}{2}R_k} \eqsp.
\end{align}

By \eqref{eq:guided_contraction_radius}, $2K_k^{\mathsf G} < (1-\lambda_k^{\mathsf G})R_k$, so the right-hand side of
\eqref{eq:outside_small_set_contraction} is strictly smaller than one.

\paragraph{Inside $\overline C_k$.}
Since
$$
\twistedlyap{k-1}{N}[\theta](\x)
= \frac{V(\x)}{\normbackwardfun{k-1}{N}[\theta](\x)}
$$
and $ 0 < \normbackwardfun{k-1}{N}[\theta](\x) \le 1$, if $(\x,\x')\in\overline C_k$, then
$$
V(\x)\le \twistedlyap{k-1}{N}[\theta](\x)\le R_k \eqsp,
\qquad V(\x')\le \twistedlyap{k-1}{N}[\theta](\x')\le R_k \eqsp.
$$
Therefore, by the minorization part of
\Cref{lem:harris_forward_smoothing},
$$
\fsmk{k}{N}[\theta](\x,\cdot) \ge \overline\varepsilon_k \nu_{k,R_k}^{\mathsf G}(\cdot) \eqsp,
\qquad \fsmk{k}{N}[\theta](\x',\cdot) \ge \overline\varepsilon_k \nu_{k,R_k}^{\mathsf G}(\cdot) \eqsp.
$$
Hence the measures
$$
\fsmk{k}{N}[\theta](\x,\cdot) - \overline\varepsilon_k\nu_{k,R_k}^{\mathsf G}(\cdot) \eqsp,
\qquad \fsmk{k}{N}[\theta](\x',\cdot) - \overline\varepsilon_k\nu_{k,R_k}^{\mathsf G}(\cdot) \eqsp,
$$
are nonnegative and both have total mass $1-\overline\varepsilon_k$. Therefore, using $|\psi|\le1+b_N\twistedlyap{k}{N}[\theta]$,
\begin{align*}
& \left| \fsmk{k}{N}[\theta]\psi(\x) - \fsmk{k}{N}[\theta]\psi(\x') \right| \\
&= \Big| \left( \fsmk{k}{N}[\theta](\x,\cdot) - \overline\varepsilon_k\nu_{k,R_k}^{\mathsf G} \right)[\psi] - \left( \fsmk{k}{N}[\theta](\x',\cdot) - \overline\varepsilon_k\nu_{k,R_k}^{\mathsf G} \right)[\psi] \Big| \\
&\le 2(1-\overline\varepsilon_k) + b_N\fsmk{k}{N}[\theta] \twistedlyap{k}{N}[\theta](\x) + b_N \fsmk{k}{N}[\theta] \twistedlyap{k}{N}[\theta](\x') \eqsp.
\end{align*}
Using \Cref{lem:harris_forward_smoothing},
\begin{multline}
\label{eq:inside_small_set_contraction}
\left| \fsmk{k}{N}[\theta]\psi(\x) - \fsmk{k}{N}[\theta]\psi(\x') \right| \le 2(1-\overline\varepsilon_k) + b_N \lambda_k^{\mathsf G} \left( \twistedlyap{k-1}{N}[\theta](\x) + \twistedlyap{k-1}{N}[\theta](\x') \right) + 2 b_N K_k^{\mathsf G} \eqsp.
\end{multline}
Dividing by $2+ b_N \left( \twistedlyap{k-1}{N}[\theta](\x) + \twistedlyap{k-1}{N}[\theta](\x') \right)$
gives
\begin{align*}
\frac{ \left| \fsmk{k}{N}[\theta]\psi(\x) - \fsmk{k}{N}[\theta]\psi(\x') \right|}{ 2+ b_N \left( \twistedlyap{k-1}{N}[\theta](\x) + \twistedlyap{k-1}{N}[\theta](\x') \right) }  & \le \frac{2}{2+b_N \left( \twistedlyap{k-1}{N}[\theta](\x) +\twistedlyap{k-1}{N}[\theta](\x') \right)}
\left(1-\overline\varepsilon_k+b_NK_k^{\mathsf G}\right) \\
& \qquad \qquad + \lambda_k^{\mathsf G} \frac{b_N \left( \twistedlyap{k-1}{N}[\theta](\x) +  \twistedlyap{k-1}{N}[\theta](\x') \right)  }{
2+b_N \left( \twistedlyap{k-1}{N}[\theta](\x)  + \twistedlyap{k-1}{N}[\theta](\x') \right) } 
\\
& \le \max\left\{ 1-\overline\varepsilon_k + b_N K_k^{\mathsf G} \eqsp, \eqsp \lambda_k^{\mathsf G} \right\} \eqsp.
\end{align*}
Since $2b_N K_k^{\mathsf G} \le \overline\varepsilon_k$ by definition of $b_N$ we have
$$
1-\overline\varepsilon_k+b_NK_k^{\mathsf G} \le 1-\frac{\overline\varepsilon_k}{2}.
$$
Moreover,
$$
\lambda_k^{\mathsf G} < \frac{ 1+\frac{b_N}{2}\lambda_k^{\mathsf G}R_k + b_N K_k^{\mathsf G} }{ 1+\frac{b_N}{2}R_k }
\eqsp.
$$
Combining the two regions gives, for every $\x,\x'\in\rset^\xdim$ and every measurable $\psi$ satisfying
$|\psi|\le1+b_N\twistedlyap{k}{N}[\theta]$,
\begin{equation}
\label{eq:pointwise_adapted_contraction}
\left| \fsmk{k}{N}[\theta]\psi(\x) - \fsmk{k}{N}[\theta]\psi(\x') \right|
\le \alpha_{k,N} \left[ 2+ b_N \left( \twistedlyap{k-1}{N}[\theta](\x) + \twistedlyap{k-1}{N}[\theta](\x') \right) \right] \eqsp.
\end{equation}

We now deduce the contraction between probability measures. Let $\eta\eqdef\mu_1-\mu_2$ so that $\eta(\rset^\xdim)=0$. Let
$$
\eta=\eta^+-\eta^-
$$
be its Jordan decomposition. If $\eta=0$, the result is immediate.
Otherwise, since $\eta(\rset^\xdim)=0$,
$$
s \eqdef \eta^+(\rset^\xdim) = \eta^-(\rset^\xdim) > 0 \eqsp,
$$
and the measures
$$
\pi^+\eqdef\frac{\eta^+}{s} \eqsp, \qquad \pi^-\eqdef\frac{\eta^-}{s}
$$
belong to $\mathcal P(\rset^\xdim)$. Moreover,
$$
\eta=s(\pi^+-\pi^-) \eqsp.
$$
Hence, for every measurable $|\psi|\le1+b_N\twistedlyap{k}{N}[\theta]$,
\begin{align*}
\left| \left(\eta\fsmk{k}{N}[\theta]\right)[\psi] \right|
&= s \left| \left(\pi^+\fsmk{k}{N}[\theta]\right)[\psi] - \left(\pi^-\fsmk{k}{N}[\theta]\right)[\psi] \right| \\
&= s\left| \int\!\!\int \left( \fsmk{k}{N}[\theta]\psi(\x) - \fsmk{k}{N}[\theta]\psi(\x') \right) \pi^+(\rmd\x)\pi^-(\rmd\x') \right| \\
& \le \alpha_{k,N}s \int\!\!\int \left[ 2+ b_N \twistedlyap{k-1}{N}[\theta](\x) + b_N \twistedlyap{k-1}{N}[\theta](\x') \right] \pi^+(\rmd\x)\pi^-(\rmd\x') \\
&= \alpha_{k,N}s \left[ 2 + b_N \pi^+\!\left[ \twistedlyap{k-1}{N}[\theta] \right] + b_N \pi^-\!\left[ \twistedlyap{k-1}{N}[\theta] \right] \right] \\
&= \alpha_{k,N} |\eta| \left[ 1+ b_N \twistedlyap{k-1}{N}[\theta] \right] \\
&= \alpha_{k,N} \vnorm[ \eta ][b_N \twistedlyap{k-1}{N}[\theta]] \eqsp.
\end{align*}
Taking the supremum over all measurable $|\psi|\le1+b_N\twistedlyap{k}{N}[\theta]$ yields
\begin{equation} \label{eq:contraction_one_step_forward}
\vnorm[ \eta\fsmk{k}{N}[\theta] ][b_N\twistedlyap{k}{N}[\theta]] \le \alpha_{k,N} \vnorm[ \eta ][b_N\twistedlyap{k-1}{N}[\theta]] \eqsp.
\end{equation}
Since $\eta=\mu_1-\mu_2$, this is precisely \eqref{eq:guided_one_step_contraction}.
\end{proof}

\begin{corollary}[Finite-horizon forward-smoothing stability]
\label{cor:finite_horizon_forward_smoothing_stability}
Under the assumptions of \Cref{cor:contractive_forward_smoothing}, define
\begin{equation}
\label{eq:def:_gamma_ell}
\Gamma_{\ell,N} \eqdef
\begin{cases}
\displaystyle \frac1{b_N} \prod_{k=\ell+1}^{N}\alpha_{k,N} \eqsp, & 0\le\ell<N \eqsp, \\[2mm]
1 \eqsp, & \ell=N \eqsp.
\end{cases}
\end{equation}
Then, for every $0\le\ell\le N$ and probability measures $\mu_1,\mu_2$ with finite $\twistedlyap{\ell}{N}[\theta]$-moments,
\begin{equation}
\label{eq:finite_horizon_forward_smoothing}
\vnorm[ \mu_1\prodfsmk{\ell}{N}[\theta] - \mu_2\prodfsmk{\ell}{N}[\theta] ][V]
\le \Gamma_{\ell,N} \vnorm[ \mu_1-\mu_2 ][\twistedlyap{\ell}{N}[\theta]] \eqsp. 
\end{equation}
In particular,
$$
\vnorm[ \mu_1\prodfsmk{\ell}{N}[\theta] - \mu_2\prodfsmk{\ell}{N}[\theta] ][V] \le \Gamma_{\ell,N} \left\{ \mu_1[1+\twistedlyap{\ell}{N}[\theta]] + \mu_2[1+\twistedlyap{\ell}{N}[\theta]] \right\} \eqsp.
$$
\end{corollary}

\begin{proof}
The case $\ell=N$ is immediate. Let $0\le\ell<N$, we first iterate the result of \Cref{cor:contractive_forward_smoothing} 
\begin{align*}
&\vnorm[ \mu_1\prodfsmk{\ell}{N}[\theta] - \mu_2\prodfsmk{\ell}{N}[\theta] ][b_N\twistedlyap{N}{N}[\theta]] \le \left( \prod_{k=\ell+1}^{N}\alpha_{k,N} \right) \vnorm[ \mu_1-\mu_2 ][b_N\twistedlyap{\ell}{N}[\theta]] \eqsp.
\end{align*}
Since, by definition
$$
\twistedlyap{N}{N}[\theta]=V \eqsp,
$$
we obtain
\begin{equation}
\label{eq:guided_product_contraction}
\vnorm[ \mu_1\prodfsmk{\ell}{N}[\theta] - \mu_2\prodfsmk{\ell}{N}[\theta] ][b_NV] \le \left( \prod_{k=\ell+1}^{N}\alpha_{k,N} \right) \vnorm[ \mu_1-\mu_2 ][b_N\twistedlyap{\ell}{N}[\theta]] \eqsp.
\end{equation}
For $0\le\ell<N$, since $0<b_N\le1$,
$$
1+V\le\frac1{b_N}(1+b_NV) \eqsp, \qquad
1+b_N\twistedlyap{\ell}{N}[\theta] \le 1+\twistedlyap{\ell}{N}[\theta] \eqsp.
$$
Therefore, by \eqref{eq:guided_product_contraction},
\begin{align*}
\vnorm[ \mu_1\prodfsmk{\ell}{N}[\theta] - \mu_2\prodfsmk{\ell}{N}[\theta] ][V] 
& \le \frac1{b_N} \vnorm[ \mu_1\prodfsmk{\ell}{N}[\theta] - \mu_2\prodfsmk{\ell}{N}[\theta]][b_NV] \\
& \le \frac1{b_N} \left( \prod_{k=\ell+1}^{N}\alpha_{k,N} \right) \vnorm[ \mu_1-\mu_2 ][b_N\twistedlyap{\ell}{N}[\theta]]\\
& \le \Gamma_{\ell,N} \vnorm[ \mu_1-\mu_2 ][\twistedlyap{\ell}{N}[\theta]] \eqsp.
\end{align*}
The final bound follows from
$|\mu_1-\mu_2|\le\mu_1+\mu_2$.
\end{proof}

\section{Monte Carlo Error}
\label{app:monte_carlo}

\subsection{Detailed assumptions}
\label{app:detailed_monte_carlo_assumptions}

To control the finite-particle approximation error in $L_q$ norm. We need the following local moment condition ensuring that the selection and mutation empirical averages have finite $L_q$ fluctuations.

\begin{assumption}[Local particle moments]
\label{ass:local_lp_bound_moment}
Let $q\ge2$. Assume
$$
\approxfilter{0}[(1+V)^q]<\infty,
$$
and, for every $1\le k\le N$,
$$
\sup_{M\ge1} \E\left[ \tau_{\approxpot{k-1}}(\mcfilter{k-1}) \left[ \mk{k}[\theta,h](1+V)^q \right] \right] < \infty \eqsp.
$$
\end{assumption}
Under \Cref{ass:local_lp_bound_moment}, the local Monte Carlo
fluctuations are of order $M^{-1/2}$; see
\Cref{lem:local_error_lp}. Their propagation through the normalized
Feynman--Kac flow additionally introduces random normalization
factors.
For $0\le\ell\le N$, define the one-step predictive law
$$
\eta_\ell^{M,\theta} \eqdef \begin{cases} \approxfilter{0}, & \ell=0,\\ \filtertransform{\ell}[\theta](\mcfilter{\ell-1}), & 1\le \ell\le N \eqsp.
\end{cases}
$$
\begin{assumption}
\label{ass:mc_normalization_moments}
Let $p\ge1$. For every $0\le \ell\le N$, assume that
$$
\mathcal A_{\ell,N} \eqdef \sup_{M\ge1} \normEc{ \frac{1}{ \mcfilter{\ell} [ \normbackwardfun{\ell}{N}[\theta] ]} \left( 1+ \frac{ \eta_\ell^{M,\theta} [ \normbackwardfun{\ell}{N}[\theta]+V ] }{ \eta_\ell^{M,\theta} [ \normbackwardfun{\ell}{N}[\theta] ] } \right) }_{L_{2p}} <\infty \eqsp.
$$
\end{assumption}
For the $L_p$ particle-error bound below,
\Cref{ass:local_lp_bound_moment} is applied with $q=2p$.
Under these conditions, the finite-particle approximation is controlled
by \Cref{thm:final_smc_error_bound}.

\subsection{Proof}

The proof follows the standard local-error propagation strategy for
Feynman--Kac particle systems; see, for example,
\citet[Section~7.4.3]{delmoral2004feynman} and
\citet{whiteley2012sequential}.

For $0\le \ell\le N$, define the future normalized approximate
Feynman--Kac map
$$
\filtertransform{\ell:N}[\theta]
\eqdef \filtertransform{N}[\theta]\circ\cdots\circ \filtertransform{\ell+1}[\theta] \eqsp,
\qquad \filtertransform{N:N}[\theta]\eqdef \operatorname{Id} \eqsp.
$$
Thus, for every probability measure $\nu$,
$$
\filtertransform{\ell:N}[\theta](\nu) = \left( \filtertransform{N}[\theta]\circ\cdots\circ \filtertransform{\ell+1}[\theta] \right)(\nu) \eqsp.
$$

\paragraph{Selected and propagated empirical measures.}
At time $k-1$, considering $M$ particles, let
$$
\mathcal F_{k-1}^{\theta} \eqdef \sigma\left( \X_\ell^{\theta,(i)},\,1\le i\le M,\,0\le \ell\le k-1 \right) \eqsp.
$$
Given $\mcfilter{k-1}$, the \emph{selection step} draws
$$
\widetilde \X_{k-1}^{\theta,(1)},\ldots, \widetilde \X_{k-1}^{\theta,(M)}\mid \mathcal F_{k-1}^{\theta} \overset{\mathrm{i.i.d.}}{\sim} \tau_{\approxpot{k-1}}(\mcfilter{k-1}) \eqsp.
$$
We define
$$
\tildemcfilter{k-1} \eqdef \frac1M\sum_{i=1}^{M} \delta_{\widetilde \X_{k-1}^{\theta,(i)}} \eqsp.
$$
Let
$$
\widetilde{\mathcal F}_{k-1}^{\theta}
\eqdef \sigma\left( \mathcal F_{k-1}^{\theta}, \widetilde \X_{k-1}^{\theta,(1)},\ldots, \widetilde \X_{k-1}^{\theta,(M)} \right) \eqsp.
$$
The \emph{mutation step} draws, conditionally independently given $\widetilde{\mathcal F}_{k-1}^{\theta}$,
$$
\X_k^{\theta,(i)} \mid \widetilde{\mathcal F}_{k-1}^{\theta} \sim \mk{k}[\theta,h](\widetilde \X_{k-1}^{\theta,(i)},\rmd \x) \eqsp, \qquad i=1,\ldots,M \eqsp.
$$
We define
$$
\mcfilter{k} \eqdef \frac1M \sum_{i=1}^{M} \delta_{\X_k^{\theta,(i)}} \eqsp.
$$
Consequently, for every bounded measurable function $\psi$,
$$
\E\left[ \mcfilter{k}[\psi] \middle| \mathcal F_{k-1}^{\theta} \right]
= \tau_{\approxpot{k-1}}(\mcfilter{k-1})\mk{k}[\theta,h] \left[ \psi \right] 
= \filtertransform{k}[\theta](\mcfilter{k-1})[\psi].
$$

\paragraph{Local selection and mutation errors.}

For every bounded measurable test function $\psi$, define
$$
\mcfilter{k}[\psi] - \filtertransform{k}[\theta](\mcfilter{k-1})[\psi] = \muterr{k}[\psi] + \selerr{k}[\psi] \eqsp,
$$
where 
$$
\muterr{k}[\psi] \eqdef \mcfilter{k}[\psi] - \frac1M \sum_{i=1}^{M} \mk{k}[\theta,h]\psi(\widetilde \X_{k-1}^{\theta,(i)}) \eqsp, 
$$
is the \emph{mutation error} and 
$$
\selerr{k}[\psi] \eqdef \frac1M \sum_{i=1}^{M} \mk{k}[\theta,h]\psi(\widetilde \X_{k-1}^{\theta,(i)}) - \left( \tau_{\approxpot{k-1}}(\mcfilter{k-1})\mk{k}[\theta,h] \right)[\psi] \eqsp,
$$
is the \emph{selection error}. Then
$$
\E[ \muterr{k}[\psi] \mid \widetilde{\mathcal F}^\theta_{k-1} ] = 0 \eqsp,
\qquad
\E[\selerr{k}[\psi]\mid\mathcal F^\theta_{k-1}] = 0 \eqsp.
$$

\paragraph{Local $L_p$ control.}
We now quantify the size of the local empirical errors introduced at each selection and mutation step. Since the final estimates are stated in a $V$-weighted norm, we use the corresponding dual weighted supremum norm. For a measurable function $\psi:\rset^\xdim\to\rset$, define
$$
\|\psi\|_{V,\infty} \eqdef \sup_{\x\in\rset^\xdim} \frac{|\psi(\x)|}{1+V(\x)} \eqsp.
$$

\begin{lemma}
\label{lem:local_error_lp}
Let $q\ge2$ and suppose that \Cref{ass:local_lp_bound_moment} holds. Then
there exists a constant $B_q<\infty$, depending only on $q$, such that for
every measurable $\psi:\rset^\xdim\to\rset$ with $\|\psi\|_{V,\infty}<\infty$ ,
\begin{equation} \label{eq:init_smc}
\normEc{ (\mcfilter{0}-\approxfilter{0})[\psi] }_{L_q} \le\frac{C_0^{(q)}}{\sqrt M} \|\psi\|_{V,\infty} \eqsp,
\end{equation}
and, for every $1\le k\le N$,
\begin{equation} \label{eq:local_lp_total_final}
\normEc{ \mcfilter{k}[\psi] - \left(\filtertransform{k}[\theta](\mcfilter{k-1})\right)[\psi] }_{L_q}
\le \frac{C_k^{S,(q)} + C_k^{M,(q)} }{\sqrt M} \|\psi\|_{V,\infty} \eqsp,
\end{equation}
where
$$
C_0^{(q)} \eqdef 2B_q \approxfilter{0}[(1+V)^q]^{1/q} \eqsp,
$$
$$
C_k^{S,(q)} \eqdef 2B_q\left( \sup_{M\ge1} \E\left[\tau_{\approxpot{k-1}}(\mcfilter{k-1})[\mk{k}[\theta,h](1+V)^q]\right]\right)^{1/q} \eqsp,
$$
and
$$
C_k^{M,(q)} \eqdef 2B_q \left( \sup_{M\ge1} \E\left[ \tildemcfilter{k-1} [ \mk{k}[\theta,h](1+V)^q ] \right] \right)^{1/q} \eqsp.
$$
\end{lemma}

\begin{proof}
        The proof is an application of Lemma 7.7.3 (Chapter 7.3: Inequalities for independent Random Variables) in \citet{delmoral2004feynman}. 
        
\noindent\emph{Initial empirical error.}
By construction, the initial particles $\X_0^{\theta, (1)},\ldots,\X_0^{\theta, (M)}$ are i.i.d. with common law $\approxfilter{0}$. Define
$$
U_i
\eqdef
\psi(\X_0^{\theta, (i)})-\approxfilter{0}[\psi],
\qquad i = 1, \ldots , M \eqsp.
$$
Then the variables $U_1,\ldots,U_M$ are i.i.d. and centered, and
$$
(\mcfilter{0}-\approxfilter{0})[\psi]
=
\frac1M\sum_{i=1}^M U_i \eqsp .
$$
Therefore, by Marcinkiewicz--Zygmund inequality,
$$
\normEc{(\mcfilter{0}-\approxfilter{0})[\psi]}_{L_q}
\le
\frac{B_q}{M}
\normEc{
\left(\sum_{i=1}^M |U_i|^2\right)^{1/2}
}_{L_q} \eqsp.
$$
Since $q/2\ge 1$, Jensen's inequality yields
$$
\left(\sum_{i=1}^M |U_i|^2\right)^{q/2}
\le
M^{\frac q2-1}\sum_{i=1}^M |U_i|^q \eqsp.
$$
Hence
$$
\normEc{(\mcfilter{0}-\approxfilter{0})[\psi]}_{L_q}
\le
\frac{B_q}{\sqrt M}
\left(
\frac1M\sum_{i=1}^M \E[|U_i|^q]
\right)^{1/q} \eqsp.
$$
Using $|a-b|^q\le 2^{q-1}(|a|^q+|b|^q)$ and Jensen's inequality,
\begin{align*}
\E[|U_i|^q] &\le 2^{q-1} \left( \E[|\psi(\X_0^{\theta,(i)})|^q] + |\approxfilter{0}[\psi]|^q \right) \le 2^q \approxfilter{0}[|\psi|^q] \eqsp.
\end{align*}
Since $|\psi(\x)|\le \normEc{\psi}_{V,\infty} (1+ V (\x))$, we get
$$
\approxfilter{0}[|\psi|^q] \le \normEc{\psi}_{V,\infty}^q  \approxfilter{0} \left[(1+ V)^q \right] \eqsp.
$$
Combining the last displays proves \eqref{eq:init_smc}.
        
\noindent\emph{Local selection error.} By definition,
        $$
        \selerr{k}[\psi]
        = \frac1M \sum_{i=1}^M \mk{k}[\theta,h] \psi (\widetilde \X^{\theta,(i)}_{k-1}) - \tau_{\approxpot{k-1}}(\mcfilter{k-1}) \mk{k}[\theta,h]{\psi} \eqsp.
        $$
        Conditionally on $\mathcal F^\theta_{k-1}$, the selected particles $ \widetilde \X^{\theta,(1)}_{k-1},\ldots,\widetilde \X^{\theta ,(M)}_{k-1} $
        are i.i.d.\ with common law $\tau_{\approxpot{k-1}}(\mcfilter{k-1})$. Hence the variables
        $$
        Y_i
        \eqdef \mk{k}[\theta,h]\psi(\widetilde \X^{\theta,(i)}_{k-1})
        - \tau_{\approxpot{k-1}}(\mcfilter{k-1})\mk{k}[\theta,h]{\psi} \eqsp ,
        \qquad i=1,\ldots,M \eqsp,
        $$
        are conditionally i.i.d. and centered. Therefore, by
        Marcinkiewicz-Zygmund inequality, there exists $B_q$ depending only on $q$ such that,
        \begin{align*}
        \normEc{\selerr{k}[\psi] }_{L_q (.|\mathcal F^\theta_{k-1})}
        &= \frac1M \normEc{  \sum_{i=1}^M Y_i }_{L_q (.|\mathcal F_{k-1}^\theta)} \le \frac{B_q}{M} \normEc{ \left(\sum_{i=1}^M |Y_i|^2\right)^{1/2} }_{L_q (.|\mathcal F^\theta_{k-1})}
        \end{align*}
        where $L_q (.|\mathcal F^\theta_{k-1})$ means the $L_q$ norms is taken conditional on $\mathcal F^\theta_{k-1}$. Since $q/2\ge 1$, the map $x\mapsto x^{q/2}$ is convex. Hence, by Jensen,
        $$
        \left( \sum_{i=1}^M |Y_i|^2 \right)^{q/2} \le M^{\frac q2-1} \sum_{i=1}^M |Y_i|^q \eqsp,
        $$
        Hence
        $$
        \normEc{ \selerr{k}[\psi] }_{L_q (.|\mathcal F^\theta_{k-1})} \le \frac{B_q}{\sqrt M} \left( \frac1M \sum_{i=1}^M \E \left[ |Y_i|^q\middle|\mathcal F^\theta_{k-1} \right] \right)^{1/q} \eqsp.
        $$
        Using $|a-b|^q\le 2^{q-1}(|a|^q+|b|^q)$ and Jensen's inequality with $q\ge 2$
        \begin{align*}
        \E \left[ |Y_i|^q \middle| \mathcal{F}^\theta_{k-1} \right]
        & \leq 2^{q-1} \left( \E \left[ | \mk{k}[\theta,h] \psi ( \widetilde \X^{\theta,(i)}_{k-1} ) |^q \middle|
        \mathcal F^\theta_{k-1} \right] + \left| \tau_{\approxpot{k-1}}( \mcfilter{k-1}) \mk{k}[\theta,h]{\psi} \right|^q \right) \\
        & \leq 2^q \tau_{\approxpot{k-1}}(\mcfilter{k-1})
        \left[ |\mk{k}[\theta,h]\psi|^q \right] \eqsp.
        \end{align*}
        Since $|\psi(\x)|\le \normEc{\psi}_{V,\infty}(1+\lyapunov{} (\x))$ and using Jensen's inequality again
        $$
        |\mk{k}[\theta,h]\psi(\x)|^q \le \mk{k}[\theta,h]|\psi|^q(\x) \le \|\psi\|_{V,\infty}^q \mk{k}[\theta,h](1+ V)^q (\x) \eqsp.
        $$
        Therefore,
        $$
        \normEc{ \selerr{k}[\psi] }_{L_q (.|\mathcal F^\theta_{k-1})} \leq \frac{2B_q}{\sqrt M} \normEc{\psi}_{V,\infty} \left( \tau_{\approxpot{k-1}}(\mcfilter{k-1})[\mk{k}[\theta,h](1+ V)^q] \right)^{1/q} \eqsp.
        $$
        Taking expectations and using the definition of $C_k^{S,(q)}$ yields        
        \begin{equation} \label{eq:up_selection}
        \normEc{\selerr{k}[\psi]}_{L_q}  \le \frac{C_k^{S,(q)}}{\sqrt M} \normEc{\psi}_{V,\infty} \eqsp,
        \end{equation}

        the claimed bound for the selection error.

        \noindent\emph{Local mutation error.} By definition,
        $$
        \muterr{k}[\psi] = \mcfilter{k}[\psi] -\frac1M \sum_{i=1}^M \mk{k}[\theta,h]\psi( \widetilde \X^{\theta,(i)}_{k-1} ) \eqdef \frac1M \sum_{i=1}^M Z_i \eqsp,
        $$
        where
        $$
        Z_i \eqdef \psi(\X^{\theta,(i)}_{k}) - \mk{k}[\theta,h]\psi(\widetilde \X^{\theta,(i)}_{k-1}) \eqsp, \qquad i=1,\ldots,M \eqsp.
        $$
        Conditionally on $\widetilde{\mathcal F}^\theta_{k-1}$, the propagated particles $\X^{\theta,(1)}_{k},\ldots,\X^{\theta,(M)}_{k} $ are independent and satisfy
        $$
        \X^{\theta, (i)}_{k}\mid \widetilde{\mathcal F}^\theta_{k-1}
        \sim \mk{k}[\theta,h](\widetilde \X^{\theta,(i)}_{k-1} , \rmd \x) \eqsp. 
        $$
        Thus the variables $Z_1,\ldots,Z_M$ are conditionally independent and centered.
        Applying again the conditional Marcinkiewicz--Zygmund inequality conditional on $\widetilde{\mathcal F}^\theta_{k-1}$ this time and then following the same steps as for the selection error yields
        $$
        \normEc{ \muterr{k}[\psi] }_{L_q(. \mid\widetilde{\mathcal F}^\theta_{k-1})}
        \le \frac{2B_q}{\sqrt M} \normEc{\psi}_{V,\infty} 
        \left( \tildemcfilter{k-1}[ \mk{k}[\theta,h] (1+V)^q ] \right)^{1/q} \eqsp.
        $$
        Taking expectations and using the definition of $C_k^{M,(q)}$ yields        
        \begin{equation} \label{eq:up_mutation}
            \normEc{\muterr{k}[\psi]}_{L_q} \le \frac{C_k^{M,(q)}}{\sqrt M} \normEc{\psi}_{V,\infty} \eqsp,
        \end{equation}
        Finally, \eqref{eq:local_lp_total_final} follows  from
        $$
        \mcfilter{k}[\psi] - \filtertransform{k}[\theta](\mcfilter{k-1}) [\psi] = \muterr{k}[\psi] + \selerr{k}[\psi] 
        $$
        and Minkowski's inequality, together with the upper bound given \eqref{eq:up_selection} for the selection error and with upper bound \eqref{eq:up_mutation} for the mutation error.

\noindent\emph{Integrability conditions.} Finally, it follows from \Cref{ass:local_lp_bound_moment} that the constants $C_0^{(q)}$ and $C_k^{S,(q)}$ are finite. For $C_k^{M,(q)}$, as
$$
\sup_{M\ge1} \E\left[ \tau_{\approxpot{k-1}}(\mcfilter{k-1}) [ \mk{k}[\theta,h](1+V)^q ] \right] < \infty \eqsp,
$$
by construction of the selected empirical measure,
$$
\E\left[ \tildemcfilter{k-1}[\mk{k}[\theta,h](1+V)^q]\middle|\mathcal F_{k-1}^{\theta}\right]=\tau_{\approxpot{k-1}}(\mcfilter{k-1})[
\mk{k}[\theta,h](1+V)^q] \eqsp.
$$
Hence, 
$$
\sup_{M\ge1} \E\left[ \tildemcfilter{k-1}[\mk{k}[\theta,h](1+V)^q]\right]=\sup_{M\ge1}\E\left[\tau_{\approxpot{k-1}}(\mcfilter{k-1})[ \mk{k}[\theta,h](1+V)^q]\right]< \infty \eqsp.
$$
\end{proof}

\begin{lemma}[Monte Carlo error under forward-smoothing forgetting]
\label{thm:final_smc_error_bound}
Let $p\ge1$. Suppose that Assumptions \ref{ass:bounded_normalized_potentials} \ref{ass:proposal_harris}, \ref{ass:local_lp_bound_moment} (with $q=2p$) and \ref{ass:mc_normalization_moments}  hold. Then, for every measurable $\psi:\rset^\xdim\to\rset$ such that
$\|\psi\|_{V,\infty}<\infty$,
$$
\normEc{ \mcfilter{N}[\psi]-\approxfilter{N}[\psi] }_{L_p} \le \frac{\|\psi\|_{V,\infty}}{\sqrt M} \sum_{\ell=0}^{N} \mathcal{C}_\ell^{\rm MC} \mathcal A_{\ell,N} \Gamma_{\ell,N} \eqsp,
$$
where $\mathcal{C}_0^{\rm MC}$ and $\mathcal{C}_\ell^{\rm MC}$ are the constants of
\Cref{lem:local_error_lp} applied with exponent $2p$, namely
$$
\mathcal{C}_0^{\rm MC}\eqdef C_0^{(2p)} \eqsp,
\qquad
\mathcal{C}_\ell^{\rm MC}\eqdef C_\ell^{S,(2p)}+C_\ell^{M,(2p)} \eqsp, 
\qquad 1\le \ell\le N \eqsp.
$$
\end{lemma}

\begin{proof}
The proof follows the standard local-error propagation argument for Feynman--Kac particle systems; see, for instance, \citet[Section~7.4.3]{delmoral2004feynman} and \citet{whiteley2012sequential}. By a telescoping sum argument,
\begin{equation} \label{eq:preuve_mc_telescoping}
\mcfilter{N}[\psi]-\approxfilter{N}[\psi] = \sum_{\ell=0}^{N} \left\{ \filtertransform{\ell:N}[\theta](\mcfilter{\ell})[\psi] - \filtertransform{\ell:N}[\theta](\eta_\ell^{M,\theta})[\psi] \right\} \eqsp,
\end{equation}
where $\eta_\ell^{M,\theta}$ is defined in \Cref{ass:mc_normalization_moments}. Fix $0\le \ell\le N$. By \Cref{prop:approx_tail_factorization}, 
$$
\filtertransform{\ell:N}[\theta](\mcfilter{\ell} )[\psi] = \frac{ \mcfilter{\ell} \left[ \normbackwardfun{\ell}{N}[\theta] \prodfsmk{\ell}{N}[\theta]\psi \right] }{ \mcfilter{\ell} \left[ \normbackwardfun{\ell}{N}[\theta] \right]} \eqsp.
$$
Hence,
\begin{align} \label{eq:preuve_mc_time_k}
\filtertransform{\ell:N}[\theta](\mcfilter{\ell})[\psi] - \filtertransform{\ell:N}[\theta](\eta_\ell^{M,\theta})[\psi] = \frac{\mcfilter{\ell}\left[\normbackwardfun{\ell}{N}[\theta] \left( \prodfsmk{\ell}{N}[\theta]\psi-\filtertransform{\ell:N}[\theta](\eta_\ell^{M,\theta})[\psi] \right) \right] }{ \mcfilter{\ell} \left[ \normbackwardfun{\ell}{N}[\theta] \right] } \eqsp.
\end{align}
Define the centered test function
$$
\varphi_{\ell,N} \eqdef \frac{ \normbackwardfun{\ell}{N}[\theta] \left( \prodfsmk{\ell}{N}[\theta]\psi - \filtertransform{\ell:N}[\theta](\eta_\ell^{M,\theta})[\psi]
\right) }{ \eta_\ell^{M,\theta} \left[ \normbackwardfun{\ell}{N}[\theta]\right] } \eqsp.
$$
By \Cref{prop:approx_tail_factorization},
$$
\filtertransform{\ell:N}[\theta](\eta_\ell^{M,\theta})[\psi]
= \frac{ \eta_\ell^{M,\theta}\left[ \normbackwardfun{\ell}{N}[\theta] \prodfsmk{\ell}{N}[\theta]\psi\right]}{\eta_\ell^{M,\theta}\left[ \normbackwardfun{\ell}{N}[\theta]\right] } \eqsp,
$$
and therefore,
\begin{align*}
\eta_\ell^{M,\theta}[\varphi_{\ell,N}]
= \frac{ \eta_\ell^{M,\theta} \left[ \normbackwardfun{\ell}{N}[\theta]\prodfsmk{\ell}{N}[\theta]\psi\right]-\filtertransform{\ell:N}[\theta](\eta_\ell^{M,\theta})[\psi] \eta_\ell^{M,\theta} \left[ \normbackwardfun{\ell}{N}[\theta] \right] }{ \eta_\ell^{M,\theta} \left[ \normbackwardfun{\ell}{N}[\theta] \right] } = 0 \eqsp.
\end{align*}
Hence combining the above with \eqref{eq:preuve_mc_time_k} yields
\begin{align} \label{eq:preuve_mc_time_k2}
\filtertransform{\ell:N}[\theta](\mcfilter{\ell})[\psi] - \filtertransform{\ell:N}[\theta](\eta_\ell^{M,\theta})[\psi] 
=
\frac{ \eta_\ell^{M,\theta} [ \normbackwardfun{\ell}{N}[\theta] ] }{\mcfilter{\ell}[\normbackwardfun{\ell}{N}[\theta]]}\left(\mcfilter{\ell}-\eta_\ell^{M,\theta}\right) [ \varphi_{\ell,N} ] \eqsp.
\end{align}

Applying  \Cref{cor:finite_horizon_forward_smoothing_stability} to
$$
\delta_\x - \frac{ \normbackwardfun{\ell}{N}[\theta]\eta_\ell^{M,\theta} }{ \eta_\ell^{M,\theta} [ \normbackwardfun{\ell}{N}[\theta] ] } \eqsp,
$$
we obtain
\begin{align*}
\left| \prodfsmk{\ell}{N}[\theta]\psi(\x) - \filtertransform{\ell:N}[\theta](\eta_\ell^{M,\theta})[\psi] \right| 
& \le \|\psi\|_{V,\infty}  \vnorm[ \delta_\x \prodfsmk{\ell}{N}[\theta]  - \filtertransform{\ell:N}[\theta](\eta_\ell^{M,\theta})][V]  \\
& \le \|\psi\|_{V,\infty} \Gamma_{\ell,N} \left\{ 1+ \frac{V(\x)}{\normbackwardfun{\ell}{N}[\theta](\x)} + \frac{ \eta_\ell^{M,\theta} [ \normbackwardfun{\ell}{N}[\theta]+V ] }{ \eta_\ell^{M,\theta} [ \normbackwardfun{\ell}{N}[\theta] ] } \right\} \eqsp.
\end{align*}
Since $0<\normbackwardfun{\ell}{N}[\theta]\le1$, the previous bound implies
\begin{align*}
\frac{ |\varphi_{\ell,N} (\x)| }{ 1+V(\x) }
&\le \|\psi\|_{V,\infty}
\frac{  \Gamma_{\ell,N}  }{ \eta_\ell^{M,\theta} [ \normbackwardfun{\ell}{N}[\theta] ] } 
\frac{ \normbackwardfun{\ell}{N}[\theta](\x) }{ 1+V(\x)} \left\{ 1+ \frac{V(\x)}{\normbackwardfun{\ell}{N}[\theta](\x)} + \frac{ \eta_\ell^{M,\theta} [ \normbackwardfun{\ell}{N}[\theta]+V ] }{ \eta_\ell^{M,\theta} [ \normbackwardfun{\ell}{N}[\theta] ] }  \right\} \\
&= \|\psi\|_{V,\infty}
\frac{ \Gamma_{\ell,N} }{\eta_\ell^{M,\theta}[\normbackwardfun{\ell}{N}[\theta]]}\left\{ \frac{ \normbackwardfun{\ell}{N}[\theta](\x)+V(\x)}{ 1+V(\x)} + \frac{ \eta_\ell^{M,\theta} [ \normbackwardfun{\ell}{N}[\theta]+V ] }{ \eta_\ell^{M,\theta} [ \normbackwardfun{\ell}{N}[\theta] ] }  \frac{ \normbackwardfun{\ell}{N}[\theta](\x) }{ 1+V(\x) } \right\} \\
&\le \|\psi\|_{V,\infty} \frac{ \Gamma_{\ell,N} }{\eta_\ell^{M,\theta}
[ \normbackwardfun{\ell}{N}[\theta] ] } \left\{ 1+ \frac{ \eta_\ell^{M,\theta} [ \normbackwardfun{\ell}{N}[\theta]+V ] }{ \eta_\ell^{M,\theta} [ \normbackwardfun{\ell}{N}[\theta] ] }  \right\} \eqsp,
\end{align*}
Taking the supremum over $\x\in\rset^\xdim$, we obtain
\begin{equation} \label{eq:bound_centered_function}
\|\varphi_{\ell,N}\|_{V,\infty} \le \|\psi\|_{V,\infty} \frac{ \Gamma_{\ell,N} }{ \eta_\ell^{M,\theta} [ \normbackwardfun{\ell}{N}[\theta] ] } \left\{
1+ \frac{ \eta_\ell^{M,\theta} [ \normbackwardfun{\ell}{N}[\theta]+V ] }{ \eta_\ell^{M,\theta} [ \normbackwardfun{\ell}{N}[\theta] ] } 
\right\} \eqsp.
\end{equation}
Using \eqref{eq:preuve_mc_time_k2} and Hölder's inequality we get
\begin{align*}
& \normEc{ \filtertransform{\ell:N}[\theta](\mcfilter{\ell})[\psi] - \filtertransform{\ell:N}[\theta](\eta_\ell^{M,\theta})[\psi] }_{L_p} \\
& \le \normEc{ \frac{ \eta_\ell^{M,\theta} [ \normbackwardfun{\ell}{N}[\theta] ]
}{ \mcfilter{\ell}[ \normbackwardfun{\ell}{N}[\theta] ] } \|\varphi_{\ell,N}\|_{V,\infty} \left| \left(\mcfilter{\ell}-\eta_\ell^{M,\theta}\right) \left[ \frac{\varphi_{\ell,N}}{\|\varphi_{\ell,N}\|_{V,\infty}} \right]\right| }_{L_p} \\
& \le \normEc{ \frac{ \eta_\ell^{M,\theta} [ \normbackwardfun{\ell}{N}[\theta] ] }{ \mcfilter{\ell} [ \normbackwardfun{\ell}{N}[\theta] ] } \|\varphi_{\ell,N}\|_{V,\infty} }_{L_{2p}} \normEc{ \left| \left(\mcfilter{\ell}-\eta_\ell^{M,\theta}\right) \left[ \frac{\varphi_{\ell,N}}{\|\varphi_{\ell,N}\|_{V,\infty}} \right] \right| }_{L_{2p}} \\
\end{align*}
Using \Cref{lem:local_error_lp},
\begin{align*}
\normEc{ \filtertransform{\ell:N}[\theta](\mcfilter{\ell})[\psi] - \filtertransform{\ell:N}[\theta](\eta_\ell^{M,\theta})[\psi] }_{L_p} 
 \le \normEc{ \frac{ \eta_\ell^{M,\theta} [ \normbackwardfun{\ell}{N}[\theta] ] }{ \mcfilter{\ell} [ \normbackwardfun{\ell}{N}[\theta] ] } \|\varphi_{\ell,N}\|_{V,\infty} }_{L_{2p}} \frac{\mathcal{C}_\ell^{\rm MC}}{\sqrt M} \eqsp,
\end{align*}
Hence, it follows from \eqref{eq:bound_centered_function} and \Cref{ass:mc_normalization_moments} that
\begin{align*}
\normEc{ \frac{ \eta_\ell^{M,\theta} [ \normbackwardfun{\ell}{N}[\theta] ] }{ \mcfilter{\ell} [ \normbackwardfun{\ell}{N}[\theta] ] } \|\varphi_{\ell,N}\|_{V,\infty} }_{L_{2p}} 
& \le \Gamma_{\ell,N}\|\psi\|_{V,\infty} \normEc{ \frac{1}{\mcfilter{\ell} [ \normbackwardfun{\ell}{N}[\theta] ] } \left( 1+ \frac{ \eta_\ell^{M,\theta} [ \normbackwardfun{\ell}{N}[\theta]+V ] }{ \eta_\ell^{M,\theta} [ \normbackwardfun{\ell}{N}[\theta] ] } \right) }_{L_{2p}} \\
& \le \Gamma_{\ell,N}\|\psi\|_{V,\infty} \mathcal A_{\ell,N} \eqsp.
\end{align*}
Therefore,
$$
\normEc{ \filtertransform{\ell:N}[\theta](\mcfilter{\ell})[\psi] - \filtertransform{\ell:N}[\theta](\eta_\ell^{M,\theta})[\psi] }_{L_p}
\le \frac{ \mathcal{C}_\ell^{\rm MC} \mathcal A_{\ell,N} \Gamma_{\ell,N} }{\sqrt M} \|\psi\|_{V,\infty} \eqsp.
$$
Using \eqref{eq:preuve_mc_telescoping} together with Minkowski's inequality yields
$$
\normEc{ \mcfilter{N}[\psi]-\approxfilter{N}[\psi] }_{L_p} \le \frac{\|\psi\|_{V,\infty}}{\sqrt M} \sum_{\ell=0}^{N} \mathcal{C}_\ell^{\rm MC} \mathcal A_{\ell,N} \Gamma_{\ell,N} \eqsp,
$$
which finishes the proof.
\end{proof}

\section{Approximation bias}
\label{app:kernel_bias}

\subsection{Local discrepancies and moment control}

We first isolate the approximation error introduced at each SMC stage.
Set

$$
\mathcal B_0^{\rm bias}
\eqdef
\vnorm[\filter{0}-\approxfilter{0}][V],
$$
and, for $1\le k\le N$,
\begin{align} \label{eq:local_bias_split}
\mathcal B_k^{\rm bias}
&\eqdef \vnorm[ \filtertransform{k}(\filter{k-1}) - \filtertransform{k}[\theta](\filter{k-1})][V] \\
&\le \underbrace{ \vnorm[ \tau_{\pot{k-1}}(\filter{k-1}) \bigl(\mk{k}-\mk{k}[\theta,h]\bigr) ][V]
}_{\mathcal B_k^{\rm SGM}} +
\underbrace{ \vnorm[ \bigl( \tau_{\approxpot{k-1}}(\filter{k-1}) - \tau_{\pot{k-1}}(\filter{k-1}) \bigr) \mk{k}[\theta,h] ][V] }_{\mathcal B_k^{\rm pot}} \eqsp. \notag
\end{align}
Thus, $\mathcal B_k^{\rm bias}$ measures the discrepancy introduced locally at stage $k$ when the ideal and implemented updates are initialized from the same ideal law $\filter{k-1}$. Its two components separate the SGM approximation error $\mathcal B_k^{\rm SGM}$, due to score approximation and numerical discretization, from the guidance-potential mismatch $\mathcal B_k^{\rm pot}$.

The following elementary moment bound ensures that the normalization factors appearing in the subsequent propagation argument are finite.

\begin{lemma}[Moment control]
\label{lem:finite_moments_lambda_bias}
Suppose that Assumption \ref{ass:proposal_harris} and \ref{ass:bounded_normalized_potentials}
hold. Assume moreover that
$$
\approxfilter{0}[V]<\infty \eqsp,
\qquad
\mathcal B_0^{\rm bias} < \infty \eqsp,
\qquad
\mathcal B_k^{\rm bias}< \infty \eqsp,
\quad 1\le k\le N.
$$
Then, for every $0\le k\le N$,
$$
\filter{k}[V]<\infty \eqsp,
$$
and, for every $1\le k\le N$,
$$
\filtertransform{k}[\theta](\filter{k-1})[V] < \infty \eqsp.
$$
\end{lemma}

\begin{proof}
Since
$\mathcal B_0^{\rm bias}<\infty$ and
$\approxfilter{0}[V]<\infty$,
we have $\filter{0}[V]<\infty$. Assume that $\filter{k-1}[V]<\infty$. By positivity and boundedness of $\approxpot{k-1}$,
$$
\tau_{\approxpot{k-1}}(\filter{k-1})[V]
= \frac{ \filter{k-1}[\approxpot{k-1}V] }{ \filter{k-1}[\approxpot{k-1}] } < \infty \eqsp.
$$
Using the proposal drift condition,
$$
\filtertransform{k}[\theta](\filter{k-1})[V] \le \lambda_k^Q \tau_{\approxpot{k-1}}(\filter{k-1})[V]
+ K_k^Q < \infty \eqsp.
$$
Finally, since
$\filter{k}=\filtertransform{k}(\filter{k-1})$,
$$
\filter{k}[V] \le \filtertransform{k}[\theta](\filter{k-1})[V] + \mathcal B_k^{\rm bias} < \infty \eqsp,
$$

which closes the induction.
\end{proof}

\subsection{Propagation of local approximation errors}

We now propagate the local discrepancies introduced above through the remaining implemented Feynman--Kac dynamics. For $0\le \ell\le N$, define

$$
F_{\ell|N}^{\theta}
\eqdef
1+\twistedlyap{\ell}{N}[\theta].
$$

\begin{proposition}[Approximation bias decomposition]
\label{prop:filtering_error_decomposition}
For every $N\ge 2$,
\begin{align*}
\filter{N} - \approxfilter{N}
&= \left( \muBias{0}{N} - \muBias{0}{N}[\theta] \right)  \fsmk{1}{N}[\theta] \cdots \fsmk{N}{N}[\theta] \\
& + \sum_{\ell=1}^{N-1} \left( \muBias{\ell}{N} - \muBias{\ell}{N}[\theta] \right) \fsmk{\ell+1}{N}[\theta] \cdots \fsmk{N}{N}[\theta] + \left( \filtertransform{N} (\filter{N-1}) - \filtertransform{N}[\theta] (\filter{N-1})\right) \eqsp,
\end{align*}
where for $1\le \ell\le N-1$,
$$
\muBias{\ell}{N}
\eqdef \frac{ \normbackwardfun{\ell}{N}[\theta] \filtertransform{\ell} (\filter{\ell-1}) }{ \filtertransform{\ell}(\filter{\ell-1}) [ \normbackwardfun{\ell}{N}[\theta] ] } \eqsp,
\qquad
\muBias{\ell}{N}[\theta] \eqdef \frac{ \normbackwardfun{\ell}{N}[\theta] \filtertransform{\ell}[\theta] (\filter{\ell-1}) }{ \filtertransform{\ell}[\theta](\filter{\ell-1}) [ \normbackwardfun{\ell}{N}[\theta] ] } \eqsp,
$$
and
$$
\muBias{0}{N} \eqdef \frac{ \normbackwardfun{0}{N}[\theta]\filter{0} }{ \filter{0}[\normbackwardfun{0}{N}[\theta]] } \eqsp, 
\qquad
\muBias{0}{N}[\theta] \eqdef \frac{ \normbackwardfun{0}{N}[\theta]\approxfilter{0} }{ \approxfilter{0}[\normbackwardfun{0}{N}[\theta]] } \eqsp.
$$
Equivalently,
\begin{align}
\label{eq:kernel_bias_decomposition_approx_forgetting}
\filter{N}-\approxfilter{N} = \sum_{\ell=0}^{N} \left( \muBias{\ell}{N} - \muBias{\ell}{N}[\theta] \right) \prodfsmk{\ell}{N}[\theta] \eqsp,
\end{align}
where $\prodfsmk{N}{N}[\theta]=\operatorname{Id}$, and
$$
\muBias{N}{N} \eqdef \filtertransform{N}(\filter{N-1}) \eqsp,
\qquad
\muBias{N}{N}[\theta] \eqdef \filtertransform{N}[\theta](\filter{N-1}) \eqsp.
$$
\end{proposition}

\begin{proof}
We write
$$
\filter{N}= \left( \filtertransform{N} \circ \cdots \circ \filtertransform{1} \right)(\filter{0}),
\qquad
\approxfilter{N}= \left( \filtertransform{N}[\theta] \circ \cdots  \circ \filtertransform{1}[\theta]\right) (\approxfilter{0}) \eqsp.
$$
A standard telescoping argument yields for any $N \geq 2$,
\begin{align*}
\filter{N}-\approxfilter{N}
&=
\left( \filtertransform{N}[\theta] \circ \cdots  \circ \filtertransform{1}[\theta]\right) (\filter{0})-
\left( \filtertransform{N}[\theta] \circ \cdots  \circ \filtertransform{1}[\theta]\right) (\approxfilter{0})
\\
&\qquad
+
\sum_{\ell=1}^{N-1} \Bigg\{
\left( \filtertransform{N}[\theta] \circ \cdots \circ \filtertransform{\ell+1}[\theta] \right) \left( \filter{\ell} \right) 
-
\left( \filtertransform{N}[\theta] \circ \cdots \circ \filtertransform{\ell+1}[\theta] \right) \left( \filtertransform{\ell}[\theta] \left(  \filter{\ell-1} \right) \right) \Bigg\}
\\
&\qquad
+
\filtertransform{N}(\filter{N-1})
-
\filtertransform{N}[\theta](\filter{N-1}) \eqsp.
\end{align*}
Then, applying~\cref{prop:approx_tail_factorization} finishes the proof.
\end{proof}

\begin{lemma}[Kernel bias propagation under forward-smoothing forgetting]
\label{prop:kernel_bias_from_forward_forgetting}
Suppose that \ref{ass:proposal_harris} \ref{ass:bounded_normalized_potentials} hold,
that $\approxfilter{0}[V]<\infty$, and that
$\mathcal B_0^{\rm bias}<\infty$ and
$\mathcal B_k^{\rm SGM},\mathcal B_k^{\rm pot}<\infty$ for
$1\le k\le N$. For $1\le \ell\le N-1$, define
$$
\Lambda_{\ell,N} \eqdef \frac{ 1+\muBias{\ell}{N}[\theta]\!\left[ F_{\ell|N}^{\theta}\right] }{ \filtertransform{\ell} (\filter{\ell-1}) \left[ \normbackwardfun{\ell}{N}[\theta] \right] } \eqsp,
$$
and
$$
\Lambda_{0,N} \eqdef \frac{ 1+ \muBias{0}{N}[\theta]\!\left[F_{0|N}^{\theta}\right] }{\filter{0}\left[\normbackwardfun{0}{N}[\theta]\right] } \eqsp, \qquad \Lambda_{N,N}\eqdef 1 \eqsp.
$$
Then, for every $N\ge2$,
\begin{align}
\label{eq:kernel_bias_gamma_local_bound}
\vnorm[ \filter{N}-\approxfilter{N}][V]
&\le \Gamma_{0,N}\Lambda_{0,N}\vnorm[\filter{0}-\approxfilter{0}][V] + \sum_{\ell=1}^{N} \Gamma_{\ell,N}\Lambda_{\ell,N}
\vnorm[ \filtertransform{\ell}(\filter{\ell-1})- \filtertransform{\ell}[\theta](\filter{\ell-1})][V] \nonumber\\
&\le \Gamma_{0,N}\Lambda_{0,N} \mathcal B_0^{\rm bias} + \sum_{\ell=1}^{N} \Gamma_{\ell,N}\Lambda_{\ell,N}
\bigl(\mathcal B_\ell^{\rm SGM}+\mathcal B_\ell^{\rm pot}\bigr) \eqsp,
\end{align}
with for $0\le \ell\le N$, $\Gamma_{\ell,N}$ defined as in \Cref{cor:finite_horizon_forward_smoothing_stability} and $\Lambda_{\ell,N} < \infty$.
\end{lemma}

\begin{proof}
By \Cref{prop:filtering_error_decomposition},
$$
\filter{N}-\approxfilter{N}
= \sum_{\ell=0}^{N} \left(\muBias{\ell}{N} - \muBias{\ell}{N}[\theta] \right) \prodfsmk{\ell}{N}[\theta] \eqsp.
$$
Using the triangle inequality and \Cref{cor:finite_horizon_forward_smoothing_stability}, we get
$$
\vnorm[ \filter{N}-\approxfilter{N} ][V] \le \sum_{\ell=0}^{N}\Gamma_{\ell,N} \vnorm[ \muBias{\ell}{N}-\muBias{\ell}{N}[\theta] ][\twistedlyap{\ell}{N}[\theta]] \eqsp.
$$
Recall that for $\x \in \mathbb R^d$,
$$
F_{\ell|N}^{\theta} (\x) = 1+\twistedlyap{\ell}{N}[\theta] (\x)  = 1+ \frac{ V(\x) }{ \normbackwardfun{\ell}{N}[\theta](\x) } \eqsp, \qquad 0\le \ell\le N \eqsp.
$$
For $0\le \ell\le N-1$, we use the following estimate. Let $\mu,\nu$ be probability measures such that
$\mu[\normbackwardfun{\ell}{N}[\theta]] > 0$ and $\nu[\normbackwardfun{\ell}{N}[\theta]] > 0$. Then
\begin{align}
\label{eq:twisting_estimate_kernel_bias}
&\vnorm[ \frac{\normbackwardfun{\ell}{N}[\theta]\mu}{\mu[\normbackwardfun{\ell}{N}[\theta]]}
-
\frac{\normbackwardfun{\ell}{N}[\theta]\nu}{\nu[\normbackwardfun{\ell}{N}[\theta]]}][\twistedlyap{\ell}{N}[\theta]] 
\le \frac{ 1+\left( \frac{\normbackwardfun{\ell}{N}[\theta]\nu}{\nu[\normbackwardfun{\ell}{N}[\theta]]}\right)[F_{\ell|N}^{\theta}]}{ \mu[\normbackwardfun{\ell}{N}[\theta]]} \vnorm[\mu-\nu][V] \eqsp.
\end{align}
Indeed, for every measurable $\psi$ such that $|\psi|\le F_{\ell|N}^{\theta}$,
\begin{align*}
\left| \frac{\mu[\normbackwardfun{\ell}{N}[\theta]\psi]}{\mu[\normbackwardfun{\ell}{N}[\theta]]} - \frac{\nu[\normbackwardfun{\ell}{N}[\theta]\psi]}{\nu[\normbackwardfun{\ell}{N}[\theta]]} \right| 
& \le \frac{ |\mu-\nu|[\normbackwardfun{\ell}{N}[\theta]F_{\ell|N}^{\theta}]}{\mu[\normbackwardfun{\ell}{N}[\theta]]} + \frac{ \nu[\normbackwardfun{\ell}{N}[\theta]F_{\ell|N}^{\theta}] }{ \nu[\normbackwardfun{\ell}{N}[\theta]]} \frac{|\mu-\nu|[\normbackwardfun{\ell}{N}[\theta]] }{ \mu[\normbackwardfun{\ell}{N}[\theta]] } \\
& \le \frac{ 1+ \left( \frac{\normbackwardfun{\ell}{N}[\theta]\nu}{\nu[\normbackwardfun{\ell}{N}[\theta]]} \right) \left[ F_{\ell|N}^{\theta} \right] }{ \mu[\normbackwardfun{\ell}{N}[\theta]] } |\mu-\nu|[ \normbackwardfun{\ell}{N}[\theta]F_{\ell|N}^{\theta}] \eqsp,
\end{align*}
where we used in the last inequality that $F_{\ell|N}^{\theta}\ge \un$, and therefore $|\mu-\nu|[\normbackwardfun{\ell}{N}[\theta]] \le |\mu-\nu|[ \normbackwardfun{\ell}{N}[\theta]F_{\ell|N}^{\theta}]$. Also, since
$$
\normbackwardfun{\ell}{N}[\theta]F_{\ell|N}^{\theta} = \normbackwardfun{\ell}{N}[\theta]+V \le 1+V \eqsp,
$$
we obtain \eqref{eq:twisting_estimate_kernel_bias}. Applying \eqref{eq:twisting_estimate_kernel_bias} with
$$
\mu=\filter{0} \eqsp,
\qquad
\nu=\approxfilter{0} \eqsp,
$$
gives
$$
\vnorm[ \muBias{0}{N} - \muBias{0}{N}[\theta] ][\twistedlyap{0}{N}[\theta]] \le \Lambda_{0,N} \vnorm[\filter{0}-\approxfilter{0} ][V] \eqsp,
$$
The initial term satisfies
$$
\vnorm[ \filter{0}-\approxfilter{0} ][V] \le \mathcal B_0^{\rm bias} \eqsp.
$$
Similarly, for $1\le \ell\le N-1$, applying
\eqref{eq:twisting_estimate_kernel_bias} with
$$
\mu= \filtertransform{\ell}(\filter{\ell-1}) \eqsp,
\qquad
\nu= \filtertransform{\ell}[\theta](\filter{\ell-1}) \eqsp,
$$
gives
$$
\vnorm[ \muBias{\ell}{N} - \muBias{\ell}{N}[\theta] ][\twistedlyap{\ell}{N}[\theta]] \le \Lambda_{\ell,N} \vnorm[ \filtertransform{\ell}(\filter{\ell-1}) - \filtertransform{\ell}[\theta](\filter{\ell-1})
][V] \eqsp.
$$
For every $1\le k\le N$, by definition of the local discrepancy,
$$
\vnorm[ \filtertransform{k}(\filter{k-1}) - \filtertransform{k}[\theta](\filter{k-1})
][V] = \mathcal B_k^{\rm bias} \eqsp.
$$
Moreover, 
$$
\mathcal B_k^{\rm bias} \le \mathcal B_k^{\rm SGM} + \mathcal B_k^{\rm pot} \eqsp.
$$
Finally, for $\ell=N$,
$$
\normbackwardfun{N}{N}[\theta]=\un \eqsp, \qquad \prodfsmk{N}{N}[\theta]=\operatorname{Id} \eqsp,
\qquad \Lambda_{N,N}\eqdef1 \eqsp,
$$
and therefore
$$
\vnorm[ \muBias{N}{N} - \muBias{N}{N}[\theta] ][V] = \vnorm[ \filtertransform{N}(\filter{N-1}) - \filtertransform{N}[\theta](\filter{N-1}) ][V] \eqsp.
$$
It remains to prove the finiteness of $\Lambda_{\ell,N}$. Since $0<\normbackwardfun{\ell}{N}[\theta]\le1$, for any probability measure $\mu$ 
$$
0<\mu[\normbackwardfun{\ell}{N}[\theta]] \le 1 \eqsp,
$$
and
$$
\normbackwardfun{\ell}{N}[\theta]F_{\ell|N}^{\theta} = \normbackwardfun{\ell}{N}[\theta]+V \le 1+V \eqsp.
$$
Hence, for $\ell=0$, using \Cref{ass:bounded_normalized_potentials},
$$
\muBias{0}{N}[\theta][F_{0|N}^{\theta}] = \frac{ \approxfilter{0}[ \normbackwardfun{0}{N}[\theta]F_{0|N}^{\theta}]}{\approxfilter{0}[ \normbackwardfun{0}{N}[\theta] ] } \le \frac{ 1+\approxfilter{0}[V] }{\approxfilter{0}[\normbackwardfun{0}{N}[\theta]]} <\infty \eqsp.
$$
For $1\le \ell\le N-1$, the same argument holds with \Cref{lem:finite_moments_lambda_bias},
$$
\muBias{\ell}{N}[\theta][F_{\ell|N}^{\theta}] = \frac{ \filtertransform{\ell}[\theta](\filter{\ell-1}) [ \normbackwardfun{\ell}{N}[\theta]F_{\ell|N}^{\theta}] }{ \filtertransform{\ell}[\theta](\filter{\ell-1}) [\normbackwardfun{\ell}{N}[\theta] ] } \le \frac{ 1+\filtertransform{\ell}[\theta](\filter{\ell-1})[V] }{ \filtertransform{\ell}[\theta](\filter{\ell-1}) [\normbackwardfun{\ell}{N}[\theta] ] } < \infty \eqsp . 
$$
Finally, $\Lambda_{N,N}=1$, which concludes the proof.
\end{proof}

\subsection{Local error controls}

We now derive explicit bounds for the two approximation terms appearing in \eqref{eq:local_bias_split}. Applying such estimates to \Cref{prop:kernel_bias_from_forward_forgetting} concludes the proof for the approximation bias.

\begin{lemma}[Local SGM approximation error]
\label{lem:block_sgm_bias}
Fix $1\le k\le N$. Suppose Assumptions \ref{ass:dissipative_score_main},
\ref{ass:bounded_normalized_potentials}, and \ref{ass:sgm_fisher_regularity_main} hold. Define
$$
\varepsilon_k^{\rm score}(\theta,h)^2 \eqdef \sum_{n=(k-1)J}^{kJ-1} \Delta_n\, \mathbb E_k \!\left[ \left\| \nabla\log p_{T-s_n}(\Xola_{s_n}) - \scorenet[T-s_n][\Xola_{s_n}][\theta] \right\|^2 \right] \eqsp,
$$
where $\mathbb E_k$ denotes expectation under the exact reverse dynamics \eqref{eq:reverse_sde_main} initialized from $\tau_{\pot{k-1}}(\filter{k-1})$ at time $t_{k-1}$.
Then there exist $C_k^{\rm discr},C_k^{\rm net}<\infty$, independent of the numerical refinement, such that
$$
\mathcal B_k^{\rm SGM} \eqdef \vnorm[ \filter{k} - \tau_{\pot{k-1}}(\filter{k-1})\mk{k}[\theta,h] ][V] \le C_k^{\rm discr}\sqrt h + C_k^{\rm net}\varepsilon_k^{\rm score}(\theta,h) \eqsp.
$$
\end{lemma}

\begin{proof}
To reuse the local Girsanov--Hellinger argument of
\citet[Section~H.2]{strasman2026forgetting}, we need to verify
its moment and score-regularity requirements under this common distribution $\tau_{\pot{k-1}}(\filter{k-1})$ as 
$$
\mathcal B_k^{\rm SGM} = \vnorm[ \tau_{\pot{k-1}}(\filter{k-1})\mk{k}[] - \tau_{\pot{k-1}}(\filter{k-1})\mk{k}[\theta,h] ][V]  \eqsp.
$$
In the SMC setting, these estimates are therefore required along a
guided exact flow rather than under the unconditional diffusion
marginals. We first transfer the corresponding moment and
score-regularity estimates from the unconditional marginals
$p_{T-u}$ of \eqref{eq:reverse_sde_main} to this SMC exact flow.

\paragraph{Control of the $V$-weighted moments.} Since $\filter{0}=p_T$ and the ideal potentials are bounded by \Cref{ass:bounded_normalized_potentials}, induction
along the ideal Feynman--Kac flow gives
$$
\tau_{\pot{k-1}}(\filter{k-1}) \le \prod_{j=0}^{k-1}
\frac{\|\pot{j}\|_\infty}{\filter{j}[\pot{j}]} p_{T-t_{k-1}} \eqsp,
$$
Propagating both sides with the same exact reverse dynamics therefore yields, for $u\in[t_{k-1},t_k]$ \eqsp,
$$
\mathcal L_k^{\rm ex}(\Xola_u) = \tau_{\pot{k-1}}(\filter{k-1})\mk{t_{k-1},u} \le D_k\,p_{T-t_{k-1}}\mk{t_{k-1},u} = D_k\,p_{T-u} \eqsp,
$$
where by assumption
$$
D_k \eqdef \prod_{j=0}^{k-1} \frac{\|\pot{j}\|_\infty}{\filter{j}[\pot{j}]} < \infty \eqsp.
$$
Hence any integrability estimate available under the unconditional marginal $p_{T-u}$ transfers to the exact reverse dynamics initialized from $\tau_{\pot{k-1}}(\filter{k-1})$. In particular, $\pidata$ has finite fourth order moment and so
$$
\sup_{u\in[t_{k-1},t_k]} \mathbb E_k^{\rm ex}\!\left[(1+V(\Xola_u))^2\right] < \infty \eqsp.
$$
The linear-growth condition in \Cref{ass:dissipative_score_main} implies
$$
\sup_{u\in[t_{k-1},t_k]} \mathbb E\!\left[ (1+V(\Xbar_u^\theta))^2 \right] < \infty,
$$
uniformly over the numerical refinement, where $(\Xbar_u^\theta)_{u\in[t_{k-1},t_k]}$ denotes the continuous-time Euler interpolation initialized from $\tau_{\pot{k-1}}(\filter{k-1})$.

\paragraph{Hessian integrability condition.}

The freezing argument of
\citet[Section~H.2]{strasman2026forgetting} further requires the
Jacobian of the true score to be square-integrable along the exact
backward flow. By the same
domination,
\begin{align}
&\int_{t_{k-1}}^{t_k}
\bwdnoisesch{u}\,
\mathbb E_k^{\rm ex}\!\left[
\|\nabla^2\log p_{T-u}(\Xola_u)\|_{\rm F}^2
\right]\rmd u
\nonumber\\
&\qquad\le
D_k
\int_{t_{k-1}}^{t_k}
\bwdnoisesch{u}
\int_{\rset^\xdim}
\|\nabla^2\log p_{T-u}(\x)\|_{\rm F}^2
p_{T-u}(\x)\,\rmd\x\,\rmd u
\nonumber\\
&\qquad=
D_k
\int_{T-t_k}^{T-t_{k-1}}
\beta(t)
\int_{\rset^\xdim}
\|\nabla^2\log p_t(\x)\|_{\rm F}^2
p_t(\x)\,\rmd\x\,\rmd t
<\infty.
\label{eq:selected_hessian_energy}
\end{align}
The finiteness of \eqref{eq:selected_hessian_energy} is a standard consequence of the Fisher-information evolution identity along the VE/VP forward flow \citep{bakry2014analysis}.

\paragraph{Local error estimate}

We can therefore repeat, on each numerical subinterval, the
Girsanov--freezing argument of \citet[Section~H.2]{strasman2026forgetting}. 
$$
\operatorname{KL}\!\left(
\filter{k}
\,\middle\|\,
\tau_{\pot{k-1}}(\filter{k-1})\mk{k}[\theta,h] \right) \lesssim h +  \varepsilon_k^{\rm score}(\theta,h)^2 \eqsp,
$$
Finally, the weighted Hellinger argument of \citet[Section~H.2]{strasman2026forgetting}, together with the uniform fourth-moment bounds above, gives
$$
\mathcal B_k^{\rm SGM} \le C_k^{\rm discr}\sqrt h + C_k^{\rm net}\varepsilon_k^{\rm score}(\theta,h) \eqsp,
$$
with $C_k^{\rm discr},C_k^{\rm net}<\infty$ independent of the numerical refinement. This proves the claim.
\end{proof}

\begin{proposition}[Potential mismatch]
\label{prop:potential_mismatch_control}
Suppose that Assumptions \ref{ass:bounded_normalized_potentials} and \ref{ass:sgm_fisher_regularity_main} hold, and that the implemented proposals satisfy the drift bound in \ref{ass:proposal_harris}. Consider the ideal flow initialized from $\filter{0}=p_T$.
For $1\le k\le N$, define
\begin{equation}
\label{eq:normalized_potential_error}
\bigl(\varepsilon_k^{\rm pot}\bigr)^2
\eqdef
\filter{k-1}\!\left[ \left( \frac{\approxpot{k-1}}{\filter{k-1}[\approxpot{k-1}]} - \frac{\pot{k-1}} {\filter{k-1}[\pot{k-1}]} \right)^2 \right] \eqsp.
\end{equation}

Then,
\begin{align}
\mathcal B_k^{\rm pot}
&\eqdef \vnorm[ \left( \tau_{\approxpot{k-1}}(\filter{k-1}) - \tau_{\pot{k-1}}(\filter{k-1}) \right) \mk{k}[\theta,h]][V] \nonumber\\
&\le \filter{k-1}\!\left[ \left| \frac{\approxpot{k-1}} {\filter{k-1}[\approxpot{k-1}]} - \frac{\pot{k-1}}{\filter{k-1}[\pot{k-1}]} \right| (1+K_k^Q+\lambda_k^Q V) \right] \label{eq:potential_bias_weighted_bound}\\
&\le \varepsilon_k^{\rm pot}   \left\{ \filter{k-1}\!\left[ (1+K_k^Q+\lambda_k^Q V)^2 \right] \right\}^{1/2}\eqsp, \label{eq:potential_bias_l2_bound}
\end{align}
where in the main text we used
$$
C_k^{\rm pot} \eqdef \left\{ \filter{k-1}\!\left[ (1+K_k^Q+\lambda_k^Q V)^2 \right] \right\}^{1/2} \eqsp.
$$
\end{proposition}

\begin{proof}
We first establish finiteness of the quantities in the statement.
Strict positivity and boundedness of the potentials imply
$$
0<\filter{j}[\pot{j}]<\infty \eqsp, \qquad 0 < \filter{j}[\approxpot{j}]<\infty \eqsp.
$$
Thus the two normalized potentials in \eqref{eq:normalized_potential_error} are bounded, and $\varepsilon_k^{\rm pot} < \infty$. Since $\filter{0}=p_T$, induction using bounded ideal selection
and the identity
$p_{T-t_j}\mk{j+1}=p_{T-t_{j+1}}$ gives
$$
\filter{j} \le \left( \prod_{i=0}^{j-1} \frac{\|\pot{i}\|_\infty}{\filter{i}[\pot{i}]} \right)p_{T-t_j} \eqsp, \qquad 0\le j\le N \eqsp.
$$
The Gaussian representation of the forward diffusion together with the fact that $\pidata$ has finite fourth moment imply $\int_{\rset^\xdim} (1+V(\x))^2\,p_t(\x)\,\rmd \x < \infty$ for every $t\in[0,T]$.
Consequently, $\filter{j}[(1+V)^2]<\infty$ and right-hand side of \eqref{eq:potential_bias_l2_bound} is finite.

By definition of the weighted total-variation norm and using
$|\mu Q|\le |\mu|Q$ for Markov kernels,
\begin{align*}
\mathcal B_k^{\rm pot}
&= \int_{\rset^\xdim} (1+V(\mathbf{u})) \left| \left( \tau_{\approxpot{k-1}}(\filter{k-1}) - \tau_{\pot{k-1}}(\filter{k-1}) \right) \mk{k} [\theta,h] \right|(\rmd \mathbf{u}) \\
&\le \int_{\rset^\xdim} \left| \frac{\approxpot{k-1}(\x)}{\filter{k-1}[\approxpot{k-1}]} - \frac{\pot{k-1}(\x)}{\filter{k-1}[\pot{k-1}]} \right|
\mk{k}[\theta,h](1+V)(\x)\, \filter{k-1}(\rmd\x) \\
&= \filter{k-1}\!\left[ \left| \frac{\approxpot{k-1}}{\filter{k-1}[\approxpot{k-1}]} - \frac{\pot{k-1}}{\filter{k-1}[\pot{k-1}]} \right| \mk{k}[\theta,h](1+V) \right] \eqsp.
\end{align*}
The proposal drift bound gives
$$
\mk{k}[\theta,h](1+V) \le 1 + K_k^Q + \lambda_k^Q V \eqsp,
$$
which proves \eqref{eq:potential_bias_weighted_bound}. Cauchy--Schwarz under $\filter{k-1}$ then yields \eqref{eq:potential_bias_l2_bound}.
\end{proof}

\begin{remark}[Interpretation]
The quantity $\varepsilon_k^{\rm pot}$ measures the mean-square discrepancy between the normalized potentials under the common
ideal input law $\filter{k-1}$. It is invariant under separate positive rescaling of either potential and vanishes when they are proportional
$\filter{k-1}$-almost surely. The bound only penalizes potential mismatch in regions that are
relevant under the current ideal filtering distribution, with errors further weighted by the Lyapunov cost induced by the next
mutation step. The data assumption is used here only through its fourth-moment condition.
\end{remark}

\section{Numerical illustration}
\label{app:diffusion_exp_SMC}

\subsection{Gaussian-mixture target and closed form diffusion model}

We consider a synthetic benchmark for which all ideal quantities are available in closed form. This allows us to measure separately the finite-particle error and the deterministic proposal bias appearing in \Cref{thm:total_smc_kernel_bias_error}.

\paragraph{Data distribution.}
Let $\pidata$ be a Gaussian mixture on $\rset^\xdim$,
\begin{equation}
\label{eq:num_gmm_prior}
\pidata(\rmd \x) = \sum_{j=1}^{K} w_j \gaussiand{\mu_j}{\Sigma_j}[\x] \rmd \x  \eqsp,
\qquad  w_j>0 \eqsp ,\qquad \sum_{j=1}^{K} w_j=1  \eqsp.
\end{equation}

\paragraph{Forward diffusion.}
We use the variance-preserving Ornstein--Uhlenbeck forward diffusion
\begin{equation*}
\rmd \Xora_t = -\beta \Xora_t\,\rmd t + \sqrt{2\beta} \rmd B_t \eqsp, \qquad \Xora_0\sim\pidata \eqsp,
\end{equation*}
with constant $\beta>0$. This corresponds to the setting
$\isvp=1$ and $\noisesch{t}=\beta$ in \eqref{eq:forward_sde_main}.
For $0\le s\le t\le T$, write
$$
m_{s,t}\eqdef \rme^{-\beta(t-s)} \eqsp,
\qquad
\sigma_{s,t}^2 \eqdef 1-m_{s,t}^2 \eqsp.
$$
Hence
\begin{equation} \label{eq:forward_marginal_view}
    \Xora_t = m_{s,t}\Xora_s + \sigma_{s,t} Z \eqsp,
    \qquad
    Z \sim \gaussiand{0}{\Id_d} \eqsp, \qquad Z \perp \Xora_s \eqsp.
\end{equation}
In particular, under mixture component $j$, that is conditionally on $J=j$, 
\begin{equation} \label{eq:Xora_knowing_J}
 \Xora_t \mid J=j \sim \gaussiand{\mu_{j,t}}{\Sigma_{j,t}} \eqsp,
\end{equation}
where 
$$
\mu_{j,t} \eqdef m_{0,t}\mu_j \eqsp, \qquad \Sigma_{j,t} \eqdef m_{0,t}^2\Sigma_j+\sigma^2_{0,t} \Id_d \eqsp.
$$
Therefore the forward marginal is again a Gaussian mixture: 
\begin{equation*} 
p_t(\x) = \sum_{j=1}^{K} w_j \gaussiand{\mu_{j,t}}{\Sigma_{j,t}}[\x] \eqsp. 
\end{equation*} 

\paragraph{Exact score.} The exact score of the noised GMM is available analytically as
\begin{equation*} 
\nabla \log p_t(\x) = \sum_{j=1}^{K} r_j(t,\x) \left[ -\Sigma_{j,t}^{-1}(\x-\mu_{j,t}) \right] \eqsp,
\end{equation*}
with
\begin{equation*} 
r_j(t,\x) \eqdef \frac{ w_j\gaussiand{\mu_{j,t}}{\Sigma_{j,t}}[\x] }{ \sum_{\ell=1}^{K} w_\ell\gaussiand{\mu_{\ell,t}}{\Sigma_{\ell,t}}[\x] }  \eqsp.
\end{equation*}

\paragraph{Exact reverse transition kernels.}
Let $0=t_0<t_1<\cdots<t_N=T$ be the reverse-time grid. The reverse process
$\Xola_{t_k}$ has the same distribution as $\Xora_{T-t_k}$ and therefore,
$$
\mk{k+1}(\x,\rmd \z) = \mathcal L \left( \Xola_{t_{k+1}}\in\rmd \z \middle| \Xola_{t_k}=\x \right) = \mathcal L \left( \Xora_{T-t_{k+1}}\in\rmd \z \middle| \Xora_{T-t_k}=\x \right) \eqsp.
$$
Set $u_k\eqdef T-t_k$, $u_{k+1}\eqdef T-t_{k+1}$, so that $u_{k+1}<u_k$ and
$$
\mk{k+1}(\x,\rmd \z)
= \mathcal L \left( \Xora_{u_{k+1}}\in\rmd \z \middle| \Xora_{u_k}=\x \right) \eqsp.
$$
The pair $\left( \Xora_{u_{k+1}}, \Xora_{u_k} \right) | J=j $ is jointly Gaussian with mean
$$
\begin{pmatrix}
\mu_{j,u_{k+1}} \\
\mu_{j,u_k}
\end{pmatrix} \eqsp,
$$
and covariance matrix using \eqref{eq:forward_marginal_view}
$$
\begin{pmatrix}
\Sigma_{j,u_{k+1}} & m_{u_{k+1},u_k}\Sigma_{j,u_{k+1}} \\
m_{u_{k+1},u_k}\Sigma_{j,u_{k+1}}  & \Sigma_{j,u_k}
\end{pmatrix} \eqsp.
$$
By the Gaussian conditioning formula, for every $\x\in\rset^\xdim$,
$$
\Xora_{u_{k+1}} \mid \left\{ \Xora_{u_k}=\x,\ J=j \right\} \sim \gaussiand{ \mu_{j,u_{k+1}|u_k}(\x) }{ \Sigma_{j,u_{k+1}|u_k} } \eqsp,
$$
where
\begin{align*}
\mu_{j,u_{k+1}|u_k}(\x)
&\eqdef \mu_{j,u_{k+1}} + m_{u_{k+1},u_k} \Sigma_{j,u_{k+1}} \Sigma_{j,u_k}^{-1} (\x-\mu_{j,u_k}) \eqsp, \\
\Sigma_{j,u_{k+1}|u_k}
&\eqdef \Sigma_{j,u_{k+1}} - m_{u_{k+1},u_k}^{2} \Sigma_{j,u_{k+1}} \Sigma_{j,u_k}^{-1} \Sigma_{j,u_{k+1}} \eqsp .
\end{align*}
By Bayes' formula, the probability that $J=j$ given $\Xora_{u_k}=\x$ is
$$
r_j(u_k,\x)
=\frac{  w_j\gaussiand{\mu_{j,u_k}}{\Sigma_{j,u_k}}[\x] }{ \sum_{\ell=1}^{K} w_\ell \gaussiand{\mu_{\ell,u_k}}{\Sigma_{\ell,u_k}}[\x] } \eqsp .
$$
As a consequence
\begin{equation} \label{eq:num_exact_backward_kernel}
\mk{k+1}(\x,\rmd \z)  =  \sum_{j=1}^{K}  r_j(u_k,\x)  \gaussiand{ \mu_{j,u_{k+1}|u_k}(\x) }{ \Sigma_{j,u_{k+1}|u_k} }[\z] \rmd \z \eqsp.
\end{equation}
Sampling with these exact kernels introduces no time-discretization error. In
the unconditional case, initializing from $p_T$ and applying
$\mk{1}\cdots\mk{N}$ recovers $\pidata$ up to Monte Carlo error, consistently
with the ideal diffusion construction in \Cref{sec:diffusion_smc}.

\paragraph{Approximate proposals: separating discretization and score errors.}
The closed-form kernel $\mk{k+1}$ in \eqref{eq:num_exact_backward_kernel} will be used as the ideal reverse kernel. In order to disentangle time-discretization error from score approximation error, we introduce two time-changed Euler--Maruyama proposal kernels. First, define the oracle Euler proposal $\eulermk{k+1}$, obtained by applying
one Euler step to the reverse SDE using the exact GMM score:
\begin{equation}
\label{eq:num_oracle_euler_proposal}
\eulermk{k+1}(\x,\rmd \z) = \gaussiand{ \x+\Delta_k\left[ \isvp \x + 2\nabla\log p_{u_k}(\x) \right]
}{ 2\Delta_k \Id_d }[\z]\rmd \z \eqsp, \qquad \Delta_k\eqdef \beta(t_{k+1}-t_k) \eqsp.
\end{equation}
This kernel is biased only because it discretizes the reverse dynamics. Second, define the perturbed-score Euler proposal
$\eulermk{k+1}[\theta]$ by replacing the exact score with
$$
\scorenet[u_k][\x][\theta] = \nabla\log p_{u_k}(\x) + \varepsilon_{\rm net}\lambda(\x) \eqsp,
$$
where $\lambda:\rset^\xdim\to\rset^\xdim$ is a fixed bounded vector
field, while $\varepsilon_{\rm net}$ controls the bias magnitude.  This perturbation is used as a controlled proxy for a learned-score
error.  The corresponding proposal is
\begin{equation}
\label{eq:num_learned_euler_proposal}
\eulermk{k+1}[\theta](\x,\rmd \z) = \gaussiand{ \x+\Delta_k\left[ \isvp \x + 2 \scorenet[u_k][\x][\theta] \right] }{ 2 \Delta_k \Id_d }[\z]\rmd \z \eqsp .
\end{equation}
The difference between $\mk{k+1}$ and $\eulermk{k+1}$ isolates the
\emph{time-discretization bias}, whereas the difference between $\eulermk{k+1}$ and
$\eulermk{k+1}[\theta]$ isolates \emph{the score approximation bias}.
In the experiments, the approximate mutation kernel used by SMC is
$\eulermk{k+1}[\theta]$.

In the reported experiments in Section \ref{sec:num_2d_gmm_experiment} we use bounded $\tanh$ perturbations.  In the
two-dimensional experiment, the globally biased Euler curve uses
$$
\lambda_{\rm bias}(\x)=(0,\tanh(x_2)) \eqsp,
$$
whereas the one-step forgetting diagnostic uses
$$
\lambda_{\rm loc}(\x)=(\tanh(x_1),\tanh(x_2)) \eqsp .
$$
In the 50-dimensional experiment in Section \ref{sec:num_50d_gmm_experiment}, the same perturbations are applied only in the visible coordinates:
$$
\lambda_{\rm bias}(\x)=(0,\tanh(x_2),0,\ldots,0) \eqsp,
\qquad
\lambda_{\rm loc}(\x)=(\tanh(x_1),\tanh(x_2),0,\ldots,0) \eqsp.
$$

\subsection{Conditional sampling model: noisy inpainting.}

\paragraph{Observation model.}
Let $\dobs\ge 1$ be the observation dimension. We consider a linear Gaussian
observation of the clean variable,
\begin{equation}
\label{eq:num_observation_model}
\Xobs = H \Xora_0 + \varepsilon \eqsp,
\qquad H\in\rset^{\dobs\times \xdim} \eqsp,
\qquad \varepsilon\sim\gaussiand{0}{R} \eqsp,
\qquad \varepsilon\perp \Xora_0 \eqsp,
\end{equation}
where $R$ is symmetric positive definite with compatible dimensions. We condition on the observed value $\Xobs=\xobs\in\rset^{\dobs}$. The target
distribution is
$$
\filter{N}(\rmd \x) = \mathcal L(\Xora_0\in\rmd \x\mid \Xobs=\xobs) \eqsp.
$$
Since the prior is a Gaussian mixture and the likelihood is linear Gaussian,
this target is again a Gaussian mixture and is available in closed form.

\begin{lemma}[Closed-form posterior for a GMM under linear Gaussian observation]
\label{lem:num_gmm_linear_gaussian_posterior}
Let $\pi_{\rm data}$ be defined as in \eqref{eq:num_gmm_prior} and $\Xobs$ as in \eqref{eq:num_observation_model}. Then, for every observation $\xobs\in\rset^{\dobs}$, the conditional law
$
\mathcal L(\Xora_0\in\rmd\x\mid \Xobs=\xobs)
$
is the Gaussian mixture
\begin{equation}
\label{eq:num_exact_posterior_gmm}
\filter{N}(\rmd \x)
= \sum_{j=1}^{K} \widetilde w_j \gaussiand{\widetilde\mu_j}{\widetilde\Sigma_j}[\x] \rmd \x \eqsp,
\end{equation}
where  $\widetilde\mu_j$, $\widetilde\Sigma_j$ and $\widetilde w_j$ are defined in the proof.
\end{lemma}

\begin{proof}
Conditionally on $J=j$, the observation model $\Xobs  = H \Xora_0+\varepsilon$ is linear. Hence the joint density of $(\Xora_0,\Xobs,J)$ satisfies
\begin{equation} \label{eq:joint_p(x,y,J)}
p(\x,\xobs,J=j) 
= w_j \gaussiand{\mu_j}{\Sigma_j}[\x] \gaussiand{H\x}{R}[\xobs] \eqsp.
\end{equation}
For each $j$, completing the square in $\x$ yields
\begin{equation}
\gaussiand{\mu_j}{\Sigma_j}[\x] \gaussiand{H\x}{R}[\xobs]
\propto_{\x} \gaussiand{\widetilde\mu_j}{\widetilde\Sigma_j}[\x] \eqsp,
\end{equation}
with
$$
\widetilde\mu_j = \widetilde\Sigma_j \left(  \Sigma_j^{-1}\mu_j+H^\top R^{-1}\xobs \right) \eqsp, \qquad 
\widetilde\Sigma_j = \left(\Sigma_j^{-1}+H^\top R^{-1}H\right)^{-1} \eqsp.
$$
So we get
$$
\mathcal L(\Xora_0\in\rmd\x\mid \Xobs=\xobs,J=j)
= \gaussiand{\widetilde\mu_j}{\widetilde\Sigma_j}[\x]\rmd\x \eqsp.
$$
Moreover, since 
\begin{equation} \label{eq:Y_knowing_J}
\Xobs\mid J=j \sim \gaussiand{H\mu_j}{S_j}, \qquad S_j\eqdef H\Sigma_jH^\top+R \eqsp,
\end{equation}
we obtain 
\begin{equation} \label{eq:gaussian_product_identity_posterior} \gaussiand{\mu_j}{\Sigma_j}[\x] \gaussiand{H\x}{R}[\xobs] 
= \gaussiand{H\mu_j}{S_j}[\xobs] \gaussiand{\widetilde\mu_j}{\widetilde\Sigma_j}[\x] \eqsp. 
\end{equation} 
Integrating \eqref{eq:joint_p(x,y,J)} with respect to $\x$, we get
$$
p(\xobs,J=j) = w_j\gaussiand{H\mu_j}{S_j}[\xobs] \eqsp.
$$
Consequently, $p(\xobs)  = \sum_{\ell=1}^{K}w_\ell\gaussiand{H\mu_\ell}{S_\ell}[\xobs]$. By Bayes' formula,
$$
\widetilde w_j \eqdef \mathbb P(J=j\mid \Xobs=\xobs) = \frac{ w_j\gaussiand{H\mu_j}{S_j}[\xobs] }{ \sum_{\ell=1}^{K} w_\ell \gaussiand{H\mu_\ell}{S_\ell}[\xobs] } \eqsp.
$$
Hence, we obtain
$$
\mathcal L(\Xora_0\in\rmd\x\mid \Xobs=\xobs)
= \sum_{j=1}^{K} \widetilde w_j \gaussiand{\widetilde\mu_j}{\widetilde\Sigma_j}[\x]\rmd\x \eqsp.
$$
\end{proof}

\paragraph{Bridge likelihoods.}
For $0\le k\le N$, define the likelihood of the observation given the current reverse-time state by
\begin{equation}
\label{eq:num_bridge_likelihood}
h_k(\x) \eqdef p(\xobs\mid \Xola_{t_k}=\x) = p(\xobs\mid \Xora_{T-t_k}=\x).
\end{equation}
Set $u_k\eqdef T-t_k$, conditionally on $J=j$, the pair $\left(\Xobs,\Xora_{u_k}\right)$ is jointly Gaussian, using \eqref{eq:Xora_knowing_J} and \eqref{eq:Y_knowing_J}, we have
$$
\Xobs\mid J=j \sim \gaussiand{H\mu_j}{S_j} \eqsp, \qquad S_j\eqdef H\Sigma_jH^\top+R \eqsp,
$$
and
$$
\Xora_{u_k}\mid J=j \sim \gaussiand{\mu_{j,u_k}}{\Sigma_{j,u_k}} \eqsp, 
\qquad \mu_{j,u_k} = m_{0,u_k}\mu_j \eqsp,
\qquad \Sigma_{j,u_k} = m_{0,u_k}^2\Sigma_j+\sigma^2_{0,u_k}\Id_d \eqsp.
$$
Hence,
$$
\operatorname{Cov} \left( \Xobs,\Xora_{u_k}  \middle| J=j \right)  = m_{0,u_k}H\Sigma_j \eqsp .
$$
Therefore, by the Gaussian conditioning formula,
$$
\Xobs \mid \left\{ \Xora_{u_k}=\x,\ J=j \right\} \sim \gaussiand{\mu^y_{j,k}(\x)}{S^y_{j,k}},
$$
with
\begin{align*}
\mu^y_{j,k}(\x) &\eqdef H\mu_j + m_{0,u_k} H \Sigma_j \Sigma_{j,u_k}^{-1} (\x-\mu_{j,u_k}) \eqsp, \\
S^y_{j,k} &\eqdef S_j - m_{0,u_k}^2 H \Sigma_j \Sigma_{j,u_k}^{-1} \Sigma_jH^\top \eqsp.
\end{align*}
The posterior probability of component $j$ given the current state $\Xora_{u_k}=\x$ is
$$
r_j(u_k,\x) = \frac{ w_j\gaussiand{\mu_{j,u_k}}{\Sigma_{j,u_k}}[\x] }{ \sum_{\ell=1}^{K} w_\ell \gaussiand{\mu_{\ell,u_k}}{\Sigma_{\ell,u_k}}[\x] } \eqsp.
$$
Hence,
\begin{equation} \label{eq:num_hk_formula}
h_k(\x)  = \sum_{j=1}^{K} r_j(u_k,\x) \gaussiand{\mu^y_{j,k}(\x)}{S^y_{j,k}}[\xobs] \eqsp.
\end{equation}
At the terminal reverse time $k=N$, we have $u_N=0$. Hence
$$
\mu_{j,u_N}=\mu_j \eqsp, \qquad \Sigma_{j,u_N}=\Sigma_j \eqsp,
$$
and the conditional observation law reduces to
$$
\Xobs\mid \{\Xora_0=\x,J=j\} \sim \gaussiand{H\x}{R} \eqsp.
$$
\paragraph{Bridge weights and numerical implementation.}
The abstract Feynman--Kac model in \Cref{sec:fk_smc_general} is written in selection-before-propagation form: at time $k$, particles are first selected according to a state potential $\pot{k}(\x)$, and are then propagated with the
mutation kernel $\mk{k+1}$. In this controlled numerical conditional-sampling experiment, the observation is more naturally introduced through a bridge weight depending on two consecutive states. We introduce,
\begin{equation} \label{eq:num_bridge_potential}
\pot{k}^{\rm br}(\x,\z) \eqdef \frac{h_{k+1}(\z)}{h_k(\x)} \eqsp,  \qquad k=0,\ldots,N-1 \eqsp,
\end{equation}
The reason for using the ratio \eqref{eq:num_bridge_potential} is the telescoping identity
$$
h_0(\Xola_{t_0}) \prod_{k=0}^{N-1} \pot{k}^{\rm br}(\Xola_{t_k},\Xola_{t_{k+1}}) = h_N(\Xola_{t_N}) \eqsp.
$$
Hence, for every bounded measurable test function $\psi$,
$$
\frac{ \E\left[ \psi(\Xola_{t_N}) h_0(\Xola_{t_0}) \prod_{k=0}^{N-1} \pot{k}^{\rm br}(\Xola_{t_k},\Xola_{t_{k+1}}) \right] }{ \E\left[ h_0(\Xola_{t_0}) \prod_{k=0}^{N-1} \pot{k}^{\rm br}(\Xola_{t_k},\Xola_{t_{k+1}}) \right]}
= \frac{ \E\left[ \psi(\Xola_{t_N})h_N(\Xola_{t_N}) \right] }{ \E\left[ h_N(\Xola_{t_N}) \right] } \eqsp.
$$
The right-hand side is precisely
$$
\E[\psi(\Xola_{t_N})\mid \Xobs=\xobs] \eqsp.
$$
Strictly speaking, \eqref{eq:num_bridge_potential} depends on both endpoints of the transition, whereas the abstract presentation uses state potentials $\pot{k}(\x)$. This is only a notational issue: the bridge formulation can be written as a standard Feynman--Kac model on the augmented state $(\Xola_{t_k},\Xola_{t_{k+1}})$ as we will show in the next paragraph.
\paragraph{Relation with the selection-before-propagation form.}
For some generic proposal kernel $Q_{k+1}$, define
$$
g_k^{\star}(\x) \eqdef \int Q_{k+1}(\x,\rmd \z) \pot{k}^{\rm br}(\x,\z)
= \frac{  Q_{k+1}h_{k+1}(\x) }{ h_k(\x) } \eqsp,
$$
and, whenever $g_k^\star(\x)>0$,
$$
Q_{k+1}^\star(\x,\rmd \z) \eqdef \frac{ Q_{k+1}(\x,\rmd \z) \pot{k}^{\rm br}(\x,\z)}{ g_k^\star(\x)} \eqsp.
$$
Then $Q^\star_{k+1}$ is a Markov kernel and
$$
Q_{k+1}(\x,\rmd \z) \pot{k}^{\rm br}(\x,\z)
= g^{\star}(\x)Q_{k+1}^\star(\x,\rmd \z) \eqsp.
$$
Hence the transition-potential update defined by \eqref{eq:num_bridge_potential} can be written in the same selection-before-propagation form as \Cref{sec:fk_smc_general}: first select with the state potential $\bar g_k^Q(\x)$, then propagate with the twisted Markov kernel $\bar Q_{k+1}$.




\subsection{Two-dimensional Gaussian mixture}
\label{sec:num_2d_gmm_experiment}

The prior target is the Gaussian mixture
$$
    p_0(\x)
    =
    \sum_{j=1}^{4} w_j
    \gaussiand{\mu_j}{\Sigma_j}[\x],
    \qquad \x\in\rset^2,
$$
with weights $(w_1,w_2,w_3,w_4) = (0.05,0.15,0.30,0.50)$, means
$$
\mu_1 = (-5,-4)^\top \eqsp, \qquad \mu_2 = (-4,4.5)^\top \eqsp, \qquad \mu_3 = (4,-3.5)^\top \eqsp, \qquad \mu_4 = (5,4)^\top \eqsp,
$$
and covariance matrices
\begin{align*}
& \Sigma_1 =
\begin{pmatrix}
0.20 & 0.12\\
0.12 & 0.50
\end{pmatrix} \eqsp, \qquad
\Sigma_2 =
\begin{pmatrix}
0.80 & -0.35\\
-0.35 & 0.30
\end{pmatrix} \eqsp, \\
& \Sigma_3 =
\begin{pmatrix}
0.35 & 0.20\\
0.20 & 0.90
\end{pmatrix} \eqsp, \qquad
\Sigma_4 =
\begin{pmatrix}
0.50 & -0.25\\
-0.25 & 0.40
\end{pmatrix} \eqsp.
\end{align*}
The components are deliberately separated and anisotropic, so that conditioning on a one-dimensional observation produces a posterior with two visible modes. The observation model is scalar and acts only on the first coordinate:
$$
\Xobs = H \Xora_0 + \varepsilon \eqsp, \qquad H = (1,0) \eqsp, \qquad \varepsilon \sim \gaussiand{0}{R} \eqsp, \qquad R = 0.16 \eqsp.
$$
In the reported experiment we condition on the realization $\xobs = 3.8$. This observation is chosen so that the two left components receive negligible posterior mass, while the two right components remain visible with unequal weights.

\begin{figure}[t]
    \centering
    \includegraphics[width=.82\linewidth]{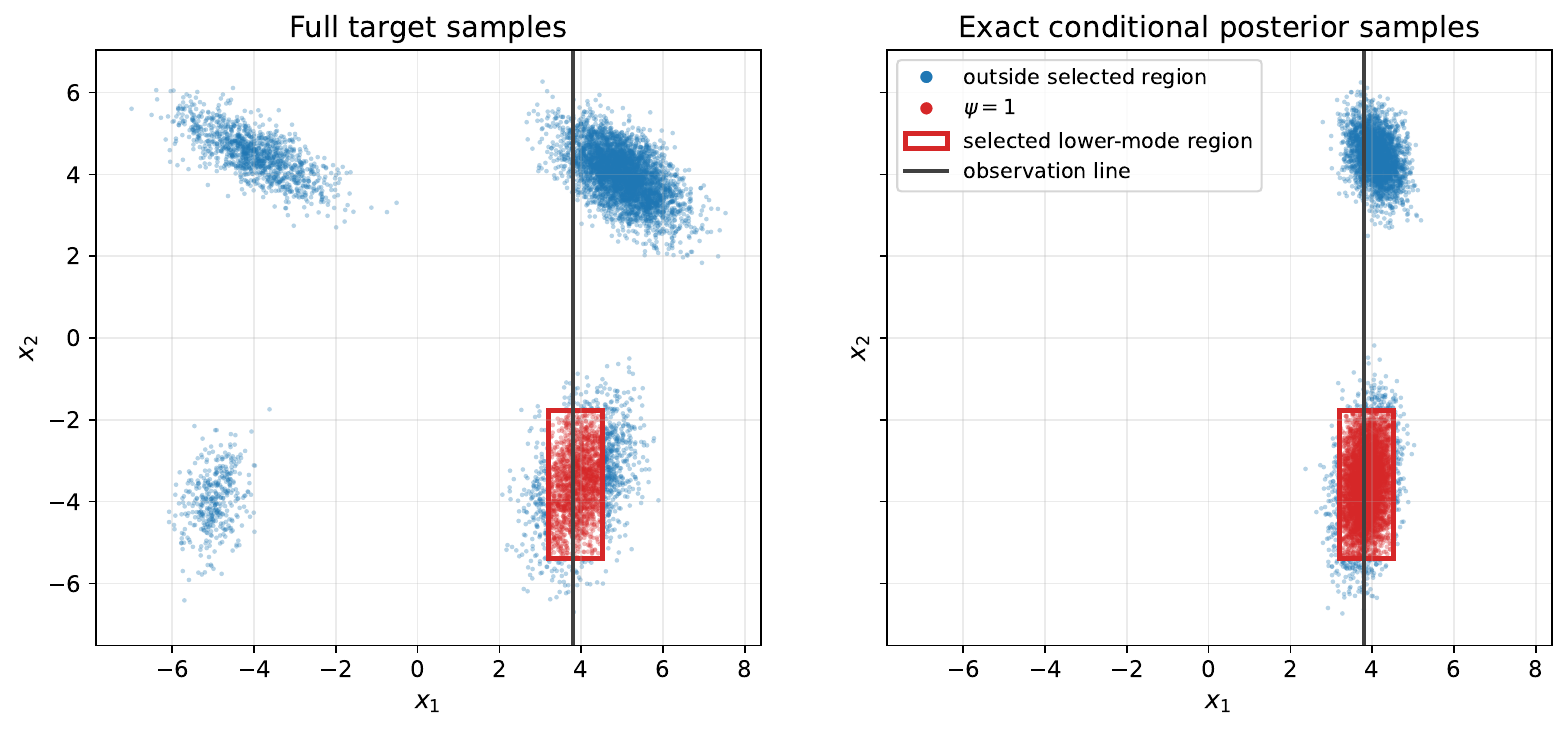}
    \caption{
    Two-dimensional Gaussian-mixture benchmark. The red rectangle defines the
    event $\{\psi_{\mathcal R}=1\}$ used to evaluate the samplers. Left: samples
    from the original Gaussian mixture. Right: samples from the exact
    conditional posterior. Blue dots correspond to points outside the selected region, while red dots correspond to points inside it.  The black vertical line indicates the observed value $\xobs=3.8$.}
    \label{fig:num_2d_geometry_main}
\end{figure}

We test different SMC sampling methods to approximate the posterior distribution. To make the comparison quantitative, we choose a visible region, the red rectangle, around one posterior mode. For each method, we estimate the posterior probability of this region by the weighted fraction of final particles falling inside it, and compare this estimate with the exact posterior probability, which is accurately estimated by standard numerical routines for the Gaussian CDF, as
\begin{equation} \label{eq:exact_posterior_value_ref}
\filter{N}[\psi_R]  = \sum_{j=1}^{4} \widetilde w_j \mathbb{P} \left(
Z_j \in \mathcal{R} \right) \eqsp, \qquad Z_j \sim \gaussiand{\widetilde\mu_j}{\widetilde\Sigma_j} \eqsp.
\end{equation}
The selected rectangle $\mathcal{R}$ is the red rectangle in \Cref{fig:num_2d_geometry_main}. The left panel shows samples from the original GMM and the right panel shows samples from the exact
conditional posterior.  Points are colored according to whether $\psi_R(\x)=1$. Moreover, we show a representative conditional particle cloud for the three mutation mechanisms used in the experiment.  The first panel in \Cref{fig:num_2d_conditional_clouds} shows independent samples from the exact posterior GMM.  The remaining panels show the final SMC particles obtained when the mutation kernels are, respectively, the exact reverse kernels $\{\mk{k+1}\}_{k=0}^{N-1}$ in \eqref{eq:num_exact_backward_kernel}, the oracle Euler proposals $\{\eulermk{k+1}\}_{k=0}^{N-1}$ in \eqref{eq:num_oracle_euler_proposal}, and the perturbed-score Euler proposals $\{\eulermk{k+1}[\theta]\}_{k=0}^{N-1}$ in \eqref{eq:num_learned_euler_proposal}. This plot is meant as a qualitative diagnostic: it shows whether the particle system places mass in the correct posterior regions and around the selected event $\{\psi_R=1\}$.

\begin{figure}[t]
    \centering
     \includegraphics[width=.95\linewidth]{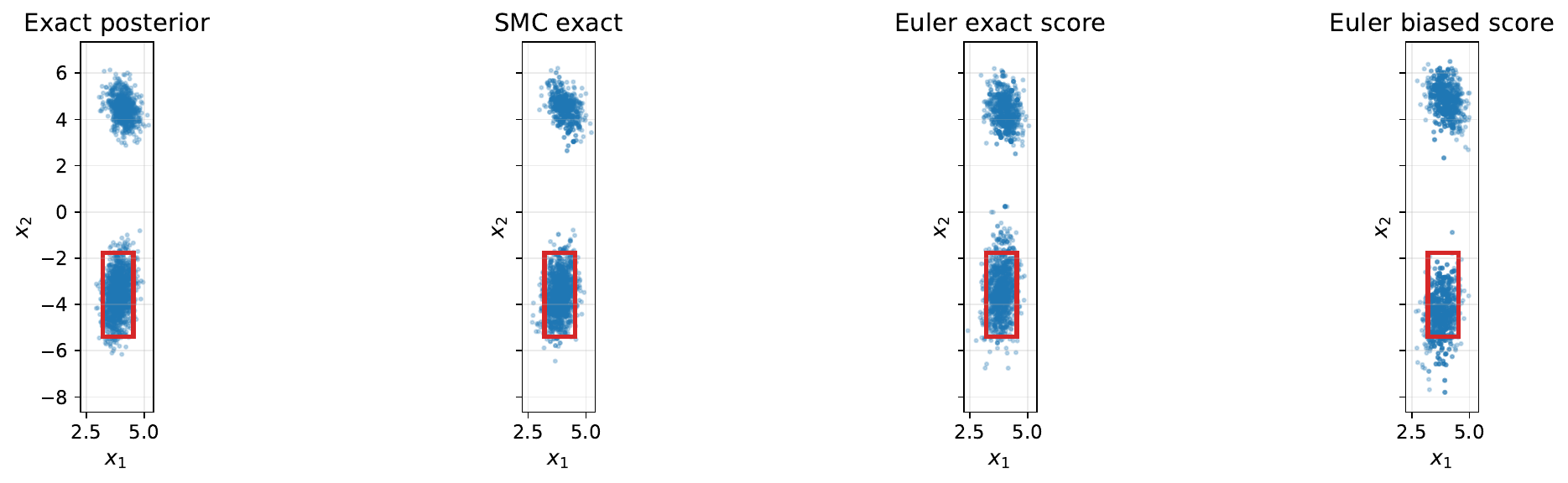}
    \caption{
Conditional samples in the two-dimensional benchmark.  The first panel shows
samples from the exact posterior GMM.  The other panels show final SMC particle
clouds obtained with the exact reverse kernels $\mk{k+1}$, the oracle Euler
kernels $\eulermk{k+1}$, and the perturbed-score Euler kernels
$\eulermk{k+1}[\theta]$.  The red rectangle is the modal region defining
$\psi_R$.
}
    \label{fig:num_2d_conditional_clouds}
\end{figure}

\paragraph{Illustration of \Cref{thm:total_smc_kernel_bias_error}}
We next quantify the two effects predicted by \Cref{thm:total_smc_kernel_bias_error}: \emph{the finite-particle Monte Carlo error} and \emph{the deterministic bias} induced by approximate mutation kernels.

For each particle count $M$, we run independent SMC repetitions and report the empirical $L_2$ error with respect to the exact posterior value $\filter{N}[\psi_{\mathcal R}]$ given in \eqref{eq:exact_posterior_value_ref}.  The three curves in \Cref{fig:num_2d_theorem_illustration} (Left) and in \Cref{fig:num_2d_kernel_bias_log} (same plot in log scale) correspond to SMC with the exact reverse kernels $\mk{k+1}$, the Euler kernels with true score function $\eulermk{k+1}$, and the Euler kernels with perturbed-score $\eulermk{k+1}[\theta]$.  The exact reverse kernels introduce no time-discretization error and therefore exhibit the expected Monte Carlo decay. The Euler kernels with true score function isolate the effect of time discretization, while the perturbed-score Euler kernels exhibit an additional non-vanishing bias plateau.

\begin{figure}[t]
    \centering
    \begin{minipage}{.49\linewidth}
        \centering
        \includegraphics[width=\linewidth]{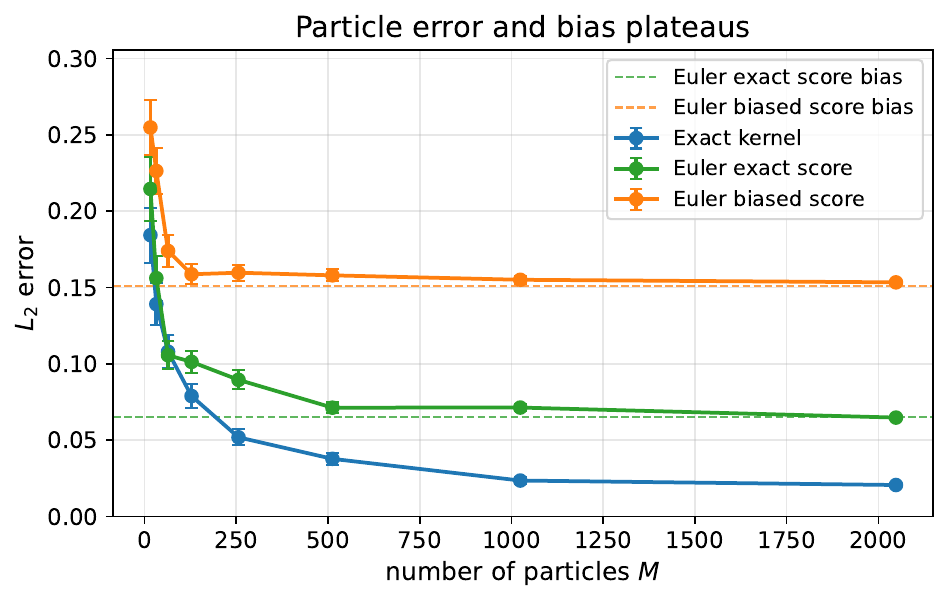}
    \end{minipage}
    \hfill
    \begin{minipage}{.49\linewidth}
        \centering
        \includegraphics[width=\linewidth]{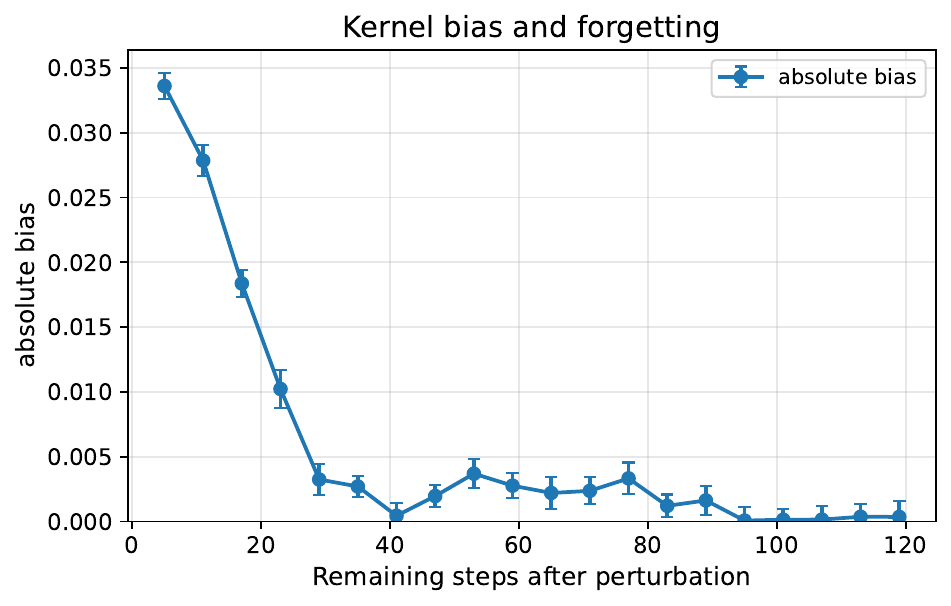}
    \end{minipage}
    \caption{
    Illustration of the two terms in
    \Cref{thm:total_smc_kernel_bias_error}. Left: empirical $L_2$ error as a
    function of the number of particles $M$, showing Monte Carlo decay and
    deterministic bias plateaus for approximate mutation kernels. Right:
    propagated effect of a single local kernel perturbation as a function of the
    number of remaining bridge steps, illustrating forward-smoothing forgetting. On the right panel, the left side corresponds to perturbations close to the data distribution $p_0$, while the right side corresponds to perturbations closer to the initial noised distribution $p_T$. Error bars show empirical standard errors across independent repetitions.
    }
    \label{fig:num_2d_theorem_illustration}
\end{figure}

\begin{figure}[t]
    \centering
     \includegraphics[width=.65\linewidth]{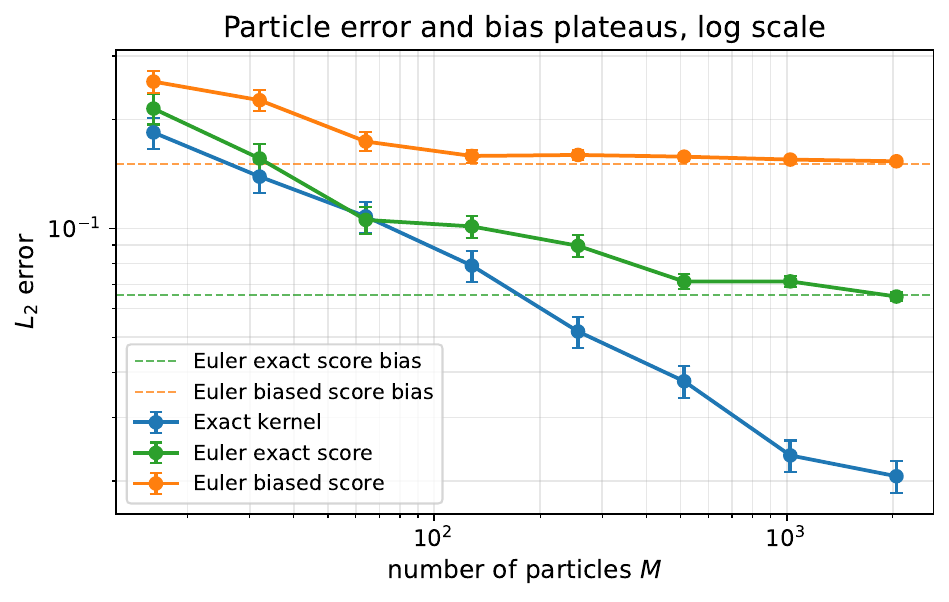}
    \caption{
        Log-scale version of the particle error and kernel bias plot for the
        two-dimensional GMM.  The empirical $L_2$ error is plotted as a function of
        the number of particles $M$, using the same repetitions and methods as in
        \Cref{fig:num_2d_theorem_illustration}.  The log scale makes the initial
        Monte Carlo decay and the subsequent bias plateaus easier to distinguish.
    }
    \label{fig:num_2d_kernel_bias_log}
\end{figure}

We then isolate the forgetting of a local kernel perturbation, which includes both discretization and score-approximation effects.  In this experiment all mutation kernels are exact, except at one selected reverse step where $\mk{k+1}$ is replaced by a perturbed-score Euler kernel $\eulermk{k+1}[\theta]$.  We vary the location of this single perturbation and measure the resulting absolute bias in the final estimate of $\filter{N}[\psi_{\mathcal R}]$.  The horizontal axis in \Cref{fig:num_2d_theorem_illustration} (Right) is the number of remaining reverse steps after the perturbation.  Moving to the right therefore corresponds to placing the
local perturbation earlier in the reverse samlping process, closer to the initial distribution $p_T$ and farther from the data distribution $p_0$, leaving more subsequent kernels through which its effect can be forgotten.

As an additional diagnostic, \Cref{fig:num_2d_discretization_error} varies the number of reverse steps $N$ while keeping the particle count fixed.  This plot checks that the discrepancy of the oracle Euler kernel is indeed tied to the time discretization of the reverse dynamics.  Increasing $N$ refines the reverse-time grid, and therefore reduces the discretization component of the
error, up to the remaining finite-particle variability and score perturbation
bias.

\begin{figure}[t]
    \centering
     \includegraphics[width=.65\linewidth]{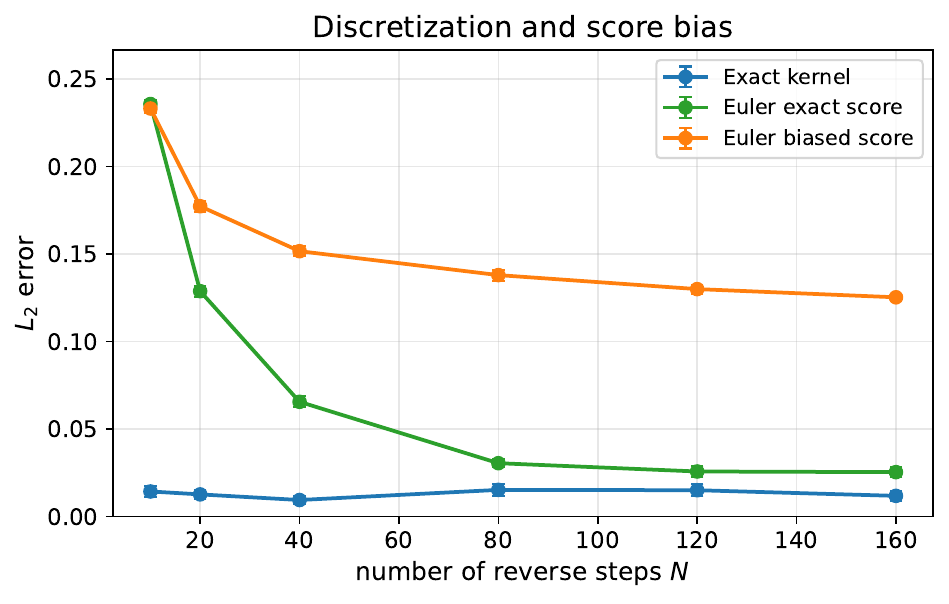}
    \caption{
    Discretization diagnostic for the two-dimensional GMM.  The empirical
    $L_2$ error is plotted as a function of the number of reverse steps $N$,
    with the particle count kept fixed.  This separates the effect of refining
    the Euler time discretization from the particle-count experiment in
    \Cref{fig:num_2d_theorem_illustration}.
    }
    \label{fig:num_2d_discretization_error}
\end{figure}

\subsection{A 25-component Gaussian mixture in dimension 50}
\label{sec:num_50d_gmm_experiment}

We next consider a higher-dimensional benchmark.  The data distribution is a Gaussian mixture on $\rset^{50}$ with $K=25$ components,
$$
p_0 (\x) = \sum_{j=1}^{25} w_j \gaussiand{\mu_j}{\Sigma_j}[\x] \eqsp, \qquad \x\in\rset^{50} \eqsp.
$$
The component means are arranged on a $5\times5$ grid in the first two coordinates and are zero in the remaining coordinates.  More precisely, if $(a_j,b_j)_{j=1}^{25}$ denotes the grid $\{-10,-5,0,5,10\}^2$, then
$$
\mu_j = (a_j,b_j,0,\ldots,0)^\top \in \rset^{50}.
$$
The weights are generated once from independent chi-square random variables with three degrees of freedom and then normalized. The covariance matrices are full and anisotropic: for each component,
$$
\Sigma_j = U_j \operatorname{diag}(1,2^{-1},\ldots,50^{-1}) U_j^\top \eqsp,
$$
where $U_j$ is an orthogonal matrix obtained from a random Gaussian matrix. All random quantities defining the mixture are fixed using the same seed in all runs. 

As in the two-dimensional experiment, we use a scalar observation acting only on the first coordinate,
$$
\Xobs = H\Xora_0+\varepsilon \eqsp, \qquad H=(1,0,\ldots,0) \eqsp,
\qquad \varepsilon\sim\gaussiand{0}{R} \eqsp,
\qquad R=1 \eqsp.
$$
In the reported experiment we condition on the fixed realization $\xobs=5.0$.  The observation is not resampled across repetitions: all methods are compared on the same conditional target $\mathcal L(\Xora_0\mid \Xobs=\xobs)$.

Although the sampler evolves in the full ambient space $\rset^{50}$, the test function is defined through the first two coordinates.  Let $\Pi_{1:2}\x=(x_1,x_2)$ and let $\mathcal R\subset\rset^2$ be an axis-aligned rectangle centered at the dominant posterior mode in the projected $(x_1,x_2)$ plane. We define, $\psi_{\mathcal R}(\x) \eqdef \indi{\mathcal R}[\Pi_{1:2}\x]$. Thus $\psi_{\mathcal R}$ measures the posterior probability assigned to one projected modal region, while still testing an SMC sampler operating in dimension $50$.

The exact posterior value is again computed deterministically from the closed-form posterior GMM.  If $\widetilde\mu_{j,1:2}$ and $\widetilde\Sigma_{j,1:2}$ denote the first two coordinates and the corresponding $2\times2$ marginal covariance of posterior component $j$, then
$$
\filter{N}[\psi_{\mathcal R}] = \sum_{j=1}^{25} \widetilde w_j \mathbb P(Z_j\in\mathcal R) \eqsp, \qquad Z_j\sim \gaussiand{\widetilde\mu_{j,1:2}}{\widetilde\Sigma_{j,1:2}} \eqsp.
$$

\begin{figure}[t]
    \centering
     \includegraphics[width=.78\linewidth]{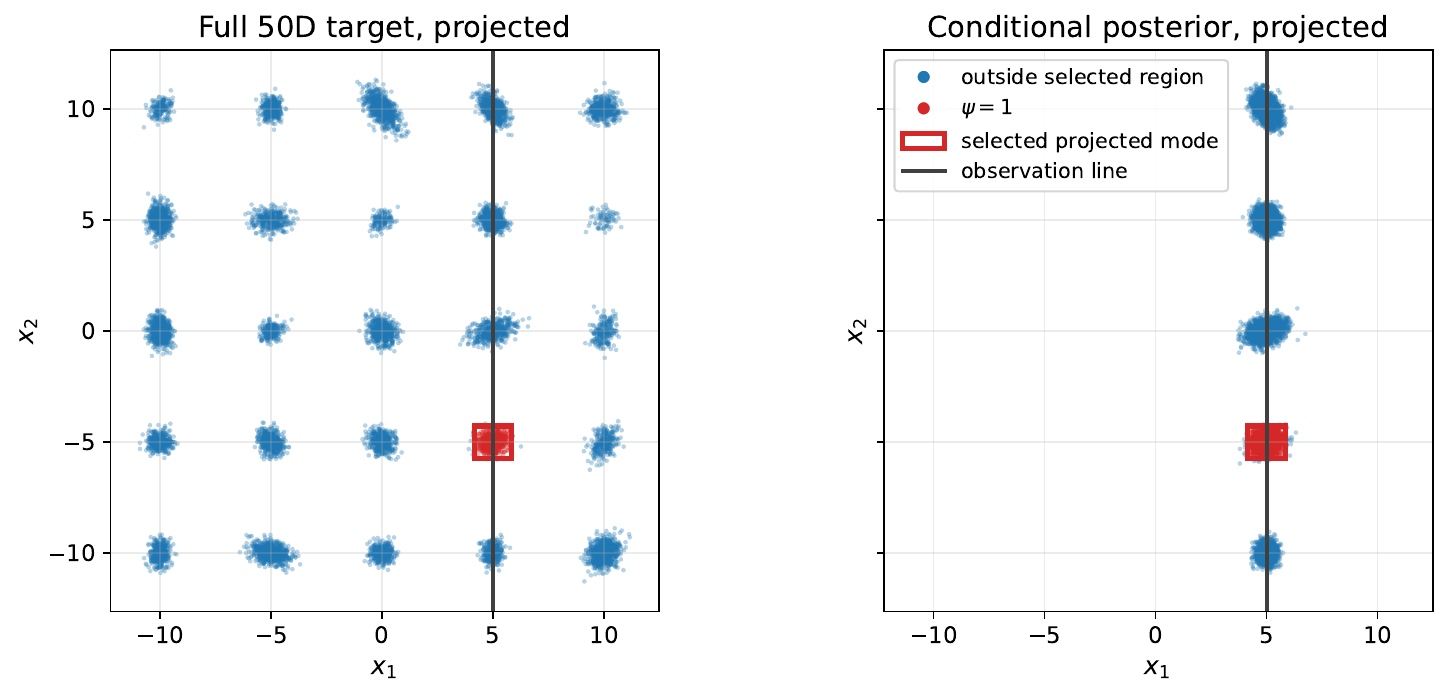}
    \caption{
    Projected geometry of the 50-dimensional test function.  The sampler evolves in $\rset^{50}$, but the plotted coordinates are $(x_1,x_2)$. The red rectangle defines the projected event $\{\psi_{\mathcal R}=1\}$ around the selected posterior mode.  Left: samples from the original GMM projected onto the first two coordinates.
    Right: samples from the exact conditional posterior projected onto the same
    plane. The black vertical line indicates the observed value $\xobs=5$
    }
    \label{fig:num_50d_geometry}
\end{figure}

\Cref{fig:num_50d_geometry} displays the projected geometry of the test function.  The figure is only a two-dimensional visualization of a 50-dimensional experiment: the  SMC sampler, the weights, and the proposal kernels all act in the full ambient space.

We also show representative projected conditional particle clouds in \Cref{fig:num_50d_conditional_clouds}.  The first panel shows independentssamples from the exact posterior GMM, while the remaining panels show the final
SMC particles obtained with the exact reverse kernels $\mk{k+1}$, the oracle Euler kernels $\eulermk{k+1}$, and the perturbed-score Euler kernels $\eulermk{k+1}[\theta]$.  The projection makes it possible to inspect whether
the particle system places mass in the correct visible posterior regions.

\begin{figure}[t]
    \centering
     \includegraphics[width=.98\linewidth]{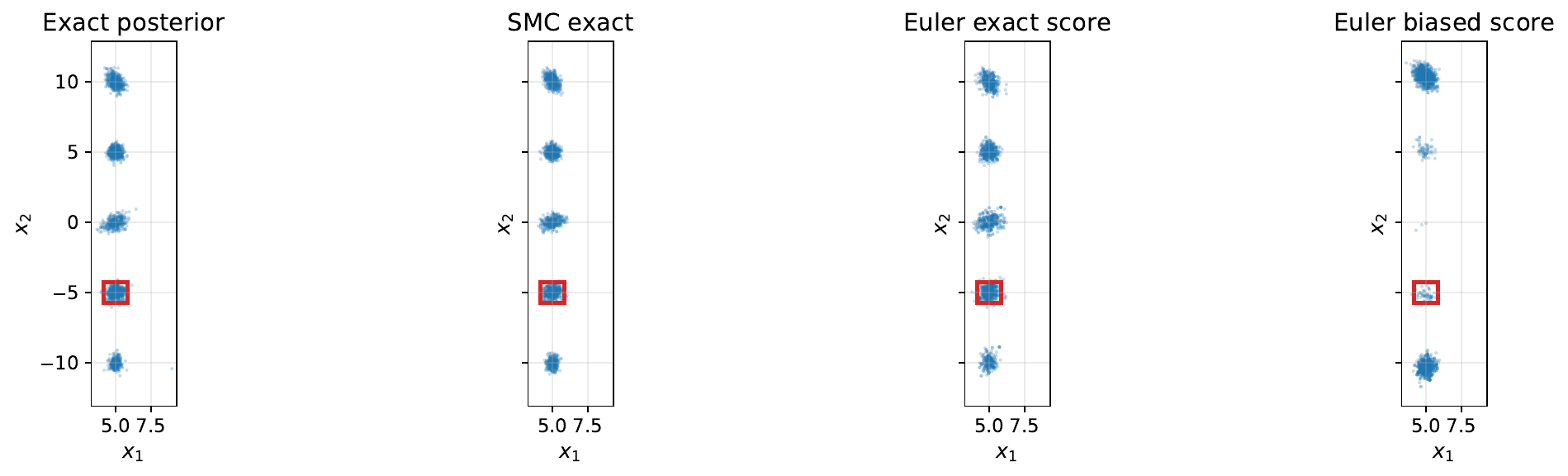}
    \caption{
    Conditional samples in the 50-dimensional benchmark, projected onto the
    first two coordinates.  The first panel shows samples from the exact
    posterior GMM.  The other panels show final SMC particle clouds obtained
    with the exact reverse kernels $\mk{k+1}$, the oracle Euler kernels
    $\eulermk{k+1}$, and the perturbed-score Euler kernels
    $\eulermk{k+1}[\theta]$.  The red rectangle is the projected modal region
    defining $\psi_{\mathcal R}$.
    }
    \label{fig:num_50d_conditional_clouds}
\end{figure}

\paragraph{Illustration of \Cref{thm:total_smc_kernel_bias_error}}
We repeat the same particle-error experiment as in the two-dimensional benchmark.  For each particle count $M$, we run independent SMC repetitions and report the empirical $L_2$ error with respect to the exact posterior value
$\filter{N}[\psi_{\mathcal R}]$.  The three curves in \Cref{fig:num_50d_kernel_bias} correspond to the exact reverse kernels, the oracle Euler kernels, and the perturbed-score Euler kernels.  This experiment checks whether the same Monte Carlo decay and deterministic bias plateau remain visible in a substantially higher-dimensional ambient space.

\begin{figure}[t]
    \centering
     \includegraphics[width=.65\linewidth]{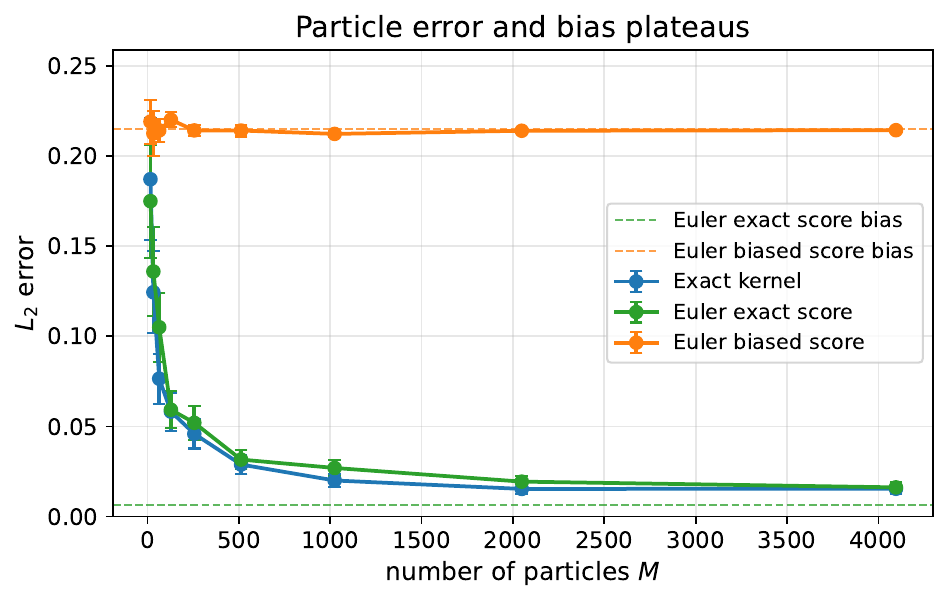}
    \caption{
    Particle error and kernel bias for the 50-dimensional benchmark.  The
    empirical $L_2$ error is plotted as a function of the number of particles
    $M$, using $30$ independent SMC repetitions for each value of $M$ and each
    method.  The horizontal dashed lines show
    large-particle estimates of the Euler bias plateaus, computed with
    $M_{\rm ref}=40\,000$ particles and $4$ independent reference runs.  Error
    bars show empirical standard errors across repetitions.
    }
    \label{fig:num_50d_kernel_bias}
\end{figure}

We finally report the local-perturbation diagnostic in the 50-dimensional setting.  As before, all mutation kernels are exact except at one selected reverse step, where the exact kernel is replaced by a perturbed-score Euler kernel.  This diagnostic is more demanding than in the two-dimensional example:
the sampler evolves in the full 50-dimensional space, while the test function is a discontinuous indicator depending only on a projected modal region.  As a result, the attenuation pattern is less smooth, but the experiment still illustrates how the effect of a local kernel perturbation depends on the number
of subsequent remaining steps.

\begin{figure}[t]
    \centering
     \includegraphics[width=.65\linewidth]{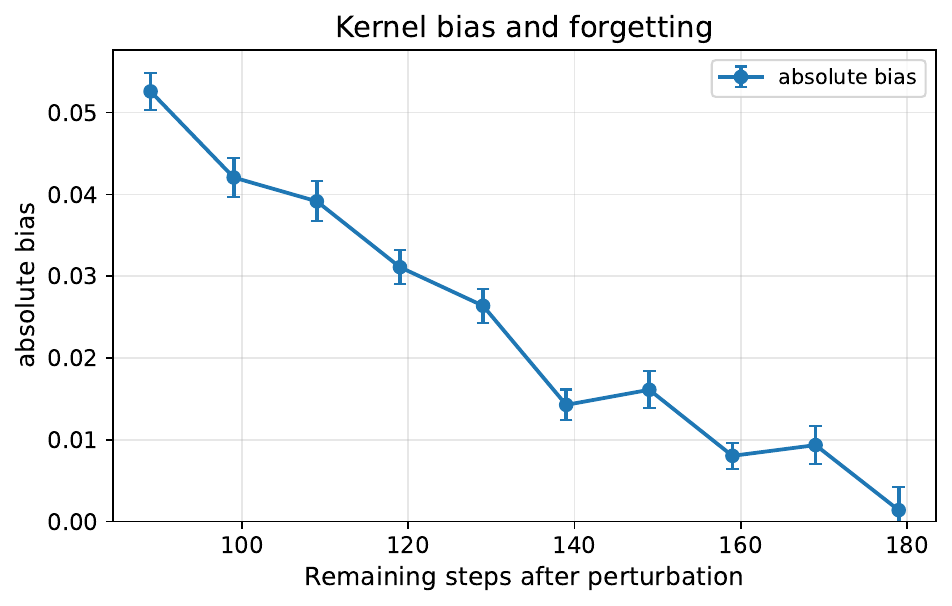}
    \caption{
    Kernel bias and forgetting for the 50-dimensional benchmark.  A single
    reverse transition is replaced by the perturbed-score Euler kernel
    $\eulermk{k+1}[\theta]$.  Error bars show empirical standard errors across independent repetitions.
    }
    \label{fig:num_50d_forgetting}
\end{figure}

\end{document}